\documentclass[10pt,a4paper]{extarticle}
\usepackage{graphicx}

\usepackage[round]{natbib}
\renewcommand{\cite}{\citep}

\usepackage[T1]{fontenc}
\usepackage{fullpage}
\usepackage{libertine}
\usepackage{authblk}

\usepackage[table]{xcolor}
\usepackage{microtype}
\usepackage{graphicx}
\usepackage{booktabs}
\usepackage{tabularx}
\usepackage{longtable}
\usepackage{fvextra}
\usepackage{subcaption}
\usepackage{placeins}
\usepackage{hyperref}
\usepackage{url}

\usepackage{amsmath}
\usepackage{amssymb}
\usepackage{mathtools}
\usepackage{amsthm}
\usepackage{nicefrac}
\usepackage[capitalize,nameinlink]{cleveref}
\crefname{appendix}{Supplement}{Supplements}
\Crefname{appendix}{Supplement}{Supplements}
\theoremstyle{plain}
\newtheorem{theorem}{Theorem}[section]
\newtheorem*{privacytheorem}{Theorem}

\theoremstyle{definition}
\newtheorem{definition}[theorem]{Definition}

\theoremstyle{remark}

\newcommand{\bx}{\mathbf{x}}
\newcommand{\bq}{\mathbf{q}}
\newcommand{\bk}{\mathbf{k}}
\newcommand{\bv}{\mathbf{v}}
\newcommand{\bs}{\mathbf{s}}
\newcommand{\bo}{\mathbf{o}}

\newcommand{\bw}{\mathbf{w}}
\newcommand{\bz}{\mathbf{z}}
\newcommand{\bu}{\mathbf{u}}
\newcommand{\br}{\mathbf{r}}
\newcommand{\calX}{\mathcal{X}}
\newcommand{\calY}{\mathcal{Y}}
\newcommand{\calD}{\mathcal{D}}
\newcommand{\calA}{\mathcal{A}}
\newcommand{\calS}{\mathcal{S}}
\newcommand{\calW}{\mathcal{W}}
\newcommand{\bpsi}{\pmb{\psi}}
\newcommand{\btheta}{\pmb{\theta}}
\newcommand{\bphi}{\pmb{\phi}}
\newcommand{\btau}{\pmb{\tau}}

\title{Efficient Provably Private Classification with a Tabular Foundation Model}
\author[1]{\mbox{Talal Alrawajfeh$^*$}}
\author[2]{\mbox{Cristiana Diaconu$^*$}}
\author[3]{\mbox{Ossi Räisä}}
\author[4]{\mbox{Sebastian Rodriguez Beltran}}
\author[1]{\mbox{Yuan He}}
\author[2]{\mbox{John Bronskill}}
\author[2]{\mbox{Richard E.
Turner}}
\author[1]{\mbox{Antti Honkela}}
\affil[1]{Department of Computer Science, University of Helsinki, Helsinki, Finland}
\affil[2]{Department of Engineering, University of Cambridge, Cambridge, United Kingdom}
\affil[3]{CISPA Helmholtz Center for Information Security, Saarbrücken, Germany}
\affil[4]{Current address: Vienna, Austria}
\date{}

\begin{document}
\maketitle
\vspace{-1.5em}
\begingroup
\renewcommand{\thefootnote}{\fnsymbol{footnote}}
\footnotetext[1]{These authors contributed equally to this work.}
\endgroup
\noindent\textbf{
Tabular data underpin prediction and decision-making in medicine, finance, government and science, but often contain sensitive individual-level information, creating a need for accurate prediction while preserving privacy.
Traditional private learning provides formal privacy guarantees, but requires slow dataset-specific optimisation, suffers substantial utility loss under strong privacy, and is often difficult to apply correctly.
Tabular foundation models adapt rapidly to new datasets, but existing models lack formal privacy guarantees, and are highly vulnerable to membership-inference attacks, limiting their use on sensitive data.
Here we introduce PrivTab, an easy to use tabular foundation model for differentially private classification that embeds a privacy mechanism within its architecture.
Pretrained on simulated datasets, PrivTab uses in-context learning to transform sensitive rows into compact, provably private summaries---effectively learning how to learn under privacy.
PrivTab outperforms private linear and neural-network baselines under moderate-to-strong privacy, shows negligible membership leakage, maintains well-calibrated predictions under strong privacy, and reduces dataset fitting time by 10,000 times, requiring only a single forward pass.
By combining formal privacy, speed, and easy of use, PrivTab brings recent advances in AI to applications where sensitive individual-level data have limited their adoption.
}

\par\medskip

\begin{figure*}[t]
\centering
\includegraphics[width=1.0\textwidth]{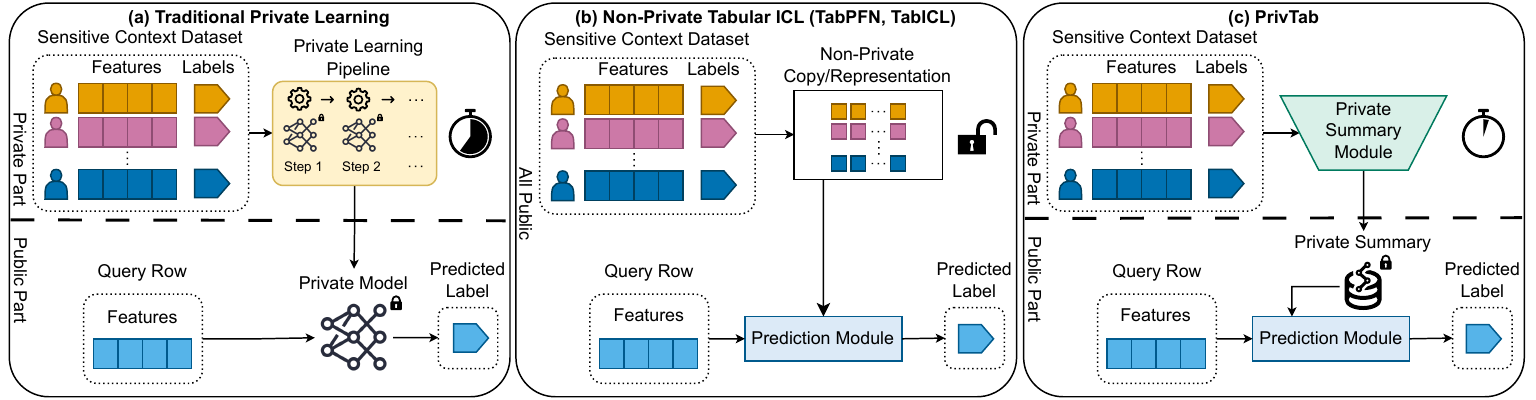}
\vspace{0.5em}
\begin{minipage}[t]{0.54\textwidth}
\centering {\footnotesize
\textbf{(d) Qualitative comparison}\par
\vspace{0.2em}
\newcommand{\tableYes}{\ensuremath{\textcolor[HTML]{009E73}{\boldsymbol{\checkmark}}}}
\newcommand{\tableNo}{\ensuremath{\textcolor[HTML]{D55E00}{\boldsymbol{\times}}}}
\newcommand{\tableNA}{\textcolor[HTML]{666666}{---}}
\setlength{\tabcolsep}{3pt}
\renewcommand{\arraystretch}{1.06}
\rowcolors{2}{gray!8}{white}
\begin{tabularx}{\linewidth}{>{\raggedright\arraybackslash\bfseries}X>{\centering\arraybackslash}p{0.15\linewidth}>{\centering\arraybackslash}p{0.24\linewidth}>{\centering\arraybackslash}p{0.14\linewidth}}
\toprule
Property & \shortstack{Traditional\\Private\\Learning} & \shortstack{Non-Private\\Tabular\\In-Context\\Learning\\(TabPFN, TabICL)} & \textbf{PrivTab} \\
\midrule
Formal privacy guarantee      & \tableYes & \tableNo & \tableYes \\
Fast fitting for new datasets & \tableNo & \tableYes & \tableYes \\
Privacy-critical path is easy to inspect and audit & \tableNo & \tableNA & \tableYes \\
\bottomrule
\end{tabularx}
}
\end{minipage}\hfill
\begin{minipage}[t]{0.44\textwidth}
\centering
\resizebox{\linewidth}{!}{\textbf{(e) Membership inference --- Maternal Health Risk dataset}}\par
\vspace{0.1em}
\includegraphics[width=\linewidth]{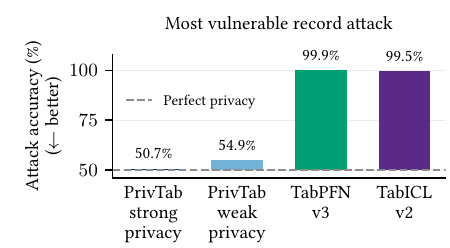}
\end{minipage}
\vspace{-0.35em}
\captionsetup{skip=2pt}
\caption{
\textbf{Alternative approaches to tabular prediction from sensitive data.}
\textbf{a}, Private learning trains a new model on each sensitive dataset, requiring relatively slow dataset-specific optimisation.
\textbf{b}, Non-private tabular learning fits new datasets in a single forward pass through in-context learning, but retains sensitive data for prediction.
\textbf{c}, PrivTab instead releases a reusable private summary in one forward pass; all subsequent predictions are private due to the post-processing property of DP.
Dashed lines mark the privacy boundary, and stopwatches indicate relative rather than measured time.
\textbf{d}, Qualitative comparison; green ticks, orange crosses and grey dashes denote presence, absence and non-applicability, respectively.
The auditability comparison concerns the complete data-dependent path, including hyperparameter tuning and privacy accounting, rather than an isolated private update \citep{Ponomareva2023HowToDPfyML,Papernot2022Hyperparameter,Ganev2025PATEGAN}.
\textbf{e}, Membership-inference attack accuracy for the most vulnerable record of each method.
For each record in the dataset, we compute the lower endpoint of the 95\% confidence interval for attack accuracy; the bar shows the maximum of these lower endpoints.
An accuracy of 50\% corresponds to perfect privacy.
TabPFN v3 and TabICL v2 approach 100\%, whereas PrivTab remains much closer to perfect privacy (50\%) at the displayed strong and weak formal privacy levels.
}
\label{fig:dp-tnp-overview}
\end{figure*}

Records about individuals are predominantly tabular, from medical histories and financial transactions to census returns and administrative files.
Machine learning offers immense potential to transform the analysis of tabular data, improving predictions and the decisions they inform.
Yet a model, or even a single prediction produced by it, can inadvertently reveal sensitive information.
This tension raises a central scientific question: can a model deliver fast and accurate classification on sensitive tabular data while formally guaranteeing that any single individual's record remains private?

Differential privacy (DP) \citep{Dwork2006DP,Dong2022GDP} has emerged as the principal framework for resolving this conflict, formally limiting how much a model's outputs can depend on any one individual's record.
In standard private machine learning \citep{Ponomareva2023HowToDPfyML}, DP stochastic optimisation \citep{Song2013DPSGD,Bassily2014DPEmpiricalRisk,Abadi2016DPSGD} achieves this protection by bounding the contribution of any individual's record to the parameter updates and injecting calibrated noise during training.
However, applying this approach to tabular data creates major operational hurdles.
Training must be repeated from scratch for every new dataset and target privacy level, frequently at a steep cost to predictive performance, especially under strong privacy.
Moreover, hyperparameter tuning and model selection can weaken privacy guarantees, and their privacy cost must be accounted for alongside model fitting~\citep{Papernot2022Hyperparameter}.
Even when the underlying mechanisms are theoretically private, implemented pipelines can fail to provide the intended guarantee because of privacy-accounting mistakes, noise-calibration errors or subtle implementation bugs \citep{Ding2018DetectingDPViolations,Bichsel2021DPSniper,Tramer2022DebuggingDP,Nasr2023TightAuditing,Ganev2025PATEGAN,Cebere2026GreyBoxAuditing} (\cref{fig:dp-tnp-overview}d).
This reveals a critical gap for domain practitioners---who are rarely privacy specialists---and motivates classification methods with privacy built in that are easy to use and straightforward to audit.

Tabular foundation models (TFMs) address the problem of repeated training from scratch for each new dataset.
Models such as TabPFN \citep{Hollmann2023TabPFNICLR,Hollmann2025TabPFNNature} and TabICL \citep{Qu2025TabICL} take as input a labelled context dataset together with new query rows and predict their labels in a single forward pass.
This eliminates the need for iterative retraining and enables near-instant fitting and prediction.
They achieve this through in-context learning, by pretraining across millions of simulated tasks consisting of context--target tabular dataset pairs.
The labelled context dataset is used to fit the model, whereas the held-out target dataset contains the query rows and their target labels which are used to evaluate the predictions and update the model.
These simulated tasks contain no sensitive information.
However, these models lack formal privacy protections, and for both TabPFN and TabICL, sensitive context rows---or non-private representations of them---must remain available at inference time to generate predictions.
Indeed, using only black-box access to model predictions, we demonstrate that TabPFN and TabICL are vulnerable to context membership-inference attacks, allowing an attacker to determine whether a specific individual's record was included in the context dataset with near-perfect accuracy (\cref{fig:dp-tnp-overview}e).

To address both the challenges of traditional private learning and the vulnerability of existing TFMs, we introduce PrivTab, which embeds a privacy mechanism directly within a TFM.
Evaluated across a broad tabular benchmark, PrivTab achieves stronger predictive performance and better uncertainty estimates than conventional private baselines under moderate-to-strong privacy, remains competitive under weaker privacy, and reduces fitting time from minutes to milliseconds.
It also retains its performance in a workflow that includes private preprocessing.
Together, these results show that provable and auditable privacy can be built natively into foundation-model architectures, avoiding costly and difficult dataset-specific private training and opening a path to practical, privacy-preserving prediction from sensitive tabular data at scale.

\section{Principled and efficient privacy-aware in-context learning}
PrivTab embeds the privacy mechanism directly within the foundation model architecture at the point where sensitive information enters.
It compresses a sensitive context dataset into a compact DP summary (\cref{fig:dp-tnp-overview}).
Once constructed, this summary can be reused for all downstream query rows, avoiding repeated processing of the sensitive context and making prediction independent of the original context size.
This design enables efficient fitting and prediction using standard computational hardware.
Because all downstream predictions rely solely on this summary, the post-processing property of DP~\citep{Dwork2006DP} ensures they incur no additional privacy loss, allowing the summary to be stored and queried indefinitely.

Existing work based on pretrained models typically privatise either the inputs to or outputs from a model~\citep{gao2025data,wu2024privacypreserving,carey2024dptabicl}.
While prior approaches have mainly targeted textual data and natural-language tasks~\citep{gao2025data,wu2024privacypreserving}, DP-TabICL~\citep{carey2024dptabicl} computes noisy aggregate statistics over groups of tabular records and converts them into textual examples for large language model (LLM) prediction.
Such methods rely on general-purpose LLM inference, which can be substantially more expensive because prediction requires repeated inference for each query row.
DP-TabICL, for example, reports prediction times of 0.14--1.11 seconds per query row, and evaluates only 10\% of the target rows for the larger datasets owing to inference costs~\citep{carey2024dptabicl}.

Drawing on TabArena's computational guidelines~\citep{EricksonTabArena2025}, which constrain computational resources and limit model evaluation to at most one hour, we adopt a controlled computational setting.
We therefore consider only dedicated private tabular methods compatible with this computational regime: methods that can be fitted on a central processing unit (CPU) or a graphics processing unit (GPU) and support efficient prediction on CPU.
This is particularly important for broad adoption, since downstream users might have limited computational resources and rely on standard CPU-based hardware.

PrivTab pretrains the foundation model with the privacy mechanism on millions of simulated tasks and across a range of privacy levels, achieving privacy-aware in-context learning---effectively ``learning how to learn under privacy''~\citep{raisa2024noise}.
This enables PrivTab to account explicitly for the effect of privacy noise in its predictions and uncertainty estimates---a crucial requirement for high-stakes decision-making.
By contrast, standard private-learning pipelines do not provide this guarantee by default unless the effect of privacy noise on predictive uncertainty is explicitly modelled.

Designing a private summary mechanism  presents several technical challenges: the influence of each individual record must be bounded, the attention architecture must remain expressive despite the added noise, and the implementation must remain simple enough to be auditable.
We address these requirements by designing differentially private multi-head cross-attention (DP-MHCA), a bounded cross-attention layer that produces a private summary with sensitivity that does not grow with context size.
This design confines the privacy-critical operations to a self-contained transformer module.

PrivTab's privacy-critical path is thus more localised and easier to inspect and audit than that of a conventional private-training pipeline.
To strengthen the connection between the mathematical privacy guarantee (\cref{app:dp-tnp}) and the executable implementation, we used three complementary approaches.
First, we isolated the privacy-critical operations in a self-contained reference implementation (\cref{app:dp-mhca-reference-implementation}).
Second, we formally verified the sensitivity bound and adaptive Gaussian differential privacy (GDP) composition in Lean 4 \citep{deMouraUllrich2021Lean4}.
Third, empirical sensitivity and privacy audits found behaviour consistent with the intended noise calibration across all evaluated privacy regimes (\cref{app:empirical-audits}).
These checks strengthen confidence that the implementation follows the proved mechanism; the Lean formalisation itself does not verify the exact software execution.

\section{Private classification across datasets}

\begin{figure}[!th]
    \centering
    \captionsetup{skip=3pt}
    \begin{subfigure}[t]{0.48\textwidth}
        \centering
        \includegraphics[width=\linewidth]{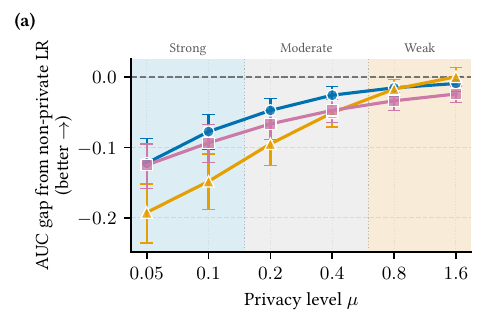}
        \phantomcaption
        \label{fig:exp-nondp-auc}
    \end{subfigure}
    \hfill
    \begin{subfigure}[t]{0.48\textwidth}
        \centering
        \includegraphics[width=\linewidth]{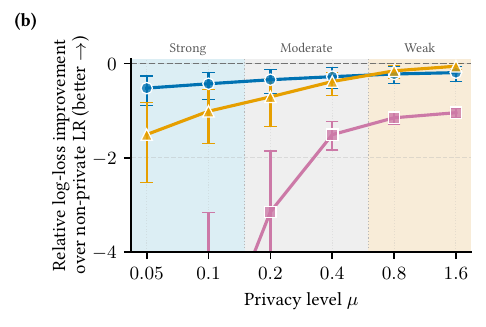}
        \phantomcaption
        \label{fig:exp-nondp-loss}
    \end{subfigure}
    \par\vspace{-0.1em}
    \includegraphics[width=0.42\textwidth]{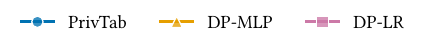}
    \par\vspace{-0.1em}
    \begin{subfigure}[t]{0.48\textwidth}
        \centering
        \includegraphics[width=\linewidth]{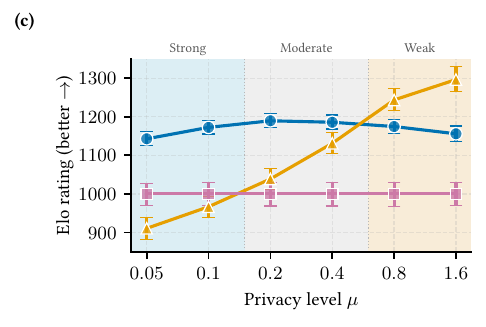}
        \phantomcaption
        \label{fig:exp-elo-by-mu}
    \end{subfigure}
    \hfill
    \begin{subfigure}[t]{0.48\textwidth}
        \centering
        \includegraphics[width=\linewidth]{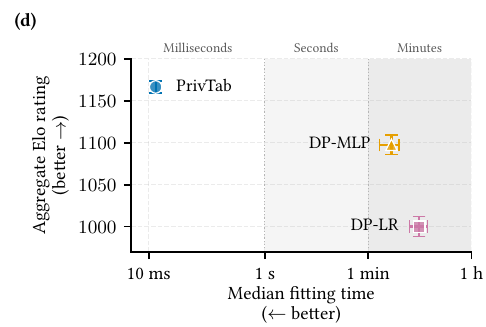}
        \phantomcaption
        \label{fig:exp-elo-aggregate}
    \end{subfigure}
    \par\vspace{-0.5em}
    \caption{
    \textbf{Predictive performance and fitting time across 33 TabArena datasets.}
    \textbf{a}, Paired AUC gap, \(\mathrm{AUC}_{m}-\mathrm{AUC}_{\mathrm{non\text{-}private\,LR}}\), for each method $m$ (PrivTab, DP-MLP, DP-LR).
    \textbf{b}, Relative log-loss improvement over non-private logistic regression for the same three methods.
    Each quantity is computed on matched context--target splits before aggregation; zero denotes parity with non-private logistic regression and positive values mean \(m\) is better.
    Error bars are bootstrap 95\% confidence intervals over datasets.
    \textbf{c}, Elo rating at each privacy level; higher ratings indicate better predictive performance, and shaded bands mark strong (\(\mu<0.15\)), moderate (\(0.15\leq\mu<0.6\)) and weak (\(\mu\geq0.6\)) formal privacy.
    Elo combines binary-task AUC and multiclass-task log loss and is calibrated separately at each privacy level so that differentially private logistic regression (DP-LR) scores 1,000.
    All results use ten context--target splits per dataset.
    \textbf{d}, Aggregate Elo over all privacy levels versus median fitting time per context--target split on a logarithmic scale, so better methods lie towards the upper left.
    Horizontal and vertical bars are bootstrap 95\% confidence intervals for median fitting time and Elo rating, respectively.
    }
    \label{fig:exp-elo}
    \label{fig:exp-nondp-gap}
\end{figure}

\begin{figure}[!ht]
    \centering
    \captionsetup{skip=3pt}
    \includegraphics[width=\textwidth,trim=0 7bp 0 0,clip]{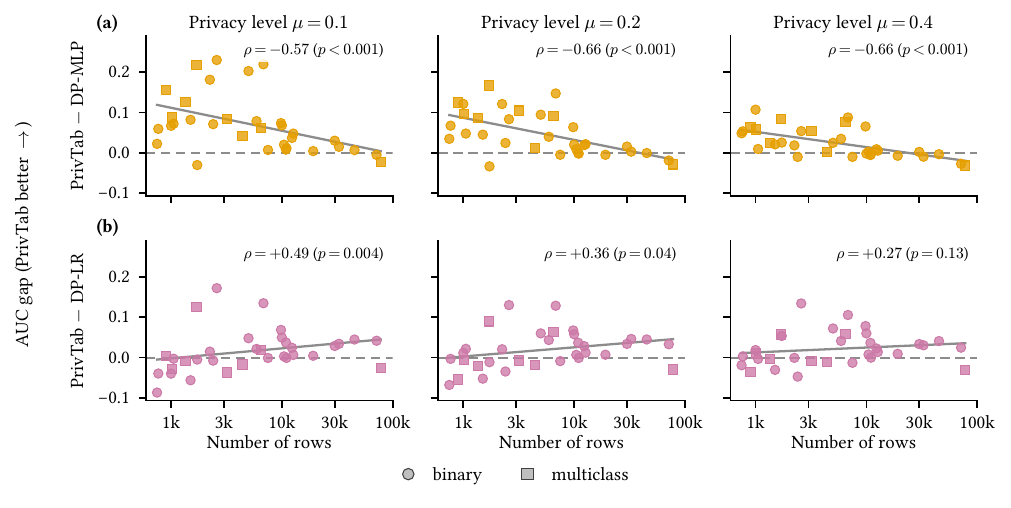}
    \caption{
    \textbf{AUC gaps versus dataset size across 33 TabArena datasets.}
    \textbf{a}, Mean paired AUC gap between PrivTab and the differentially private multilayer perceptron (DP-MLP).
    \textbf{b}, The corresponding gap between PrivTab and differentially private logistic regression (DP-LR).
    Columns show privacy levels \(\mu=0.1, 0.2\) and \(0.4\); each point is a dataset's mean paired gap across ten context--target splits, plotted against its number of rows on a logarithmic scale.
    Positive gaps favour PrivTab, while the dashed horizontal line marks equal AUC.
    Circles denote binary datasets and squares denote multiclass datasets; AUC is macro-averaged one-versus-rest for multiclass tasks.
    Grey lines are least-squares fits against log dataset size, shown as visual guides; annotations report Spearman's \(\rho\) and the corresponding \(p\)-value.
    }
    \label{fig:elo-battle-breakdown}
\end{figure}

We evaluated PrivTab on 33 classification datasets from TabArena \citep{EricksonTabArena2025}, a standard tabular benchmark, across six privacy regimes and ten independent context--target splits per dataset.
In each split, 80\% of a dataset's labelled rows form the context used to fit each method, and the remaining 20\% are held-out target rows used to evaluate its predictions.
To accommodate varying table sizes, we pretrained two model variants specialised for small and large datasets, routing incoming datasets accordingly.
We compared PrivTab with two carefully tuned baselines: differentially private logistic regression (DP-LR) and a differentially private multilayer perceptron (DP-MLP).
Both are standard model classes for private tabular learning and can be trained using DP stochastic optimisation for each dataset and privacy level. In contrast, PrivTab requires no dataset-specific optimisation at any privacy level.

PrivTab is more accurate throughout the moderate-to-strong privacy regimes, with DP-MLP closing the gap only under weak privacy constraints (\cref{fig:exp-nondp-auc,fig:exp-nondp-loss}).
PrivTab's advantage is most pronounced in log loss under heavy noise (\cref{fig:exp-nondp-loss}), consistent with training through an embedded privacy mechanism helping to preserve reliable uncertainty estimates when privacy bounds are tightest.
In aggregate Elo comparisons spanning all datasets, privacy levels and splits, PrivTab achieves the highest overall rating (\cref{fig:exp-elo-aggregate}), driven by its better performance in the moderate-to-strong privacy regimes (\cref{fig:exp-elo-by-mu}).

The dataset-level breakdown shows contrasting trends with row count (\cref{fig:elo-battle-breakdown}): PrivTab's AUC advantage over DP-MLP is largest on smaller datasets and diminishes as datasets grow, whereas its advantage over DP-LR generally grows with dataset size.
This is consistent with differences in how the models exploit additional context: PrivTab, like other tabular foundation models, may eventually saturate as context size increases, while DP-LR has limited capacity to benefit from larger datasets.
In contrast, the greater capacity of DP-MLP allows it to continue benefiting from additional training data, thereby narrowing the gap to PrivTab on larger datasets.
PrivTab also has a positive AUC gap for most of the 33 datasets at each of the three plotted privacy levels, particularly against DP-MLP (\cref{fig:supplement-auc-gap-features}).
The corresponding relationship with feature count is less pronounced.
Complete per-dataset breakdowns against DP-MLP and DP-LR are provided in \cref{fig:supplement-auc-gap-dp-mlp,fig:supplement-auc-gap-dp-lr}.

These results are even more striking when considering the computational difference between PrivTab and traditional private learning.
PrivTab requires a median fitting time of approximately 10 milliseconds per context--target split---more than four orders of magnitude faster than DP-MLP and DP-LR, which require minutes of dataset-specific private training (\cref{fig:exp-elo-aggregate}).
Across the entire benchmark, fitting DP-LR and DP-MLP required approximately 400 and 130 graphics processing unit (GPU) hours\footnote{GPU time is the cumulative wall-clock time spent running experiments on this hardware.}, respectively.
In contrast, fitting PrivTab required only 30 GPU seconds.
This cost reflects the repeated dataset-specific training and private model selection required by the baselines: there were 198 dataset--privacy configurations per baseline (33 datasets $\times$ 6 privacy levels).
Ten splits per configuration required 1,980 final training runs per baseline, each preceded by private learning-rate selection that consumed part of the privacy budget.
In contrast, PrivTab requires neither dataset-specific nor privacy-level-specific optimisation or hyperparameter tuning.
Each pretrained PrivTab variant can be used across different privacy levels without retraining.
Comprehensive dataset-level runtimes for each method are provided in \cref{fig:exp-runtime}.

On a CPU, median prediction times for the full target set were 66.6 milliseconds (ms) for PrivTab, 0.35\,ms for DP-MLP and 0.18\,ms for DP-LR (\cref{fig:inference-timing,tab:supplement-inference-timing}).
These medians were computed across the 33 datasets after averaging timings within each dataset.

\FloatBarrier
\section{Private preprocessing in practice}

The TabArena benchmark isolates prediction by applying a common idealised preprocessing procedure to all methods.
Here, preprocessing refers to shifting and scaling the  columns in the sensitive dataset before model fitting, using statistics computed from those columns.
The same statistics are then used to transform the queries accordingly and must therefore be privatised before being released alongside the private summary.
We show that this is possible in practice for datasets with public column descriptions detailed enough to infer meaningful clipping intervals, and whose modest feature counts allow private preprocessing to retain useful signal.
A clipping interval specifies lower and upper bounds for a numeric feature, with values outside the interval truncated to the nearest bound.
These bounds limit the sensitivity of the estimated means and variances, allowing the noise required for private preprocessing to be calibrated.
We inferred numeric clipping intervals from the public descriptions using an LLM; no sensitive values were supplied to that model.

The total privacy budget used in the case study here, \(\mu=0.4\), was split between private preprocessing and private prediction or training.
Bank Marketing, Pima Diabetes and Maternal Health Risk originate from the UCI Machine Learning Repository \citep{Kelly2023UCI}; the other two datasets are COMPAS from OpenML and ACS Income from Folktables \citep{OpenML2025,Ding2021RetiringAdult}.
PrivTab achieved the highest Elo rating across the five datasets in the case study (\cref{fig:end-to-end-workflows}c), with numerical values reported in \cref{tab:end-to-end-elo}; method-level AUC and log-loss results are provided in \cref{fig:end-to-end-case-auc}b and \cref{fig:end-to-end-workflows}b, respectively.
Oracle (empirical min--max) bounds and LLM-derived bounds yielded similar AUC and log loss on four of the five datasets (\cref{fig:end-to-end-case-auc}a and \cref{fig:supplement-preprocessing-logloss}).
The oracle comparison is partially private: its standardisation and prediction are privatised, but its data-dependent bounds do not have an end-to-end privacy guarantee. The LLM-derived bounds use only public descriptions.
Pipeline details, prompts, bounds and numerical results appear in \cref{app:end-to-end-details}.

\begin{figure}[!t]
    \centering
    \captionsetup{skip=3pt}
    \includegraphics[width=\textwidth,trim=0 2bp 0 0,clip]{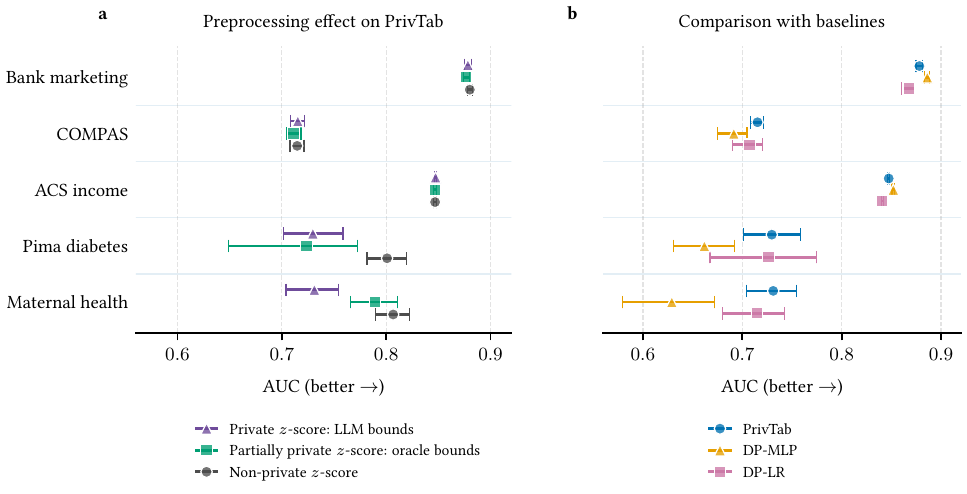}
    \caption{
    \textbf{Private preprocessing case study.}
    \textbf{a}, Mean PrivTab AUC across ten context--target splits at total privacy level \(\mu=0.4\).
    The three preprocessing variants compare non-private standardisation, partially private preprocessing using oracle bounds computed from empirical dataset minima and maxima, and private preprocessing using bounds inferred from public dataset descriptions by an LLM.
    Only the LLM-derived bounds avoid using sensitive values to select clipping intervals; the oracle-bounds comparison is partially private.
    \textbf{b}, Mean AUC for PrivTab, DP-MLP and DP-LR using LLM-derived bounds and private standardisation across the same five datasets and ten splits.
    Higher AUC is better. Error bars show bootstrap 95\% confidence intervals. Legends beneath each panel identify its plotted variants.
    }
    \label{fig:end-to-end-case-auc}
\end{figure}

\FloatBarrier

\section{Non-private models reveal membership}

TFMs use a context dataset to infer labels for incoming query rows.
Because these models fit directly to the sensitive context dataset, they provide no formal guarantee against membership inference of specific context records.
To evaluate this vulnerability, we conducted a likelihood-ratio membership-inference attack \citep{Carlini2022MembershipInference}, which distinguishes members of a context set from non-members based on prediction confidence.
Because these models fit to context records without privacy protection, they can exhibit higher confidence when queried on the same context records.
In contrast, PrivTab and private-learning methods use DP mechanisms that limit how reliably an attacker can infer context membership from prediction confidence.

\begin{figure}[!t]
    \centering
    \captionsetup{skip=3pt}
    \includegraphics[width=\textwidth,trim=0 9bp 0 0,clip]{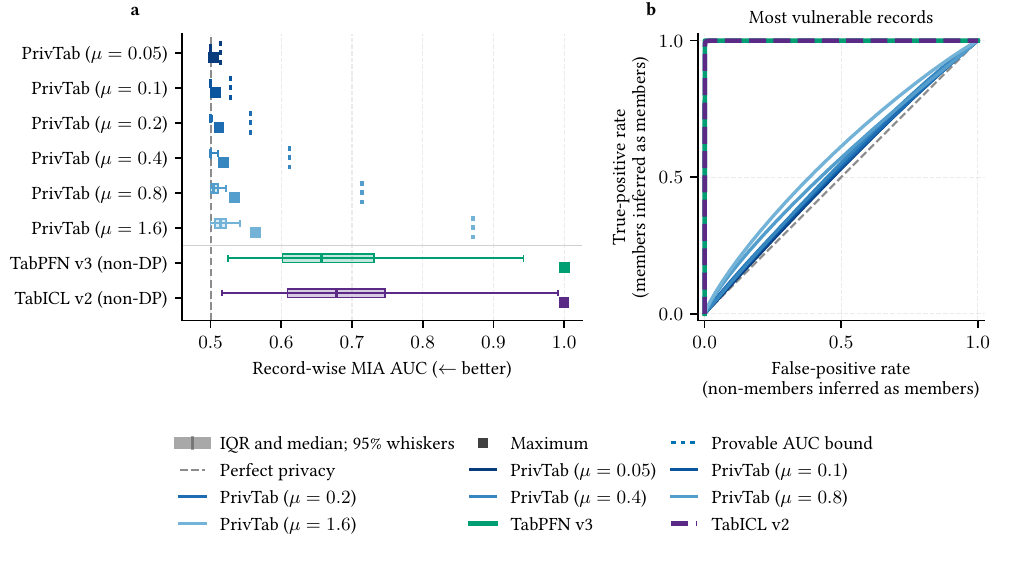}
    \caption{
    \textbf{Record-level membership-inference protection.}
    \textbf{a}, Distribution of the record-wise membership-inference AUC on the Maternal Health Risk dataset. For each record, we compute its AUC and 95\% confidence interval, then plot the lower endpoint of that interval. Boxes show the interquartile range and median, whiskers span the central 95\%, squares mark the maximum across records, and dotted vertical lines mark the provable AUC upper bound for PrivTab at each privacy level.
    \textbf{b}, ROC curves for the most vulnerable record at each of the six PrivTab privacy levels and for TabPFN v3 and TabICL v2.
    Grey dashed lines in \textbf{a,b} denote perfect privacy.
    PrivTab remains much closer to perfect privacy whereas the most vulnerable rows of the non-private models are highly vulnerable.
    }
    \label{fig:membership-inference-summary}
\end{figure}

We performed the analysis on clinical records from the Maternal Health Risk dataset.
Rather than reporting aggregate averages, which can hide highly vulnerable records \citep{Knolle2026DisparatePrivacyRisks}, we report record-level vulnerability using fitted attack accuracy \cref{fig:dp-tnp-overview}e and the AUC of record-level ROC curves \cref{fig:membership-inference}a,b.
The conservative lower confidence estimates of fitted attack accuracy in \cref{fig:dp-tnp-overview}e approach 100\% for the most vulnerable records under TabPFN v3 and TabICL v2.
In contrast, PrivTab remains much closer to perfect privacy across all evaluated privacy levels.
While empirical attacks cannot replace formal proofs, they illustrate the practical importance of DP guarantees.
Complete attack specifications, formal confidence bounds and full benchmark evaluations are detailed in the Methods and \cref{app:membership-inference-results}.

\FloatBarrier

\section{Discussion}

Our results suggest that introducing privacy into the TFM paradigm can substantially reduce computational cost and improve ease of use while retaining competitive predictive utility under privacy.
Traditional private learning methods enforce privacy through the optimisation procedure for every new dataset, which requires repeatedly fitting and privatising a new model.
PrivTab instead places the privacy mechanism inside the foundation-model architecture.
The model is pretrained once to perform prediction through a private summary, and fitting to a new sensitive dataset only requires constructing this summary through a single forward pass.
This amortises computational cost and allows the model to be used across privacy levels without retraining.

This architectural choice also makes the privacy boundary simpler and easier to audit.
Once the sensitive dataset has been converted into the private summary, all later predictions are post-processing and do not require further access to the sensitive records.
The private summary mechanism can therefore be inspected, its mathematical guarantees formally verified, and its implementation tested through targeted empirical audits.
The summary itself can also be stored and shared with other users for later predictions.
Moreover, because PrivTab is pretrained using the same privacy mechanism used at prediction, it learns how to make predictions in the presence of privacy noise rather than encountering privacy noise only during dataset-specific private training.

Several limitations remain.
We study only classification and not regression.
PrivTab's privacy model assumes one user per row, and extending the guarantee to settings in which one individual contributes multiple rows is less straightforward than for approaches based on DP stochastic optimisation.
PrivTab has a smaller relative advantage when privacy is weak or context sizes are large, as models trained separately on each dataset can benefit more from both the relaxed privacy constraint and additional data. 
Similar diminishing returns on larger context sizes have also been observed in earlier TFMs such as TabPFN v1, and future research could benefit from recent advancements in this area to improve PrivTab.
Finally, preprocessing must also be included in the privacy analysis in practice.
Our private-preprocessing experiments demonstrate one possible approach, but the appropriate preprocessing pipeline remains application-dependent.
The same preprocessing requirement also applies to the private-learning baselines.

More broadly, our results suggest that foundation models can be designed around privacy from the outset, rather than having privacy added retrospectively to every downstream training procedure.
PrivTab provides one example of this principle by learning once how to perform inference through a private representation of the sensitive dataset, while keeping the mechanism that accesses the sensitive data explicit and auditable.

\clearpage
\section{Methods}

\subsection{Differential privacy}
Differential privacy (DP) limits how much the output distribution of a randomised algorithm can change when one person's record changes~\citep{Dwork2006DP}.
Let \(\calD=\{\bz_i\}_{i=1}^{n}\) be a dataset of individual records and \(\calA:\mathbb D\to\mathcal O\) a randomised mechanism from datasets $\mathbb{D}$ to an output space $\mathcal{O}$; in supervised learning, \(\bz_i\) can be an input--label pair $(\bx, y)$.
Two datasets $\mathcal{D}$ and $\mathcal{D}'$ are neighbouring or adjacent, denoted \(\calD\simeq\calD'\), if they have the same size and differ in at most one record $\bz_i$.
This is known as the \emph{replace-one} or \emph{substitute} adjacency relation.

We report privacy in the Gaussian DP (GDP) parameterisation~\citep{Dong2022GDP}.
A mechanism satisfies \(\mu\)-GDP when its outputs on neighbouring datasets are no easier to distinguish than \(\mathcal N(0,1)\) and \(\mathcal N(\mu,1)\) in a hypothesis test; smaller \(\mu\) means stronger privacy.
For a vector-valued function \(f\), the Gaussian mechanism \(\calA(\calD)=f(\calD)+\mathcal N(\mathbf 0,\sigma^2\mathbf I)\) has \(\mu=\Delta(f)/\sigma\), where \(\sigma\) is the noise standard deviation, \(\mathbf I\) is the identity matrix of the output dimension, and the \(\ell_2\)-sensitivity is
\(\Delta(f)=\sup_{\calD\simeq\calD'}\|f(\calD)-f(\calD')\|_2\).
Thus, bounding a one-row change in \(f\) determines the noise needed for a target privacy level.
The adaptive composition of mechanisms with GDP parameters \(\mu_1,\ldots,\mu_k\) has parameter
\begin{equation}
    \mu_{\mathrm{total}}^2=\sum_{i=1}^{k}\mu_i^2.
    \label{eq:methods-gdp-composition}
\end{equation}
Once a DP summary is released, computation using only that summary is post-processing and costs no additional privacy budget.
Formal definitions, conversion between GDP and \((\varepsilon,\delta)\)-DP, and the precise adaptive-composition statement are in \cref{app:differential-privacy}.

\subsection{Model architecture}
Neural processes~\citep{Garnelo2018NPs,Garnelo2018CNPs} use a labelled context set to make predictions at new query inputs without dataset-specific training, requiring only a single forward pass.
Transformers~\citep{Vaswani2017Transformers} process collections of input representations using self-attention, which allows each representation to be updated according to its relationships with the others.
Attentive and transformer neural processes~\citep{Kim2019ANPs,Nguyen2022TNPs} adapt this mechanism to conditional prediction, allowing query representations to attend to the labelled context and extract information relevant to each prediction.
Latent-bottleneck attentive neural processes (LBANP)~\citep{Feng2023LBANP} instead first compress the context into a fixed-size set of latent representations, which provides an intermediate summary through which information passes before predictions are made at the queries.

\paragraph{High-level architecture.}
PrivTab follows the structure of the LBANP but makes the context-to-summary operation differentially private (\cref{fig:architecture}).
Each row has a feature vector \(\bx_i=(x_{i1},\ldots,x_{id})\), where \(d\) is the number of input features and \(x_{ik}\) is the value of feature \(k\) in row \(i\).
Let \(\calD_c=\{(\bx_i^c,y_i^c)\}_{i=1}^{n_c}\) be the labelled context and \(X_t=(\bx_1^t,\ldots,\bx_{n_t}^t)\) the query inputs.
During pretraining and evaluation, \(\calD_t=\{(\bx_j^t,y_j^t)\}_{j=1}^{n_t}\) also contains target labels, which are withheld from the predictor and used only for the learning objective or evaluation.

The row embedding model \(\mathrm{RowEmbed}\) converts each row into a fixed-width token, using the observed label for context rows and an unknown-label marker for queries.
It maps each context row $(\bx_i^c, y_i^c)$ to
\begin{equation}
\bz_i^c=\mathrm{RowEmbed}\left(\hat\bx_i^c,\mathrm{LabelEmbed}(y_i^c)\right)
\end{equation}
and each query $\bx_j^t$ to
\begin{equation}
\bz_j^{t,(0)}=\mathrm{RowEmbed}\left(\hat\bx_j^t,\mathrm{LabelEmbed}(y_{\mathrm{null}})\right).
\end{equation}
Here \(\hat\bx\) is the padded, fixed-width feature representation of $\bx$, \(\mathrm{LabelEmbed}\) is a learned label embedding, and \(\mathrm{RowEmbed}\) is the row-wise embedding network.
The query uses the dedicated unknown-label index \(y_{\mathrm{null}}\); see \cref{app:dp-tnp} for the exact construction.
Each row token has width \(d_{\mathrm{tok}}=256\).
The context embeddings pass through the private summary module to produce \(\tilde S\); the prediction module then combines each query embedding with this summary to obtain label probabilities:
\begin{align}
    \tilde S
    &=\mathrm{PrivateSummary}_{\bphi}(\bz_{1:n_c}^c;\mu), \nonumber\\
    q_{\bpsi}(y\mid\bx_j^t,\tilde S)
    &=\bigl[\mathrm{Prediction}_{\bpsi}(\bz_j^{t,(0)};\tilde S)\bigr]_y.
    \label{eq:methods-module-flow}
\end{align}
Here $[\cdot]_y$ denotes the selection of the probability for label $y$, and \(\bphi\) and \(\bpsi\) denote the fixed summary and prediction parameters, respectively, and \(\mu\) is the requested privacy level.
The model has 9.3 million parameters, with 8.9 million in the private summary module and 0.4 million additional parameters in the prediction module.

\paragraph{Private summary module.}
The private summary module compresses the labelled context into a fixed-size collection of noisy summaries.
The module starts from \(m=128\) \emph{pseudo-tokens}, \(\bu_1,\ldots,\bu_m\), learned during pretraining and subsequently fixed.
Each of the \(L_{\mathrm{main}}=3\) main layers first applies private multi-head cross-attention (DP-MHCA) to the context embeddings and then multi-head self-attention (MHSA)~\citep{Vaswani2017Transformers} to the resulting summary tokens.
Denote $\bu_{1 : m}$ as a short-hand for $\bu_{1}, \ldots, \bu_{m}$, and let \(\tilde{\bs}_{1:m}^{(0)}=\bu_{1:m}\), then for each layer $\ell=1,\ldots,L_{\mathrm{main}}$,
\begin{align}
    \bs_{1:m}^{(\ell)}
    &=\mathrm{DP\text{-}MHCA}_{\mathrm{block}}^{(\ell)}(\tilde{\bs}_{1:m}^{(\ell-1)};\bz_{1:n_c}^c), \nonumber\\
    \tilde{\bs}_{1:m}^{(\ell)}
    &=\mathrm{MHSA}_{\mathrm{block}}^{(\ell)}(\bs_{1:m}^{(\ell)}).
    \label{eq:methods-layer-flow}
\end{align}
The state \(\bs_{1:m}^{(\ell)}\) is the intermediate output of the private cross-attention block; \(\tilde{\bs}_{1:m}^{(\ell)}\) is the saved summary after the self-attention block.
Another \(L_{\mathrm{post}}=2\) layers refine the summary using only self-attention:
\begin{equation}
    \tilde{\bs}_{1:m}^{(\ell)}
    =\mathrm{MHSA}_{\mathrm{block}}^{(\ell)}(\tilde{\bs}_{1:m}^{(\ell-1)}),
    \qquad \ell=L_{\mathrm{main}}+1,\ldots,L.
\end{equation}
Here \(L=L_{\mathrm{main}}+L_{\mathrm{post}}=5\), and the module returns the collection of all saved states,
\begin{equation}
    \tilde S=\left(\tilde{\bs}_{1:m}^{(1)},\ldots,\tilde{\bs}_{1:m}^{(L)}\right)
    =\mathrm{PrivateSummary}_{\bphi}(\bz_{1:n_c}^c;\mu).
\end{equation}
The operators \(\mathrm{DP\text{-}MHCA}_{\mathrm{block}}\) and \(\mathrm{MHSA}_{\mathrm{block}}\) denote complete blocks, including a token-wise feed-forward network; their details are in \cref{app:dp-tnp-architecture-details}.
Only DP-MHCA reads the context embeddings; self-attention and the remaining summary updates are post-processing of its private outputs.

\paragraph{DP-MHCA layer.} DP-MHCA replaces the usual softmax attention~\citep{Vaswani2017Transformers} with a sum of bounded contributions and adds Gaussian noise to protect individual context rows.
For each layer \(\ell\) and head \(h\), learned matrices project the summary token \(\tilde{\bs}_i^{(\ell-1)}\) into a query and the context token \(\bz_j^c\) into a key and a value:
\begin{align}
    \bq_i^{(h)}
    &=\mathbf{W}_Q^{(\ell,h)}\tilde{\bs}_i^{(\ell-1)}, \nonumber\\
    \bk_j^{(h)}
    &=\mathbf{W}_K^{(\ell,h)}\bz_j^c, \nonumber\\
    \bar{\bv}_j^{(h)}
    &=\mathbf{W}_V^{(\ell,h)}\bz_j^c.
    \label{eq:methods-dp-projections}
\end{align}
Here each projection matrix lies in \(\mathbb{R}^{d_h\times d_{\mathrm{tok}}}\).
In the first layer, \(\tilde{\bs}_i^{(0)}=\bu_i\); later layers use summary tokens obtained by post-processing previous private outputs.
For each summary token \(i\), the attention output before normalisation is
\begin{equation}\label{eq:simple-dp-attention}
\begin{aligned}
    \mathbf{a}_i^{(h)}
    &=\sum_{j=1}^{n_c}
    \tanh\!\left(\frac{\langle\bq_i^{(h)},\bk_j^{(h)}\rangle}{\sqrt{d_h}}\right)
    \frac{\bar{\bv}_j^{(h)}}{\max\{\lVert\bar{\bv}_j^{(h)}\rVert_2,\varepsilon_v\}}, \\
    \tilde{\bs}_{i,\mathrm{attn}}
    &=\left[\mathbf{a}_i^{(1)},\ldots,\mathbf{a}_i^{(H)}\right]+\boldsymbol{\xi}_i.
\end{aligned}
\end{equation}
Here \(\mathbf{a}_i^{(h)}\) is the pre-noise head output, \(d_h\) is the per-head key dimension, and brackets denote concatenation across heads.
The noise vectors \(\boldsymbol{\xi}_i\sim\mathcal{N}(\mathbf{0},\sigma^2\mathbf{I}_{Hd_h})\) are independent across summary tokens and layers; the noise scale \(\sigma\) is calibrated below.
The numerical stabiliser is \(\varepsilon_v=10^{-12}\).
Keys and values are computed row-wise, so each summand depends on one context row.

Normalising the noisy attention outputs controls their scale before the residual update.
We applied this step in certain configurations, especially during evaluation, by rescaling with the largest token norm:
\begin{equation}
    \tilde{\bs}_{i,\mathrm{attn}} =
    \frac{G\tilde{\bs}_{i,\mathrm{attn}}}
    {\max_{k=1,\ldots,m}\lVert\tilde{\bs}_{k,\mathrm{attn}}\rVert_2+\varepsilon_s},
    \label{eq:methods-summary-normalisation}
\end{equation}
where \(G=512\) and \(\varepsilon_s=10^{-6}\) in the reported experiments.
This preserves the relative magnitudes of the attention outputs while controlling their overall scale, which substantially improved performance on longer contexts (\cref{tab:supplement-ablation-normalization-combination}).
These normalised attention outputs precede the residual and feed-forward updates in the private block and the subsequent self-attention block; the saved states \(\tilde{\bs}_{1:m}^{(\ell)}\) include those operations.
Applied after noise addition, this normalisation is post-processing and does not affect the privacy guarantee.

\paragraph{Privacy guarantees.}
The following theorem states the privacy guarantee for the context dataset and all predictions made from its private summary.

\begin{privacytheorem}[Privacy of PrivTab]
Fix the pretrained parameters which are independent from \(\calD_c\).
For \(\mu>0\), let each of the \(L_{\mathrm{main}}\) DP-MHCA layers add independent noise from \(\mathcal{N}(0,\sigma^2)\) to every attention-output coordinate, where
\begin{equation}
    \sigma=\frac{2\sqrt{H m L_{\mathrm{main}}}}{\mu}.
    \label{eq:methods-noise-scale}
\end{equation}
Releasing the private summary \(\tilde S\) together with predictions for any number of public queries satisfies \(\mu\)-GDP with respect to \(\calD_c\) under replace-one adjacency.
\end{privacytheorem}

\emph{Justification.}
\Cref{thm:dp-mhca-sensitivity} bounds each layer's sensitivity of \cref{eq:simple-dp-attention} by \(2\sqrt{Hm}\), independently of context size.
With the noise scale in \cref{eq:methods-noise-scale}, each layer satisfies \(\mu/\sqrt{L_{\mathrm{main}}}\)-GDP, and \cref{thm:dp-tnp-gdp} establishes the overall \(\mu\)-GDP guarantee by adaptive composition.
Predictions from the released summary preserve this guarantee by post-processing~\citep{Dwork2006DP}.

\paragraph{Prediction module.}
The prediction module combines each query with the private summaries to obtain label probabilities.
At each layer, query tokens use standard multi-head cross-attention (MHCA)~\citep{Vaswani2017Transformers} to attend to the corresponding saved summary state:
\begin{equation}
    \bz_{1:n_t}^{t,(\ell)}
    =\mathrm{MHCA}_{\mathrm{block}}^{(\ell)}(\bz_{1:n_t}^{t,(\ell-1)};\tilde{\bs}_{1:m}^{(\ell)}),
    \qquad \ell=1,\ldots,L.
\end{equation}
Here \(\mathrm{MHCA}_{\mathrm{block}}\) denotes the complete query-to-summary block, including pre-normalisation, residual connections and a token-wise feed-forward network.
Each query is processed independently using only its embedding and the saved summary states.
After the \(L\) layers, a token-wise multilayer perceptron maps the final query representation \(\bz_j^{t,(L)}\) to a vector of raw class scores (logits) \(\hat{\mathbf y}_j\in\mathbb R^{C_{\mathrm{max}}}\), as in the appendix.
For a dataset with \(C\leq C_{\mathrm{max}}\) public classes indexed \(1,\ldots,C\), the probability assigned to class \(y\) is
\begin{equation}
    q_{\bpsi}(y\mid\bx_j^t,\tilde S)
    =\frac{\exp(\hat y_{j,y})}{\sum_{c=1}^{C}\exp(\hat y_{j,c})},
    \label{eq:methods-prediction-map}
\end{equation}
where \(\hat y_{j,c}\) is the score for class \(c\) in \(\hat{\mathbf y}_j\), and \(\bpsi\) denotes the predictive-network parameters.
The context and query paths share their row embedding weights, so these weights occur in both \(\bphi\) and \(\bpsi\).
The denominator makes these probabilities sum to one; the predicted class is the one with the largest probability.
Because the query-side computation depends on the context only through \(\tilde S\), predictions for any number of public queries are post-processing of the private release.

\paragraph{The effect of re-ordering rows and columns.} PrivTab is invariant to permutations of context rows and equivariant to permutations of target rows: reordering the context leaves predictions unchanged, while reordering targets reorders their predicted labels.
This follows because the model uses no positional encodings~\citep{Vaswani2017Transformers}, and every operation before the DP-MHCA layers is row-wise.
PrivTab, however, is neither permutation invariant nor equivariant to column ordering.
This is because each row is mapped to a fixed-size vector by padding it to a certain number of columns and then feeding it into a fully connected neural network ($\mathrm{RowEmbedding}$).
It learns during pretraining, similar to TabPFN and TabICL, to become less sensitive to the exact ordering of the columns through the randomness of the data simulator.

\subsection{Pretraining}

Following prior work, the model was pretrained once on simulated tabular classification datasets~\citep{Muller2022PFNs}.
We drew tasks from the mixed structural-causal-model simulator introduced for TabICL \citep{Qu2025TabICL}.
Tasks contained between 2 and 120 active features and at most 10 classes.
Each simulated task \(\btau=(\calD_c,\calD_t)\) has the context and labelled target sets defined above.
For each task and sampled privacy level \(\mu\), \(\tilde S=\mathrm{PrivateSummary}_{\bphi}(\bz_{1:n_c}^c;\mu)\) was produced by the same noisy summary mechanism used at prediction; only \(X_t\), not target labels, was supplied to the query-side network. More details on pretraining and ablations can be found in \cref{app:dp-tnp-pretraining} and \cref{app:ablations}.

\paragraph{Noise-aware objective.} PrivTab was trained to predict target labels from noisy private summaries, using the same summary mechanism as at prediction.
Let \(q_{\bpsi}(y\mid\bx,\tilde S)\) denote the predicted label probabilities given target features \(\bx\) and noisy private summary \(\tilde S\).
The main training objective is the target cross-entropy
\begin{equation}
    \mathcal L_{\mathrm{priv}}
    =\mathbb E_{\btau}\!\left[-\sum_{(\bx,y)\in\mathcal D_t}
    \log q_{\bpsi}(y\mid\bx,\tilde S)\right].
    \label{eq:methods-pretraining-loss}
\end{equation}
The expectation averages over simulated tasks and the Gaussian noise used to form each private summary.
Conditional on \(\tilde S\), the query rows are processed independently, so the joint probability is the product \begin{equation}
    q_{\bpsi}(y_{1:n_t}^t\mid X_t,\tilde S)=\prod_{j=1}^{n_t}q_{\bpsi}(y_j^t\mid\bx_j^t,\tilde S);    
\end{equation}
its negative log is the sum in \cref{eq:methods-pretraining-loss}.
We evaluated the effect of noise-aware pretraining in \cref{tab:supplement-ablation-noise-training}.
In contrast, privacy noise in DP stochastic optimisation acts through parameter updates; it is not supplied as an explicit conditioning input when the fitted model produces label probabilities.
The probabilistic interpretation of \cref{eq:methods-pretraining-loss}, including its relation to the posterior predictive conditioned on a private summary, is derived in \cref{app:noise-aware-amortized-inference}; \cref{thm:nll-noise-aware-predictive-kl} characterises the objective for a fixed private mechanism.

\paragraph{Training curriculum and model combination.} We pretrained two variants for small and large contexts, sampling privacy levels logarithmically and gradually introducing stronger noise.
Both started with 128,000 steps using context and target sizes of 100--2,048 rows, without summary normalisation.
The small-context variant continued with the same context range and no normalisation for 384,000 steps in total.
The large-context variant continued with contexts of 1,024--8,192 rows and summary normalisation for 256,000 steps in total.
Both used summary normalisation at inference.
We used the small-context variant for contexts below 4,096 rows and the large-context variant for all others.
This threshold was selected on held-out context--target pairs from the TabICL simulator used for pretraining.
The same two pretrained variants were used across all real datasets and privacy levels; changing \(\mu\) changes the summary noise while their weights remain fixed.

\paragraph{Optimisation.} We first pretrained the full model, then froze the private summary module late in pretraining and updated only the prediction module for the remaining steps.
We optimised cross-entropy using AdamW \citep{Loshchilov2017AdamW}, batches of 128 simulated datasets and gradient clipping.

\subsection{Private baselines}
We compared with DP-LR and DP-MLP, the latter with one hidden layer, both trained using DP Adam~\citep{avent2020DPAdam, bu2020DPAdam}.
Because the same training data are used in repeated, adaptive updates, the privacy loss must be composed over all \(T\) steps.
Under the same substitute adjacency as PrivTab, weW numerically computed the privacy profile of the Poisson-subsampled Gaussian mechanism, composed its privacy-loss distribution over \(T\) updates, and converted the resulting trade-off function to the smallest valid non-asymptotic \(\mu\)-GDP guarantee~\citep{Gomez2025GaussianDP}.
The noise scale was chosen by numerically inverting this accounting for the target training budget; we did not use the asymptotic subsampled-Gaussian GDP approximation.
The accountant and update rule are specified in \cref{app:dp-stochastic-optimisation-accounting}.

\paragraph{Model selection stages.}
Baseline model selection combined broad calibration on separate datasets with private learning-rate selection for each evaluation dataset.
First, we used Bayesian optimisation on eight classification datasets from prior work~\citep{Hegselmann2023TabLLM} to select broad architectural and optimiser settings separately for each baseline and privacy level.
Second, because a single learning rate was not reliable across all new datasets, we selected it privately for every evaluation dataset, baseline and privacy level.
The context dataset was divided into training and validation subsets.
Candidate models were trained privately on the training rows, evaluated with a private validation loss and compared under an explicitly composed learning rate selection budget.
We set the privacy parameter for learning-rate selection to
\(\mu_{\mathrm{sel}}=0.5\mu\) and chose the privacy parameter
\(\mu_{\mathrm{train}}\) for the final training run such that
\(\mu_{\mathrm{sel}}^2+\mu_{\mathrm{train}}^2=\mu^2\)
under GDP composition.

This protocol prevented the comparison from treating data-dependent model selection as free: the privacy cost of all candidate training and validation evaluations was accounted for within the fixed learning-rate selection budget.
It also illustrates why DP stochastic optimisation is a pipeline rather than a single private optimiser call.
Search grids, fixed hyperparameters and selected configurations are listed in \cref{app:baseline-details}.

\subsection{TabArena evaluation}

We compared predictive performance across a broad benchmark using the same datasets, splits and privacy levels for all methods.
We retrieved the 33 TabArena classification datasets from the TabArena-v0.1 benchmark suite on OpenML \citep{EricksonTabArena2025,OpenML2025}, selecting tasks with at most 100,000 rows and 120 features.
This is the same number of compatible datasets from TabArena as TabPFNv2 and consistent with TabArena’s model-specific applicability constraints.
The row limit kept baseline evaluation computationally tractable, while the feature limit matched the PrivTab pretraining configuration used here.
Each dataset was evaluated at \(\mu\in\{0.05,0.1,0.2,0.4,0.8,1.6\}\) using ten independent 80/20 splits into \(\calD_c\) and a held-out \(\calD_t\); only \(X_t\) was supplied to the predictor.
All methods received the same numeric and categorical feature conversion.
To isolate the private prediction mechanism in this broad benchmark, features were standardised using context-set statistics under a common idealised preprocessing convention.
We evaluated the cost of private preprocessing separately, as described below.

\paragraph{Metrics and rankings.} We evaluated classification with AUC and label-probability quality with log loss, and summarised model comparisons using Elo rankings.
AUC was computed directly for binary tasks and macro-averaged one-versus-rest for multiclass tasks.
Log loss (cross-entropy loss) penalises low probabilities assigned to true labels, especially confident incorrect predictions.
Following TabArena~\citep{EricksonTabArena2025}, maximum-likelihood Elo rankings used AUC for binary tasks and log loss for multiclass tasks.
Ratings were recalibrated so that DP-LR scored 1,000 at each privacy level and in the aggregate comparison.
For dataset-level breakdowns, we averaged the AUC differences between PrivTab and each specified private baseline over the ten splits.

\paragraph{Uncertainty intervals.} We quantified performance uncertainty using 95\% confidence intervals (CIs); unless stated otherwise, figures and tables report percentile-bootstrap CIs based on 10,000 replicates.
The sampling unit was the dataset for benchmark aggregates and the context--target split for within-dataset evaluations.
Record-wise membership-AUC intervals were instead computed using the closed-form standard-error approximation described below.

\paragraph{Fitting time.} Runtime measures the GPU time needed to fit the context dataset.
We timed each method on one Nvidia V100 using identical software environments and hardware settings, including data transfer and computation.
PrivTab required one forward pass to form the private summary; DP Adam baseline timings included private learning-rate selection and training.
We reported the median per-context--target-split runtime across datasets, with a percentile-bootstrap 95\% CI over datasets.

\paragraph{Prediction time.} We separately measured post-fit prediction time on a laptop with a 13th-generation Intel Core i5 CPU across 33 datasets, six privacy levels and ten context--target splits.
Predictions were evaluated in batches of up to 1,024 target rows, using one warm-up call and three timed runs per configuration.
Timings exclude preprocessing, model loading, fitting and private-summary construction.
The complete dataset-level results appear in \cref{fig:inference-timing,tab:supplement-inference-timing}.

\subsection{Private preprocessing in practice}

We evaluated prediction together with private standardisation for Bank Marketing, COMPAS, ACS Income, Pima Diabetes and Maternal Health Risk datasets.
Bank Marketing and Pima Diabetes were accessed through OpenML and originated in the UCI Machine Learning Repository; Maternal Health Risk was fetched directly from UCI \citep{OpenML2025,Kelly2023UCI}.
COMPAS was accessed through OpenML, and ACS Income is a Folktables task derived from U.S. Census American Community Survey data \citep{OpenML2025,Ding2021RetiringAdult}.
There are no missing values in these datasets, so that no data imputation was needed.

\paragraph{Private standardisation.}
We standardised context and query features using private estimates of means and variances, with public clipping bounds limiting each row's influence.
For each of the \(d\) input features, indexed by \(k=1,\ldots,d\), we selected a finite clipping interval \([a_k,b_k]\): numeric-feature intervals came from public descriptions, and categorical variables used publicly known ranges.
For context and query rows alike $\bx_i = (x_{i1}, \ldots, x_{id})$, indexed here by \(i\), the clipped value was rescaled to
\begin{equation}
    c_{ik}=\operatorname{clip}_{[0,1]}\!\left(\frac{x_{ik}-a_k}{b_k-a_k}\right).
    \label{eq:methods-prep-clipping}
\end{equation}
Here, \(\mathbf c_i=(c_{i1},\ldots,c_{id})\) denotes the rescaled and clipped row. The bounds determine the scale and cap each feature value before any private statistic is computed.
From the \(n_c\) context rows only, we released noisy vectors of sums and sums of squares,
\begin{equation}
    \tilde{\mathbf m}_1=\sum_{i=1}^{n_c}\mathbf c_i+\boldsymbol\xi_1,
    \qquad
    \tilde{\mathbf m}_2=\sum_{i=1}^{n_c}\mathbf c_i^{\odot 2}+\boldsymbol\xi_2,
    \label{eq:methods-prep-moments}
\end{equation}
Here the superscript \(\odot 2\) means squaring each feature value, and the independent Gaussian noise vectors \(\boldsymbol\xi_1\) and \(\boldsymbol\xi_2\) have standard deviations for each feature of \(\sqrt d/\mu_1\) and \(\sqrt d/\mu_2\), respectively.
Replacement of one row changes either sum vector by at most \(\sqrt d\) in Euclidean norm.
We allocated \(\mu_1=\mu_{\mathrm{prep}}/2\) and \(\mu_2=\sqrt{\mu_{\mathrm{prep}}^2-\mu_1^2}\), so these two releases compose to \(\mu_{\mathrm{prep}}\)-GDP.
Writing \(\tilde m_{1k}\) and \(\tilde m_{2k}\) for the entries of these vectors at feature \(k\), the mean and variance estimates are \(\tilde m_k=\tilde m_{1k}/n_c\) and \(\tilde v_k=\tilde m_{2k}/n_c-\tilde m_k^2\).
When \(\tilde v_k\geq10^{-6}\) and the interval is non-constant, the input z-score was
\begin{equation}
    z_{ik}=
    \frac{c_{ik}-\tilde m_k}{\sqrt{\tilde v_k}},
    \label{eq:methods-prep-zscore}
\end{equation}
otherwise that feature value was set to zero.

\paragraph{Privacy guarantees with private standardisation.}
When clipping bounds are public, the full pipeline's privacy guarantee follows by composing private standardisation with PrivTab's private summary mechanism or the baseline's private training procedure.
The private standardisation in \cref{eq:methods-prep-zscore} was applied for both context and query rows using the released statistics from the context rows in \cref{eq:methods-prep-moments}.
Query transformation is post-processing; the context then enters PrivTab's private summary mechanism or the baseline's private training procedure.
Conditional on the released moments, standardisation remains row-wise, preserving one-row adjacency for the next private mechanism.
Clipping bounds must also be chosen without unaccounted access to sensitive data.
Our LLM-derived bounds used only public feature descriptions.
For comparison, oracle bounds used empirical dataset minima and maxima; this workflow is not fully private because bound selection accesses sensitive data without privacy accounting.

\paragraph{Privacy-budget allocation for the case study.} The case-study privacy budget covered both preprocessing and prediction or baseline fitting.
The total privacy level was \(\mu_{\mathrm{total}}=0.4\).
We allocated \(\mu_{\mathrm{prep}}=0.24\) to the two private moment releases above and \(\mu_{\mathrm{pred}}=0.32\) to PrivTab's private summary or to the complete baseline training and model-selection pipeline, giving
\[
\mu_{\mathrm{total}}
=
\sqrt{\mu_{\mathrm{prep}}^2+\mu_{\mathrm{pred}}^2}
=
0.4
\]
under GDP composition.
Results were based on ten independent 80/20 context--target splits.
The private moment mechanism, prompts, responses, exact clipping intervals and complete numerical results are given in \cref{app:end-to-end-details}.

\subsection{Membership-inference audit}

\paragraph{Attack setup.}
We measured membership exposure by testing whether model predictions reveal that a particular row was included in the context.
We used a likelihood-ratio framework~\citep{Carlini2022MembershipInference} on all 1,014 Maternal Health Risk rows.
We generated 4,000 balanced context draws, each containing exactly 507 rows, so that every row appeared as a member in exactly 2,000 contexts.
For every draw, the model's true-label confidence log-odds was recorded for all rows, using the same membership masks for every method.

\paragraph{Attack metrics.}
Record AUC holds the row fixed and measures the fitted Gaussian distinguishability across all contexts in which that row was or was not a member.
We computed a closed-form approximation to the standard error of each record-AUC estimate~\citep{Knolle2026DisparatePrivacyRisks} and defined its lower 95\% confidence endpoint as \(L_i=\max\{0.5,\widehat{\mathrm{AUC}}_i-1.96\,\mathrm{SE}_i\}\).
The Gaussian reference trade-off implies the attack-AUC upper bound $\Phi(\mu/\sqrt{2})$, where $\Phi$ is the standard normal cumulative distribution function (CDF).

We also reported balanced attack accuracy to summarise how reliably membership can be inferred.
We calculated balanced accuracy from fitted Gaussian member (IN) and non-member (OUT) score distributions with variances adjusted by a finite-population correction (FPC)~\citep{Jalko2026ReliableMIA}, rather than from observed classification counts.
We used the delta method to estimate its standard error at the fitted accuracy-maximising threshold, formed a nominal 95\% CI and plotted the largest lower endpoint across records for each displayed method.
The calculation and its assumptions are detailed in \cref{app:membership-inference-results}.
Complementary point-estimate summaries are reported in \cref{tab:membership-inference}.
These attacks measure empirical exposure but cannot establish DP; they complement the mechanism-level guarantee and audit.

\paragraph{Finite-population correction.} Sampling contexts without replacement from a fixed pool makes the empirical IN/OUT variance smaller than the variance under independent population draws.
We used a finite-population factor \(f=1-N/N_{+}\)~\citep{Jalko2026ReliableMIA} and replaced each empirical fitted standard deviation by \(\hat\sigma_{\mathrm{FPC}}=\hat\sigma/\sqrt{f}\), leaving the fitted means unchanged.
Here \(N=507\), \(N_{+}=1{,}014\) and \(f=0.5\), so the standard deviations were inflated by \(\sqrt{2}\).
The corrected standard deviations were used both for the analytical record AUCs and for the held-out likelihood-ratio attack.

\subsection{Verification and auditing}

We assessed the private summary mechanism using formal verification, implementation checks and empirical auditing.

\paragraph{Formal verification.}
The Lean 4 formalisation verifies the sensitivity bound and adaptive GDP composition underlying the noise calibration of the private summary mechanism.
It therefore verifies these central mathematical steps independently of the manuscript proof, while the connection between the formal objects and the running system remains subject to implementation review. The lean implementation with the complete description is available in the source code at \url{https://github.com/TrustworthyMLHelsinki/PrivTab/tree/main/privtab-lean}.

\paragraph{Implementation checks.}
Automated tests checked that the implementation applied every transformation before DP-MHCA independently to each context row, without introducing unaccounted dependencies between rows. The implementation checks are available in the source code at \url{https://github.com/TrustworthyMLHelsinki/PrivTab/blob/main/experiments/audit_rowwise.py}. The test suite also runs this automatically.

\paragraph{Empirical auditing.}
We empirically tested the sensitivity bound and the distinguishability of private outputs.
We used gradient-based optimisation to search for adjacent context datasets that maximised the distance between their pre-noise summaries.
We then repeatedly sampled privatised outputs and applied the optimal likelihood-ratio test for each fitted Gaussian pair, using its two noise-free means and calibrated variance.
Agreement with the expected trade-off curve is evidence against several implementation errors, but it is not a substitute for the proof.
Further details are provided in \cref{app:empirical-audits}. The empirical auditing checks are available in the source code at \url{https://github.com/TrustworthyMLHelsinki/PrivTab/blob/main/experiments/dp_audit.py}. The test suite also runs this automatically.

\subsection{Code availability}
The source code used to implement, train, and evaluate PrivTab, together with the pretrained model weights, is available at \url{https://github.com/TrustworthyMLHelsinki/PrivTab}.
The repository contains complete documentation of the code, including instructions for reproducing the experiments.

\clearpage

\bibliographystyle{unsrtnat}
\begingroup\raggedright
\bibliography{main.bib}
\endgroup

\section*{Acknowledgements}
This work was supported by the Research Council of Finland (Finnish Center for Artificial Intelligence, FCAI, Grant 356499 and Grant 359111), the Strategic Research Council at the Research Council of Finland (Grant 358247) as well as the European Union (Project 101070617). 
Richard E. Turner is supported by the EPSRC Probabilistic AI Hub (EP/Y028783/1).
Cristiana Diaconu is supported by the Cambridge Trust Scholarship.
Views and opinions expressed are however those of the author(s) only and do not necessarily reflect those of the European Union or the European Commission. 
Neither the European Union nor the granting authority can be held responsible for them. This work has been performed using resources provided by the CSC– IT Center for Science, Finland (Projects 2003275 \& 462001244). 
We also acknowledge CSC for access to the LUMI supercomputer, owned by the EuroHPC Joint Undertaking, hosted by CSC (Finland) and the LUMI consortium. 
The authors acknowledge the research environment provided by ELLIS Institute Finland.
We thank Aki Rehn for his assistance in debugging the GPU-based pretraining of PrivTab.
We also thank Joonas Jälkö, Mikko A. Heikkilä and Marlon Tobaben for their discussions and comments throughout the project.

\clearpage
\section*{Extended figures}
\setcounter{figure}{0}
\renewcommand{\thefigure}{E\arabic{figure}}
\renewcommand{\theHfigure}{extended.\arabic{figure}}

\begin{figure}[!ht]
\centering
    \includegraphics[width=\textwidth]{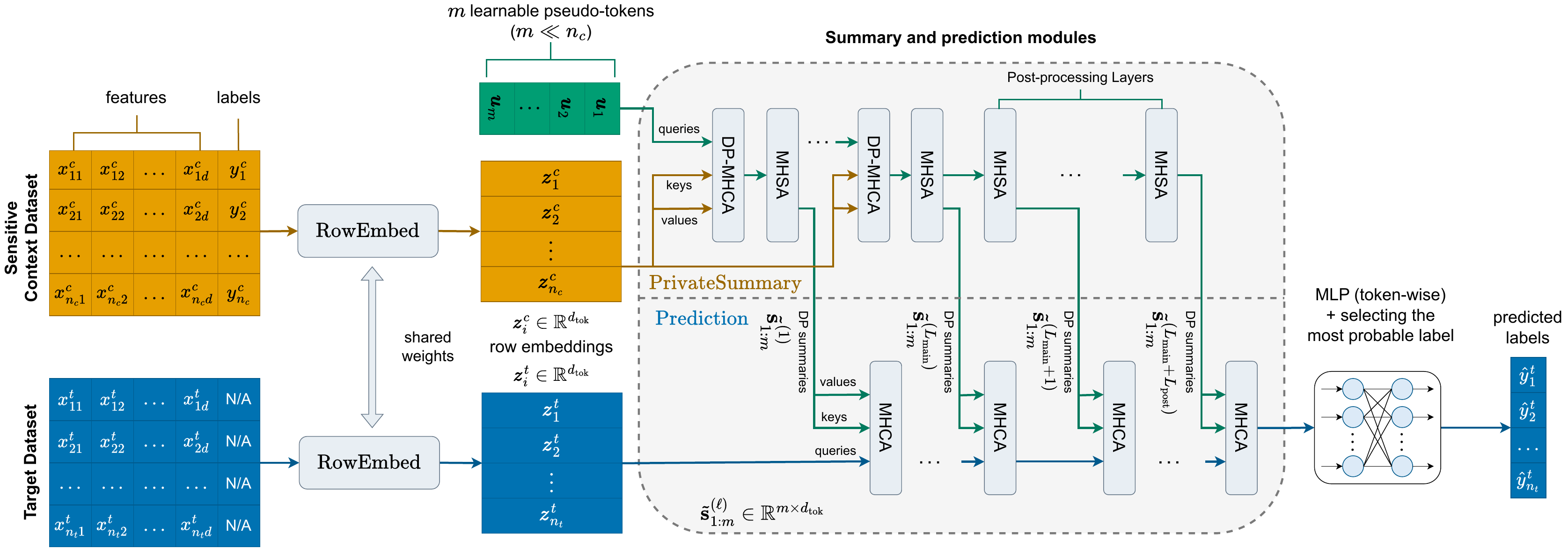}
\caption{Detailed PrivTab architecture.
Context and target rows are first mapped to row embeddings by \(\mathrm{RowEmbed}\).
In \(\mathrm{PrivateSummary}\), a small set of data-independent learnable pseudo-tokens cross-attends to the sensitive context rows through DP-MHCA, producing privatised summaries.
The public \(\mathrm{Prediction}\) module reuses these summaries through cross-attention and token-wise multilayer perceptrons.
DP-MHCA is the only operation that transfers context information into the released summaries; target-side computation is post-processing.}
\label{fig:architecture}
\end{figure}
\FloatBarrier
\clearpage

\section*{Additional information}
Correspondence: Antti Honkela (\nolinkurl{antti.honkela@helsinki.fi}).

\clearpage
\appendix
\section*{Supplement}
\setcounter{section}{0}
\setcounter{figure}{0}
\setcounter{table}{0}
\setcounter{equation}{0}
\renewcommand{\thesection}{S\arabic{section}}
\renewcommand{\thefigure}{S\arabic{figure}}
\renewcommand{\thetable}{S\arabic{table}}
\renewcommand{\theequation}{S\arabic{equation}}
\renewcommand{\theHsection}{supplement.\arabic{section}}
\renewcommand{\theHfigure}{supplement.\arabic{figure}}
\renewcommand{\theHtable}{supplement.\arabic{table}}
\renewcommand{\theHequation}{supplement.\arabic{equation}}
\setcounter{secnumdepth}{3}

\section{Detailed Background}\label{app:background}

\subsection{Notation}
Let \(d_x \in \mathbb{N}\) denote the maximum input dimension, let \(\calX := \bigsqcup_{d=1}^{d_x}\mathbb{R}^{d}\) denote the raw input space, and let $\calY \subset \mathbb{N}$ denote the label space.
An input-label pair is a tuple $(\bx, y) \in \calX \times \calY$.
Let $\mathbb{D}$ denote the space of all finite multisets of input-label pairs in $\calX \times \calY$, that is,
\begin{equation}
    \mathbb{D} = \bigcup_{n=0}^{\infty} \left(\left(\calX \times \calY\right)^n \middle/ \sim \right),
\end{equation}
where two elements of \(\left(\calX \times \calY\right)^n\) are equivalent under \(\sim\) if one is a permutation of the other.
A dataset $\calD \in \mathbb{D}$ is a finite collection of input-label pairs, i.e.
$\mathcal{D} = \{ (\bx_i, y_i) \}_{i = 1}^{N}$ for some $N \geq 0$.
For $n \geq 1$, we will write a sequence $z_1, \ldots, z_n$ briefly as $z_{1:n}$.
Whenever indexed notation is used for a multiset, it denotes an arbitrary ordered representative; all dataset-level constructions considered below are invariant to the choice of representative.

\subsection{Differential Privacy}\label{app:differential-privacy}
DP formalises the idea of protecting individuals' sensitive data when releasing statistics or training machine learning models.
It is achieved by limiting how much the algorithm's output distribution can change when one individual's data is added, removed, or modified.

\begin{definition}[Approximate DP; \citealt{Dwork2006DP,Dwork2016DP}]
\label{def:DP}
    A randomised algorithm or mechanism $\calA$ is $(\varepsilon, \delta, \simeq)$-DP under a binary relation $\simeq$ over $\mathbb{D}$, if for any two datasets $\calD, \calD' \in \mathbb{D}$ such that $\calD \simeq \calD'$, and for any subset $\calS$ of possible outcomes of $\mathcal{A}$, denoted $\calS \subseteq \mathrm{Range}(\calA)$,
    \begin{align}\label{eq:epsilon-delta-dp}
        \mathrm{Pr}\left[\calA(\calD) \in \calS\right] \leq \exp{(\varepsilon)} \times \mathrm{Pr}\left[\calA(\calD') \in \calS\right] + \delta.
    \end{align}
    Any two datasets $\calD \simeq \calD'$ are said to be adjacent datasets and $\simeq$ is called the adjacency relation.
\end{definition}

Applying any (randomised) function to the output of a DP mechanism is called \emph{post-processing}, and it does not weaken the privacy guarantees of the mechanism.
In this work, we use the substitute adjacency relation, that is, $\calD \simeq \calD'$ if for some $i^{*}$, $\calD' = \left(\calD \setminus \{ (\bx_{i^{*}}, y_{i^{*}}) \} \right) \cup \{ (\bx_{i^{*}}', y_{i^{*}}') \}$.

\subsection{Gaussian Differential Privacy}
GDP provides a more interpretable characterisation of privacy in terms of hypothesis testing.
For a random output $O$ of the mechanism $\calA$, consider the binary hypothesis test: \( H_0: O \sim \calA(\calD) \) vs.  \( H_1: O \sim \calA(\calD') \).
A test is a function
\begin{equation}
    \phi: \mathrm{Range}(\calA) \to [0,1],
\end{equation}
where for an observation $\bo$, \( \phi(\bo) = 1 \) means rejecting the null hypothesis, and \( \phi(\bo) = 0 \) means not rejecting it.
The \emph{type I error} is denoted \(\alpha_{\calA}(\phi; \calD) = \mathbb{E}_{O \sim \calA(\calD)}\left[\phi(O)\right]\) and the \emph{type II error} is denoted \( \beta_{\calA}(\phi; \calD') = 1 - \mathbb{E}_{O \sim \calA(\calD')}\left[\phi(O)\right] \).
The \emph{trade-off function} of $\calA$ is defined as:
\begin{equation}\label{eq:tradeoff-fun-def}
    T_{\calA}(\alpha; \calD, \calD') = \inf_{\phi \,:\, \alpha_{\calA}(\phi; \calD) \leq \alpha} \beta_{\calA}(\phi; \calD').
\end{equation}
\begin{definition}[$\mu$-GDP; \citealt{Dong2022GDP}]
    A randomised mechanism \( \calA \) satisfies $\mu$-GDP if for any two adjacent datasets \( \calD \simeq \calD' \), and for all \( \alpha \in [0,1] \),
    \begin{equation}\label{eq:mu-gdp-def-eq-1}
        T_{\calA}(\alpha; \calD, \calD') \ge \Phi\left(\Phi^{-1}(1 - \alpha) - \mu\right) =: T_\mu(\alpha),
    \end{equation}
    where \( \Phi \) denotes the standard normal CDF.
\end{definition}
The function $T_\mu(\alpha)$ is the trade-off function of the hypothesis test $H_0: O \sim \mathcal{N}(0, 1)$ vs. $H_1: O \sim \mathcal{N}(\mu, 1)$.
In this work, we will use $\mu$-GDP to report privacy guarantees.
A $\mu$-GDP mechanism satisfies $(\varepsilon, \delta(\varepsilon))$-DP for all $\varepsilon > 0$~\citep{Dong2022GDP}, where
\begin{equation}\label{eq:mu-gdp-to-eps-delta}
    \delta(\varepsilon) = \Phi\left(-\frac{\varepsilon}{\mu} + \frac{\mu}{2}\right)
- \exp(\varepsilon) \Phi\left(-\frac{\varepsilon}{\mu} - \frac{\mu}{2}\right).
\end{equation}
Applying a DP mechanism multiple times, or applying multiple DP mechanisms to the same dataset, is called composition.
If later mechanism applications use outputs of earlier ones in addition to the dataset, the composition is said to be \emph{adaptive}.
GDP admits a tight (adaptive) composition rule determined by the privacy guarantees of individual mechanisms.
\begin{theorem}[\citealt{Dong2022GDP}]\label{thm:mu-gdp-composition}
    The (adaptive) composition of $\mu_i$-GDP mechanisms ($i = 1, \ldots, T$) is $\mu$-GDP, where
    \begin{equation}
        \mu = \sqrt{\mu_1^2 + \ldots + \mu_T^2}.
    \end{equation}
\end{theorem}
Here, tight means that this composition rule is exact for $\mu$-GDP: in general one cannot replace $\sqrt{\mu_1^2 + \cdots + \mu_T^2}$ by a uniformly smaller value and still retain a valid guarantee for all adaptive compositions of $\mu_i$-GDP mechanisms.
The Gaussian mechanism is central to this work.
Given a function $f: \mathbb{D} \rightarrow \mathbb{R}^{d}$, define
\begin{equation}\label{eq:gaussian-mechanism}
    \calA_{\mathrm{gauss}}(\calD) := f(\calD) + \mathcal{N}(0, \sigma^2 \mathbf{I}_d),
\end{equation}
where $\mathbf{I}_d$ is the $d\times d$ identity matrix.
The privacy guarantee of $\calA_{\mathrm{gauss}}$ is determined by $\sigma$ and the $\ell_2$-sensitivity of $f$,
\begin{equation}
    \Delta(f) = \sup_{\calD \simeq \calD'} \lVert f(\calD) - f(\calD') \rVert_2,
\end{equation}
according to the following theorem.
\begin{theorem}[\citealt{Dong2022GDP}]
    The Gaussian mechanism $\calA_{\mathrm{gauss}}$ is $\mu$-GDP with $\mu = \frac{\Delta(f)}{\sigma}$.
\end{theorem}

\subsection{Bayesian Inference}

Assume that a dataset  \(\mathcal{D} = \{(\bx_i^{\mathrm{obs}}, y_i^{\mathrm{obs}})\}_{i = 1}^{N} \in \mathbb{D}\) consists of observed input-output pairs.
We treat the inputs \(\bx_{1:N}^{\mathrm{obs}}\) as given, and assume that the labels are generated from an underlying conditional data-generating distribution $p(y \mid \bx, \btheta)$, depending on some parameter \(\btheta \in \Theta\).
The likelihood of the observed labels, conditioned on the observed inputs, is
\begin{equation}
    p(y_{1:N}^{\mathrm{obs}} \mid \bx_{1:N}^{\mathrm{obs}}, \btheta)
    =
    \prod_{i = 1}^{N}
    p(y_i^{\mathrm{obs}} \mid \bx_i^{\mathrm{obs}}, \btheta).
\end{equation}
Given a prior \(p(\btheta)\), one can compute the posterior using Bayes' theorem:
\begin{equation}\label{eq:posterior}
    p(\btheta \mid \calD)
    =
    \frac{
        p(y_{1:N}^{\mathrm{obs}} \mid \bx_{1:N}^{\mathrm{obs}}, \btheta)
        p(\btheta)
    }{
        \int_\Theta
        p(y_{1:N}^{\mathrm{obs}} \mid \bx_{1:N}^{\mathrm{obs}}, \btheta')
        p(\btheta')
        \,\mathrm{d}\btheta'
    }.
\end{equation}
The posterior predictive can then be obtained by marginalising over the posterior:
\begin{equation}\label{eq:posterior-predictive}
    p(y \mid \bx, \calD)
    =
    \int_{\Theta}
    p(y \mid \bx, \btheta)
    p(\btheta \mid \calD)
    \,\mathrm{d}\btheta.
\end{equation}
The posterior predictive in \cref{eq:posterior-predictive} can be extended to multiple query points.
Assuming that the labels are conditionally independent given \(\bx_{1:n}\) and \(\btheta\), for \(n \geq 1\) we have
\begin{equation}
    \begin{split}
    p(y_{1:n} \mid \bx_{1:n}, \calD)
    &=
    \int_{\Theta}
    p(y_{1:n} \mid \bx_{1:n}, \btheta)
    p(\btheta \mid \calD)
    \,\mathrm{d}\btheta \\
    &=
    \int_{\Theta}
    \prod_{i = 1}^{n}
    p(y_i \mid \bx_i, \btheta)
    p(\btheta \mid \calD)
    \,\mathrm{d}\btheta.
    \end{split}
\end{equation}

\subsection{Stochastic Processes}
A stochastic process (SP) is a collection of random variables
\begin{equation}
    F = \{Y_{\bx}\}_{\bx \in \calX},
\end{equation}
indexed by an input space \(\calX\), where each \(Y_{\bx}\) takes values in an output space \(\calY\).
We also write \(F(\bx)\) for the random variable \(Y_{\bx}\).
Thus, for any \(\bx \in \calX\), the stochastic process defines a distribution over outputs in \(\calY\).
If this distribution admits a density, we denote by $p(y \mid \bx)$ the density of \(F(\bx)\) evaluated at \(y \in \calY\).

We denote by \(\mathcal{SP}(\calX,\calY)\) the set of stochastic processes indexed by \(\calX\) and taking values in \(\calY\).
A stochastic process \(F \in \mathcal{SP}(\calX,\calY)\) can be characterised by its finite-dimensional marginal distributions $p(y_{1:n} \mid \bx_{1:n})$, for all \(n \geq 1\), provided that these marginals satisfy the conditions of the Kolmogorov consistency theorem \citep{Oksendal2003SDE}.
That is, the finite-dimensional marginals must satisfy permutation invariance and marginalisation consistency.

\subsection{Amortised Bayesian Prediction}\label{app:amortised-bayesian-prediction}

Let \(F = \{Y_{\bx}\}_{\bx \in \calX} \in \mathcal{SP}(\calX,\calY)\) be a stochastic process.
Assume that its finite-dimensional marginal distributions are parameterised by some parameter \(\btheta \in \Theta\), so that for any \(n \geq 1\),
\begin{equation}
    Y_{\bx_1}, \ldots, Y_{\bx_n}
    \sim
    p(y_{1:n} \mid \bx_{1:n}, \btheta).
\end{equation}
A dataset $\mathcal{D} =\{(\bx_i^{\mathrm{obs}}, y_i^{\mathrm{obs}})\}_{i=1}^{N}$ can be viewed as observations of the stochastic process at input locations \(\bx_1^{\mathrm{obs}}, \ldots, \bx_N^{\mathrm{obs}}\).
Conditioned on these input locations and on \(\btheta\), the corresponding outputs are sampled as
\begin{equation}
    y_{1:N}^{\mathrm{obs}}
    \sim
    p(y_{1:N} \mid \bx_{1:N}^{\mathrm{obs}}, \btheta),
\end{equation}
where this distribution is the finite-dimensional marginal distribution of \(Y_{\bx_1^{\mathrm{obs}}}, \ldots, Y_{\bx_N^{\mathrm{obs}}}\).

Given a prior \(p(\btheta)\), Bayes' theorem gives the posterior \(p(\btheta \mid \mathcal{D})\).
The posterior predictive distribution for any \(n \geq 1\) is then
\begin{equation}
    p(y_{1:n} \mid \bx_{1:n}, \mathcal{D})
    =
    \int_{\Theta}
    p(y_{1:n} \mid \bx_{1:n}, \btheta)
    p(\btheta \mid \mathcal{D})
    \,\mathrm{d}\btheta.
\end{equation}
The collection of posterior predictive distributions $\left\{p(y_{1:n} \mid \bx_{1:n}, \mathcal{D})\right\}_{n \geq 1}$ defines a posterior stochastic process, denoted by \(F_{\mathcal{D}} \in \mathcal{SP}(\calX,\calY)\).

More generally, one can view prediction under uncertainty as learning an amortised approximation to the Bayesian map that sends an observed dataset to its posterior predictive law.
Given a dataset \(\mathcal{D}\), the ideal Bayesian map
\begin{equation}
    \pi_F : \mathbb{D} \to \mathcal{SP}(\calX,\calY)
\end{equation}
returns the posterior stochastic process \(F_{\mathcal{D}}\), whose finite-dimensional marginals are the posterior predictive distributions.
Amortised Bayesian prediction aims to learn a parametric map that directly approximates this posterior predictive object from observed context data, without performing Bayesian inference from scratch for each new task.
Concretely, given a context dataset \(\mathcal{D}_c\), one learns a predictor of the form
\begin{equation}
    q_{\bw}(y_{1:n} \mid \bx_{1:n}, \mathcal{D}_c),
    \quad n \geq 1,
\end{equation}
which approximates the corresponding finite-dimensional posterior predictive distributions.
This viewpoint is broad and includes neural processes and their attention-based variants as special cases, as well as other transformer-based amortised predictors.

We denote by $\btau = (\mathcal{D}_c,\mathcal{D}_t)$ a task, consisting of a context set \(\mathcal{D}_c\) and a target set
\begin{equation}
    \mathcal{D}_t
    =
    \left\{
    \left(\bx_j^{t}, y_j^{t}\right)
    \right\}_{j=1}^{n_t}.
\end{equation}
For a task \(\btau = (\mathcal{D}_c,\mathcal{D}_t)\) with \(\mathcal{D}_t = (\bx_{1:n_t}^{t}, y_{1:n_t}^{t})\), we write \(\bx_{1:n_t}^{t}(\btau) := \bx_{1:n_t}^{t}\) and \(y_{1:n_t}^{t}(\btau) := y_{1:n_t}^{t}\).
Let \(\mathbb{T} \subseteq \mathbb{D} \times \mathbb{D}\) denote the space of such tasks.
Tasks are sampled from a task distribution \(p(\btau)\).
This task distribution can be understood as the marginal distribution induced by first sampling a latent parameter \(\btheta\) and then sampling context and target observations from the corresponding stochastic process:
\begin{equation}
    p(\btau)
    =
    \int_{\Theta}
    p(\btau \mid \btheta)p(\btheta)
    \,\mathrm{d}\btheta.
\end{equation}
More explicitly, writing \(\btau = (\mathcal{D}_c,\mathcal{D}_t)\) with \(\mathcal{D}_c = (\bx_{1:n_c}^{c}, y_{1:n_c}^{c})\) and \(\mathcal{D}_t = (\bx_{1:n_t}^{t}, y_{1:n_t}^{t})\), we have
\begin{equation}
    p(\btau \mid \btheta)
    =
    p(y_{1:n_c}^{c}, y_{1:n_t}^{t} \mid \bx_{1:n_c}^{c}, \bx_{1:n_t}^{t}, \btheta)\,
    p(\bx_{1:n_c}^{c}, \bx_{1:n_t}^{t} \mid \btheta).
\end{equation}
Under the usual conditional independence assumption given \(\btheta\), this becomes
\begin{equation}
    p(\btau \mid \btheta)
    =
    \left[
    \prod_{i=1}^{n_c} p(y_i^{c} \mid \bx_i^{c}, \btheta)
    \right]
    \left[
    \prod_{j=1}^{n_t} p(y_j^{t} \mid \bx_j^{t}, \btheta)
    \right]
    p(\bx_{1:n_c}^{c}, \bx_{1:n_t}^{t} \mid \btheta).
\end{equation}
The task \(\btau\) itself contains only the observed context and target sets; the latent parameter \(\btheta\) is not observed by the amortised predictor.

Let \(q_{\bpsi}(y_{1:n_t}^{t} \mid \bx_{1:n_t}^{t}, \mathcal{D}_c)\) denote the amortised predictive family.
The amortised predictor is trained by minimising the expected negative log-likelihood of the target outputs conditioned on the target inputs and the context set:
\begin{equation}
    \mathcal{L}_{\mathrm{NLL}}(\bpsi)
    =
    \mathbb{E}_{\btau \sim p(\btau)}
    \left[
    -
    \log q_{\bpsi}
    \left(
        y_{1:n_t}^{t}
        \mid
        \bx_{1:n_t}^{t},
        \mathcal{D}_c
    \right)
    \right].
\end{equation}
If the target labels are conditionally independent under the model, this objective becomes
\begin{equation}
    \mathcal{L}_{\mathrm{NLL}}(\bpsi)
    =
    \mathbb{E}_{\btau \sim p(\btau)}
    \left[
    -
    \sum_{j=1}^{n_t}
    \log q_{\bpsi}
    \left(
        y_j^{t}
        \mid
        \bx_j^{t},
        \mathcal{D}_c
    \right)
    \right].
\end{equation}

In practice, the expectation over tasks is approximated using finitely many sampled tasks
\[
    \left\{
    \left(\calD_c^{(i)}, \calD_t^{(i)}\right)
    \right\}_{i = 1}^{N_{\mathrm{tasks}}},
\]
where
\begin{equation}
    \calD_t^{(i)}
    =
    \left\{
    \left(\bx_j^{t,(i)}, y_j^{t,(i)}\right)
    \right\}_{j=1}^{n_t}.
\end{equation}
This gives the empirical objective
\begin{equation}
    \widehat{\mathcal{L}}_{\mathrm{NLL}}(\bpsi)
    =
    -
    \frac{1}{N_{\mathrm{tasks}}}
    \sum_{i=1}^{N_{\mathrm{tasks}}}
    \log q_{\bpsi}
    \left(
        y_{1:n_t}^{t,(i)}
        \mid
        \bx_{1:n_t}^{t,(i)},
        \mathcal{D}_c^{(i)}
    \right).
\end{equation}
Under the conditional independence factorisation, this becomes
\begin{equation}
    \widehat{\mathcal{L}}_{\mathrm{NLL}}(\bpsi)
    =
    -
    \frac{1}{N_{\mathrm{tasks}}}
    \sum_{i=1}^{N_{\mathrm{tasks}}}
    \sum_{j=1}^{n_t}
    \log q_{\bpsi}
    \left(
        y_j^{t,(i)}
        \mid
        \bx_j^{t,(i)},
        \mathcal{D}_c^{(i)}
    \right).
\end{equation}

\subsection{Neural Processes and Conditional Neural Processes}\label{app:neural-process-transformer-background}

Neural processes \citep{Garnelo2018NPs} are a particular family of amortised predictors.
In this work, we focus on the conditional neural process (CNP) formulation \citep{Garnelo2018CNPs}.
A CNP directly parameterises the finite-dimensional predictive distributions conditioned on a context dataset
\begin{equation}
    \mathcal{D}_c
    =
    \left\{
    \left(\bx_i^{c}, y_i^{c}\right)
    \right\}_{i = 1}^{n_c}.
\end{equation}
Given such a dataset, a neural network with learnable weights \(\bw \in \calW\) defines
\begin{equation}
    q_{\bw}(y_{1:n} \mid \bx_{1:n}, \mathcal{D}_c),
    \quad n \geq 1,
\end{equation}
as an approximation to the finite-dimensional marginals of \(F_{\mathcal{D}_c} = \pi_F(\mathcal{D}_c)\).
In practice, neural processes are implemented using an encoder-decoder architecture.
Each context pair \((\bx_i^c, y_i^c)\) is first mapped to a representation \(r_i\) by an encoder network.
These context representations are then aggregated into a global summary \(r\) using a permutation-invariant operation such as averaging or summation.
A decoder network finally combines this summary with a target input \(\bx\) to parameterise the predictive distribution \(q_{\bw}(y \mid \bx, \mathcal{D}_c)\).
In a CNP, this aggregation step is the main architectural mechanism by which information from the context set is shared across target predictions.

Latent neural processes may instead introduce a task variable to represent uncertainty over the underlying task.
PrivTab follows the deterministic conditional formulation; its attention-based latent bottleneck and connection to prior-fitted prediction are described next.

\subsection{Transformer Neural Processes}

While CNPs provide an amortised approximation to posterior prediction, their predictive performance can be limited by the use of simple permutation-invariant aggregation schemes for the context set.
Such aggregators compress the context into a fixed-size representation, but may discard information about interactions among context points and about the relevance of different context rows to a given target query.
Transformer neural processes address this limitation by replacing hand-crafted aggregation with attention-based set processing.

\paragraph{Transformers and Attention.}
At a high level, a transformer is a neural architecture built from attention layers and token-wise feed-forward layers \citep{Vaswani2017Transformers}.
Given a collection of token representations, an attention layer updates each token by combining information from other tokens using data-dependent weights.
This allows the representation of each token to depend adaptively on the rest of the set.

For a query token representation \(\bq \in \mathbb{R}^{d}\), a collection of key representations \(\bk_1, \ldots, \bk_n \in \mathbb{R}^{d}\), and corresponding value representations \(\bv_1, \ldots, \bv_n \in \mathbb{R}^{d_v}\), attention produces an output of the form
\begin{equation}
    \mathrm{Attn}(\bq, \bk_{1:n}, \bv_{1:n})
    =
    \sum_{i=1}^n a_i(\bq,\bk_{1:n})\, \bv_i.
\end{equation}
In the transformer architecture, the coefficients \(a_1,\ldots,a_n\) are defined by applying a softmax to scaled dot-product similarity scores between the query and the keys~\citep{Vaswani2017Transformers}.
In practice, the query, key, and value representations are themselves obtained from token embeddings through learned linear transformations.
More explicitly, in this case
\begin{equation}
    a_i(\bq,\bk_{1:n})
    =
    \frac{\exp\!\left(\langle \bq, \bk_i\rangle / \sqrt{d}\right)}
    {\sum_{j=1}^{n} \exp\!\left(\langle \bq, \bk_j\rangle / \sqrt{d}\right)},
    \qquad i = 1, \ldots, n,
\end{equation}
where \(\langle \bq, \bk_i\rangle\) denotes the Euclidean inner product between \(\bq\) and \(\bk_i\).

In self-attention, the query token and the key-value pairs all come from the same set of tokens, so each token can incorporate information from the others in that set.
In cross-attention, the query comes from one set and the key-value pairs come from another, allowing one set of tokens to attend to another.
Multi-head attention applies several such attention operations in parallel and combines their outputs, allowing the model to represent different kinds of interactions simultaneously.

For set-structured data, attention-based architectures are especially attractive because they can model rich interactions among elements while remaining permutation-equivariant when positional encodings are not used.
This makes transformers a natural architectural choice for neural processes on tabular datasets.

\paragraph{Row-order symmetry.}
The row order of a tabular dataset is arbitrary, so a tabular predictor should not depend on the order in which context rows are presented.
Let $P_c$ and $P_t$ denote permutation matrices acting on the context and target-row axes, respectively.
A set-to-set predictor $f$ is context-row invariant and target-row equivariant when
\begin{equation}
    f\left(P_c\mathcal{D}_c, P_t\bx_{1:n_t}^t\right)
    =
    P_t f\left(\mathcal{D}_c, \bx_{1:n_t}^t\right).
\end{equation}
Consequently, permuting context rows leaves every prediction unchanged, while permuting target rows only permutes the corresponding predictions.
Token-wise maps, including MLPs and layer normalisation, preserve this property; self-attention is permutation equivariant; and cross-attention is invariant to a permutation of its key-value rows and equivariant to a permutation of its query rows.
Finally, a loss that aggregates the per-target losses by summation or averaging is invariant to target-row order.
Thus, a row-order-equivariant architecture together with an aggregated loss defines an order-invariant learning objective.

Feature order is different.
Columns have fixed coordinates in an input vector, and a generic row encoder need not be invariant to a permutation of those coordinates.
Exact feature-order invariance therefore requires an additional architectural constraint or explicit symmetrisation; it does not follow from row-order symmetry.

\paragraph{Transformer Neural Processes.}
The attentive neural process augments the CNP by replacing simple pooling with attention from the target inputs to encoded context representations~\citep{Kim2019ANPs}.
In this architecture, the context rows are first encoded independently, and each target input then attends to these encoded context rows to construct a target-dependent summary.
This addresses a key limitation of pooled encoders: different target inputs can focus on different parts of the context set.

Transformer neural processes (TNPs) extend this idea by introducing attention earlier in the encoder itself \citep{Nguyen2022TNPs}.
Rather than encoding context rows independently and only applying attention at prediction time, a TNP uses attention among context representations as part of the encoding process, allowing the context representation itself to capture interactions among rows before target-dependent prediction is performed.
Context and target rows are then mapped to token representations, and self-attention and cross-attention are used to construct target-dependent summaries of the context set.
In this way, a TNP can model both interactions within the context set and relevance of context rows to each target input.

Formally, a TNP defines
\begin{equation}
    q_{\bw}(y_{1:n} \mid \bx_{1:n}, \mathcal{D}_c)
\end{equation}
through an attention-based architecture operating on the context and target sets.
The context set provides observed input-label pairs, while the target set provides query inputs whose labels are to be predicted.
Attention allows the model to form representations that depend jointly on the target inputs and the observed context, making TNPs a natural neural-process analogue of target-dependent posterior prediction.

Several recent architectures improve the computational efficiency of attention-based neural processes by introducing a smaller latent bottleneck between the context set and the target predictions.
Latent bottleneck attentive neural processes (LBANPs) compress the context information into a fixed set of latent representations before decoding to the target queries \citep{Feng2023LBANP}.
Closely related ideas appear in perceiver-style architectures, where a small latent array cross-attends to a large input and is then processed in the latent space \citep{Jaegle2021Perceiver,Jaegle2022PerceiverIO}.
This can reduce the computational cost of full attention over large sets while still allowing expressive target-dependent prediction.
PrivTab adopts this LBANP design, using a fixed set of summary tokens to aggregate context information.
It also follows the prior-data fitted network perspective \citep{Muller2022PFNs}, in which pretraining on a task prior yields an in-context predictor for new datasets.

This perspective is particularly useful for the private setting considered in this work.
Since the sensitive information flows from the context set to the predictions through attention, the attention mechanism is the natural place to introduce a privacy-preserving transformation of the context information before any downstream prediction is performed by post-processing.

\subsection{Related Work}\label{app:related-work}

\paragraph{Non-DP Tabular Foundation Models.}
TabPFN \citep{Hollmann2023TabPFNICLR,Hollmann2025TabPFNNature} and its extensions, including TabPFN v3 \citep{grinsztajn2026tabpfn3technicalreport}, use prior-fitted inference to classify from a context dataset in a single forward pass.
TabICL and TabICL v2 \citep{Qu2025TabICL,qu2026tabiclv2} follow a closely related in-context-learning approach.
These models amortise tabular classification across tasks, but do not provide formal protection for sensitive context data.
This matters because prediction confidence can support membership inference \citep{Carlini2022MembershipInference,Knolle2026DisparatePrivacyRisks}; in our audit, TabPFN v3 and TabICL v2 show high aggregate and record-level attack success (\cref{app:membership-inference-results}).
PrivTab retains the in-context setting while protecting the transferred context representation under DP.

\paragraph{Meta-Learning.}
PrivTab also belongs to the broader meta-learning literature.
A useful distinction is between gradient-based adaptation and amortised inference \citep{bronskill2020data}.
Methods such as MAML \citep{Finn2017MetaLearning} and Reptile \citep{nichol2018first} adapt to a new task through additional test-time optimisation.
By contrast, matching networks \citep{vinyals2016matching}, prototypical networks \citep{snell2017prototypical}, and neural-process models \citep{Garnelo2018CNPs} map a context set directly to predictions.
PrivTab follows this latter approach, but explicitly constrains the context-to-prediction map with a DP mechanism.

\paragraph{Neural Processes and Transformer Neural Processes.}
PrivTab is most closely related to conditional, latent, attentive, and transformer neural processes \citep{Garnelo2018CNPs,Garnelo2018NPs,Kim2019ANPs,Nguyen2022TNPs}.
These models condition predictions on a context set in a single forward pass rather than fitting a new model for each dataset.
We build most directly on the transformer-neural-process viewpoint, but privatise the context-to-summary interface.
PrivTab therefore transfers a formally analysable DP summary rather than an unconstrained representation of the context set.

\paragraph{Differentially Private Machine Learning.}
Private machine learning most often relies on DP stochastic optimisation \citep{Abadi2016DPSGD,Ponomareva2023HowToDPfyML}.
These methods enable private training for many model classes, but privacy remains tied to an adaptive pipeline that includes clipping, accounting, model selection, and hyperparameter tuning.
PrivTab explores a different design: it privatises a compact dataset representation once and reuses it by post-processing, making it a private amortised-inference procedure rather than a private optimiser.

\paragraph{Differentially Private Meta-Learning.}
DPConvCNP \citep{raisa2024noise} is most closely aligned with our noise-aware objective.
It studies private regression with convolutional conditional neural processes and similarly trains on the privatised representation observed at test time.
Its focus, however, is convolutional regression rather than heterogeneous tabular classification.
Other work applies DP during meta-training \citep{li2020differentially,zhou2022task}, rather than privatising an inference-time context set while retaining amortised prediction for that dataset.

\paragraph{Differentially Private In-Context Learning with LLMs.}
Recent work also studies DP ICL for LLMs.
It includes private prompt synthesis \citep{gao2025data}, private aggregation of LLM responses \citep{wu2024privacypreserving}, DP generation of few-shot examples \citep{tang2024privacy}, privacy-aware nearest-neighbor retrieval \citep{koskela2025differentially}, and private prompt learning \citep{duan2023flocks}.
DP-TabICL\citep{carey2024dptabicl} privatises tabular records or group statistics and converts them to textual examples that are provided to a general-purpose LLM.
PrivTab instead is a standalone tabular model whose privacy mechanism is built into context summarisation; fitting to a new dataset requires neither external LLM prompting nor synthetic-query generation.
These methods address text prompts and LLM inference rather than tabular amortised classification, but share the goal of controlling information released from a sensitive context dataset.

\paragraph{Local Differential Privacy for Embeddings.}
Related work also privatises learned representations in a local privacy setting.
A local DP approach for transformer text embeddings using a variational information bottleneck has also been studied, with privacy characterised using Rényi divergence and Bayesian Differential Privacy~\citep{elzein2026differential}.
This is complementary to our setting: users privatise their own representations before release, whereas PrivTab applies central DP to a shared context dataset for tabular classification.

\FloatBarrier
\section{Noise-aware Amortised Inference}\label{app:noise-aware-amortized-inference}

In the private setting, the predictor does not observe the context dataset \(\mathcal{D}_c\) directly.
Instead, it observes the output of a DP mechanism applied to the context set.
When this private summary is produced by a learned encoder, the mechanism depends on parameters \(\bphi\), which we regard as fixed when defining the corresponding noise-aware target.
In PrivTab, this mechanism is instantiated by the DP-MHCA-based private encoder that maps \(\mathcal{D}_c\) to a privatised summary.
We therefore write \(\calA_{\bphi}\) for the mechanism, let \(\mathbb{S}\) denote its output space, and let \(\tilde{S} \in \mathbb{S}\) denote the released random output.
Conditional on \(\mathcal{D}_c\), this output is distributed according to
\begin{equation}
    \tilde{S} \mid \mathcal{D}_c \sim p_{\calA,\bphi}(\cdot \mid \mathcal{D}_c),
\end{equation}
where \(p_{\calA,\bphi}(\tilde{s} \mid \mathcal{D}_c)\) denotes the density or mass function induced by \(\calA_{\bphi}(\calD_c)\).
The realised DP output \(\tilde{s}\) may be a noisy statistic, a private representation, or a learned DP summary of the context dataset.

The appropriate predictive object is then not the posterior predictive conditioned on \(\mathcal{D}_c\), but the conditional distribution of the target labels given the public target inputs and the observed private summary.
Let \(X_{1:n}\) and \(Y_{1:n}\) denote the target-input and target-label random variables under the joint task law.
For fixed mechanism parameters \(\bphi\), define the \emph{noise-aware posterior predictive} directly as the regular conditional law induced by the joint task distribution and private mechanism:
\begin{equation}
    p_{\bphi}(y_{1:n} \mid \bx_{1:n}, \tilde{s})
    :=
    \Pr_{\bphi}\!\left(
        Y_{1:n}=y_{1:n}
        \,\middle|\,
        X_{1:n}=\bx_{1:n},\tilde{S}=\tilde{s}
    \right).
\end{equation}
When the joint law admits the latent-variable factorisation described in \cref{app:amortised-bayesian-prediction}, this conditional law can equivalently be written as
\begin{equation}
    p_{\bphi}(y_{1:n} \mid \bx_{1:n}, \tilde{s})
    =
    \int_{\Theta}
    p(y_{1:n} \mid \bx_{1:n}, \btheta)
    p_{\bphi}(\btheta \mid \bx_{1:n},\tilde{s})
    \,\mathrm{d}\btheta.
\end{equation}
This distribution accounts for uncertainty induced both by the latent parameter \(\btheta\), including any information about it carried by the target inputs, and by the randomness of the fixed private mechanism \(\calA_{\bphi}\).

Suppose now that we learn an amortised predictive distribution \(q_{\bpsi}(y_{1:n} \mid \bx_{1:n}, \tilde{s})\) to approximate this noise-aware posterior predictive.
In the private setting, the prediction objective is the same expected negative log-likelihood objective defined in \cref{app:amortised-bayesian-prediction}, except that the confidential context dataset is replaced by the released DP output \(\tilde{S}\).
For fixed mechanism parameters \(\bphi\), this gives
\begin{equation}
    \mathcal{L}^{\mathrm{priv}}_{\mathrm{NLL}}(\bpsi;\bphi)
    =
    \mathbb{E}_{p_{\bphi}(\btau,\tilde{S})}
    \left[
    -
    \log q_{\bpsi}
    \left(
        y_{1:n_t}^{t}
        \mid
        \bx_{1:n_t}^{t},
        \tilde{S}
    \right)
    \right].
\end{equation}
The following theorem identifies the Kullback--Leibler (KL) divergence objective corresponding to the conditional negative log-likelihood at fixed target size \(n_t\).

\begin{theorem}\label{thm:nll-noise-aware-predictive-kl}
    Fix mechanism parameters \(\bphi\) and a target size \(n_t \geq 1\).
Let \(\mathbb{T}_{n_t} \subseteq \mathbb{T}\) denote the subset of tasks whose target set has size \(n_t\), let \(\tilde{S} \mid \mathcal{D}_c \sim p_{\calA,\bphi}(\cdot \mid \mathcal{D}_c)\), and let \(p_{\bphi}^{(n_t)}(\btau,\tilde{S})\) denote the resulting joint law on \(\mathbb{T}_{n_t} \times \mathbb{S}\).
Define the conditional objective
    \begin{equation}
        \mathcal{L}^{\mathrm{priv}}_{\mathrm{NLL}}(\bpsi;\bphi,n_t)
        :=
        \mathbb{E}_{p_{\bphi}^{(n_t)}(\btau,\tilde{S})}
        \left[
            -\log q_{\bpsi}
            \left(
                y_{1:n_t}^{t}
                \mid
                \bx_{1:n_t}^{t},
                \tilde{S}
            \right)
        \right].
    \end{equation}
    Let \(q_{\bpsi}(y_{1:n_t} \mid \bx_{1:n_t}, \tilde{s})\) be an amortised predictive distribution.
Assume that the relevant joint and conditional densities or mass functions on \(\mathbb{T}_{n_t} \times \mathbb{S}\) exist and that \(\mathcal{L}^{\mathrm{priv}}_{\mathrm{NLL}}(\bpsi;\bphi,n_t)\) is finite, so that the successive integrations below are justified by Fubini's theorem.
Then
    \begin{equation}
        \mathcal{L}^{\mathrm{priv}}_{\mathrm{NLL}}(\bpsi;\bphi,n_t)
        =
        \mathrm{const}(\bphi,n_t)
        +
        \mathbb{E}_{p_{\bphi}^{(n_t)}(\btau,\tilde{S})}
        \left[
            \mathrm{KL}
            \left(
                p_{\bphi}(\cdot \mid \bx_{1:n_t}^{t}, \tilde{S})
                \,\middle\|\,
                q_{\bpsi}(\cdot \mid \bx_{1:n_t}^{t}, \tilde{S})
            \right)
        \right],
    \end{equation}
    where \(\mathrm{const}(\bphi,n_t)\) does not depend on \(\bpsi\).
Therefore, for fixed \(\bphi\), minimising \(\mathcal{L}^{\mathrm{priv}}_{\mathrm{NLL}}(\bpsi;\bphi,n_t)\) is equivalent to minimising the expected KL divergence from the noise-aware posterior predictive \(p_{\bphi}(\cdot \mid \bx_{1:n_t}^{t}, \tilde{S})\) to the amortised predictive distribution \(q_{\bpsi}(\cdot \mid \bx_{1:n_t}^{t}, \tilde{S})\).
\end{theorem}
\begin{proof}
    By definition of expectation with respect to the joint law \(p_{\bphi}^{(n_t)}(\btau,\tilde{S})\),
    \begin{equation}
        \mathcal{L}^{\mathrm{priv}}_{\mathrm{NLL}}(\bpsi;\bphi,n_t)
        =
        \int_{\mathbb{T}_{n_t}}
        \int_{\mathbb{S}}
        -\log q_{\bpsi}
        \left(
            y_{1:n_t}^{t}(\btau)
            \mid
            \bx_{1:n_t}^{t}(\btau),
            \tilde{s}
        \right)
        p_{\bphi}^{(n_t)}(\btau,\tilde{s})
        \,\mathrm{d}\tilde{s}\,\mathrm{d}\btau.
    \end{equation}
    Writing \(\btau = (\mathcal{D}_c,\mathcal{D}_t)\) with \(\mathcal{D}_t = (\bx_{1:n_t}^{t}, y_{1:n_t}^{t})\), and applying Fubini's theorem to expand the integral over the task coordinates, this is
    \begin{equation}
        \mathcal{L}^{\mathrm{priv}}_{\mathrm{NLL}}(\bpsi;\bphi,n_t)
        =
        \int_{\mathbb{D}}
        \int_{\calX^{n_t}}
        \int_{\calY^{n_t}}
        \int_{\mathbb{S}}
        -\log q_{\bpsi}
        \left(
            y_{1:n_t}^{t}
            \mid
            \bx_{1:n_t}^{t},
            \tilde{s}
        \right)
        p_{\bphi}^{(n_t)}(\mathcal{D}_c,\bx_{1:n_t}^{t},y_{1:n_t}^{t},\tilde{s})
        \,\mathrm{d}\tilde{s}\,\mathrm{d}y_{1:n_t}^{t}\,\mathrm{d}\bx_{1:n_t}^{t}\,\mathrm{d}\mathcal{D}_c.
    \end{equation}
    Integrating out the context dataset \(\mathcal{D}_c\), again by Fubini's theorem and the definition of the marginal density, gives
    \begin{equation}
        \mathcal{L}^{\mathrm{priv}}_{\mathrm{NLL}}(\bpsi;\bphi,n_t)
        =
        \int_{\calX^{n_t}}
        \int_{\calY^{n_t}}
        \int_{\mathbb{S}}
        -\log q_{\bpsi}
        \left(
            y_{1:n_t}^{t}
            \mid
            \bx_{1:n_t}^{t},
            \tilde{s}
        \right)
        p_{\bphi}^{(n_t)}(\bx_{1:n_t}^{t},y_{1:n_t}^{t},\tilde{s})
        \,\mathrm{d}\tilde{s}\,\mathrm{d}y_{1:n_t}^{t}\,\mathrm{d}\bx_{1:n_t}^{t}.
    \end{equation}
    Rewriting the joint density by the chain rule for densities or mass functions as \(p_{\bphi}^{(n_t)}(y_{1:n_t}^{t}\mid \bx_{1:n_t}^{t},\tilde{s})\,p_{\bphi}^{(n_t)}(\bx_{1:n_t}^{t},\tilde{s})\) gives
    \begin{equation}
        \begin{split}
        \mathcal{L}^{\mathrm{priv}}_{\mathrm{NLL}}(\bpsi;\bphi,n_t)
        =
        &\int_{\calX^{n_t}}
        \int_{\mathbb{S}}
        \left[
            \int_{\calY^{n_t}}
            -\log q_{\bpsi}
            \left(
                y_{1:n_t}^{t}
                \mid
                \bx_{1:n_t}^{t},
                \tilde{s}
            \right)
            p_{\bphi}^{(n_t)}(y_{1:n_t}^{t}\mid \bx_{1:n_t}^{t},\tilde{s})
            \,\mathrm{d}y_{1:n_t}^{t}
        \right]
        \\
        &\qquad\qquad\qquad\qquad
        p_{\bphi}^{(n_t)}(\bx_{1:n_t}^{t},\tilde{s})
        \,\mathrm{d}\tilde{s}\,\mathrm{d}\bx_{1:n_t}^{t}.
        \end{split}
    \end{equation}
    By definition of the noise-aware posterior predictive, the conditional law \(p_{\bphi}^{(n_t)}(y_{1:n_t}^{t}\mid \bx_{1:n_t}^{t},\tilde{s})\) is precisely \(p_{\bphi}(\cdot \mid \bx_{1:n_t}^{t},\tilde{s})\).
Therefore the inner integral is the cross-entropy
    \begin{equation}
        \mathcal{H}
        \left(
            p_{\bphi}(\cdot \mid \bx_{1:n_t}^{t}, \tilde{s}),
            q_{\bpsi}(\cdot \mid \bx_{1:n_t}^{t}, \tilde{s})
        \right).
    \end{equation}
    Hence
    \begin{equation}
        \mathcal{L}^{\mathrm{priv}}_{\mathrm{NLL}}(\bpsi;\bphi,n_t)
        =
        \int_{\calX^{n_t}}
        \int_{\mathbb{S}}
        \mathcal{H}
        \left(
            p_{\bphi}(\cdot \mid \bx_{1:n_t}^{t}, \tilde{s}),
            q_{\bpsi}(\cdot \mid \bx_{1:n_t}^{t}, \tilde{s})
        \right)
        p_{\bphi}^{(n_t)}(\bx_{1:n_t}^{t},\tilde{s})
        \,\mathrm{d}\tilde{s}\,\mathrm{d}\bx_{1:n_t}^{t}.
    \end{equation}
    Using the identity
    \begin{equation}
        \mathcal{H}(p,q) = \mathcal{H}(p) + \mathrm{KL}(p \| q),
    \end{equation}
    pointwise in \((\bx_{1:n_t}^{t},\tilde{s})\), we obtain
    \begin{equation}
        \begin{split}
        \mathcal{L}^{\mathrm{priv}}_{\mathrm{NLL}}(\bpsi;\bphi,n_t)
        =
        &\int_{\calX^{n_t}}
        \int_{\mathbb{S}}
        \mathcal{H}
        \left(
            p_{\bphi}(\cdot \mid \bx_{1:n_t}^{t}, \tilde{s})
        \right)
        p_{\bphi}^{(n_t)}(\bx_{1:n_t}^{t},\tilde{s})
        \,\mathrm{d}\tilde{s}\,\mathrm{d}\bx_{1:n_t}^{t}
        \\
        &+
        \int_{\calX^{n_t}}
        \int_{\mathbb{S}}
        \mathrm{KL}
        \left(
            p_{\bphi}(\cdot \mid \bx_{1:n_t}^{t}, \tilde{s})
            \,\middle\|\,
            q_{\bpsi}(\cdot \mid \bx_{1:n_t}^{t}, \tilde{s})
        \right)
        p_{\bphi}^{(n_t)}(\bx_{1:n_t}^{t},\tilde{s})
        \,\mathrm{d}\tilde{s}\,\mathrm{d}\bx_{1:n_t}^{t}.
        \end{split}
    \end{equation}
    Since the KL integrand depends on \(\btau\) only through the target inputs \(\bx_{1:n_t}^{t}\) and on the private release through \(\tilde{S}\), another application of the law of the unconscious statistician yields
    \begin{equation}
        \mathcal{L}^{\mathrm{priv}}_{\mathrm{NLL}}(\bpsi;\bphi,n_t)
        =
        \mathrm{const}(\bphi,n_t)
        +
        \mathbb{E}_{p_{\bphi}^{(n_t)}(\btau,\tilde{S})}
        \left[
            \mathrm{KL}
            \left(
                p_{\bphi}(\cdot \mid \bx_{1:n_t}^{t}, \tilde{S})
                \,\middle\|\,
                q_{\bpsi}(\cdot \mid \bx_{1:n_t}^{t}, \tilde{S})
            \right)
        \right].
    \end{equation}
    The first term depends only on the data-generating process and the fixed mechanism parameters \(\bphi\), not on \(\bpsi\).
Writing it as
    \begin{equation}
        \mathrm{const}(\bphi,n_t)
        =
        \int_{\calX^{n_t}}
        \int_{\mathbb{S}}
        \mathcal{H}
        \left(
            p_{\bphi}(\cdot \mid \bx_{1:n_t}^{t}, \tilde{s})
        \right)
        p_{\bphi}^{(n_t)}(\bx_{1:n_t}^{t},\tilde{s})
        \,\mathrm{d}\tilde{s}\,\mathrm{d}\bx_{1:n_t}^{t}
    \end{equation}
    gives the result.
\end{proof}

\section{Differentially Private Transformer Neural Processes}\label{app:dp-tnp}\label{sec:dp-tnp}\label{app:dp-tnp-architecture-details}
Differentially private transformer neural processes are transformer-based amortised predictors for which the sensitive information is contained in the context dataset.
We assume that each individual's sensitive data corresponds to a single row of the private context set $\calD_c$.
Accordingly, privacy is defined with respect to the substitute adjacency relation defined in \cref{app:differential-privacy}, so that neighbouring datasets differ in exactly one row.
This choice is natural for the tabular setting considered here and is compatible with the sequential representation of the context set used by the transformer architecture.

We instantiate PrivTab shown in \cref{fig:architecture}.
The model receives a private context set $\calD_c$ together with public target queries $\bx_{1:n_t}^t$ and proceeds in three parts:
\begin{enumerate}
    \item A row embedding model maps context rows $\{(\bx_i^c,y_i^c)\}_{i=1}^{n_c}$ and target queries $\{\bx_j^t\}_{j=1}^{n_t}$ to token representations.
    \item A fixed set of learned pseudo-tokens $\{\bu_i\}_{i=1}^{m}$ repeatedly aggregates the private context tokens through DP-MHCA, followed by latent self-attention, to produce privatised summary tokens.
    \item At each layer, target tokens attend to the privatised summaries through multi-head cross-attention; a final MLP predicts $y_{1:n_t}^t$.
This target-side computation is post-processing and requires no additional privacy protection.
\end{enumerate}
The final privacy guarantee depends on the number of DP-MHCA compositions.
The following subsections develop the model in order: row embedding, the private attention mechanism and its guarantees, the reference implementation, and the complete encoder architecture.
The corresponding noise-aware pretraining objective and training procedure are developed separately in \cref{app:dp-tnp-pretraining}.

\paragraph{Row-order equivariance of PrivTab.}
PrivTab uses no positional encodings.
Row-wise embeddings, layer normalisation, feed-forward networks, and decoding are applied independently to each row and do not use row order.
DP-MHCA is invariant to context-row permutations because it sums bounded per-row contributions; the remaining latent computation is post-processing.
Target-side cross-attention is equivariant to target-row permutations.
Hence, the predictor is invariant to context-row order and equivariant to target-row order.
Once target predictions are combined in the aggregated training loss, the loss is invariant to permutations of both context and target rows.

\subsection{Row Embedding Model}

The row embedding model maps each context row \((\bx_i^c,y_i^c)\) and each target row \(\bx_j^t\) to a token representation in \(\mathbb{R}^{d_{\mathrm{tok}}}\).
Let \(d \leq d_x\) denote the number of observed features in a given dataset, and let \(e_{\mathrm{pad}} \in \mathbb{R}\) denote a learned padding value.
For a raw input row \(\bx_i^{\mathrm{raw}} \in \mathbb{R}^{d}\), we first form a padded vector \(\bar{\bx}_i \in \mathbb{R}^{d_x}\) by
\begin{equation}
    \bar{\bx}_{i,k}
    =
    \begin{cases}
        \bx_{i,k}^{\mathrm{raw}}, & k \leq d, \\
        e_{\mathrm{pad}}, & k > d,
    \end{cases}
    \qquad k = 1,\ldots,d_x.
\end{equation}
We then define a binary validity mask \(\mathbf{m}_i \in \{0,1\}^{d_x}\) by
\begin{equation}
    \mathbf{m}_{i,k}
    =
    \begin{cases}
        1, & k \leq d, \\
        0, & k > d,
    \end{cases}
    \qquad k = 1,\ldots,d_x,
\end{equation}
and concatenate the padded row with the mask:
\begin{equation}
    \tilde{\bx}_i = \left[\bar{\bx}_i,\mathbf{m}_i\right] \in \mathbb{R}^{2d_x}.
\end{equation}
To keep rows with different numbers of observed features on a comparable scale, we rescale by
\begin{equation}
    \hat{\bx}_i = \sqrt{\frac{d_x}{d}}\,\tilde{\bx}_i
    \in \mathbb{R}^{2d_x}.
\end{equation}
Thus, the scaled augmented row itself serves as the \(x\)-representation.
For classification, the label encoder maps each label \(y \in \{1,\ldots,C\}\), where \(C\leq C_{\mathrm{max}}=10\), and a dedicated null label \(y_{\mathrm{null}}=C_{\mathrm{max}}+1\) to a learned embedding in \(\mathbb{R}^{d_y}\).
Equivalently, it is a learned lookup table
\begin{equation}
    \mathrm{LabelEmbed} : \{1,\ldots,C_{\mathrm{max}}+1\} \to \mathbb{R}^{d_y},
\end{equation}
where the final index is reserved for the null label regardless of the number of classes in the current dataset.
In the benchmark configuration, \(d_y=10\).
For context rows, the observed label embedding is concatenated with the row representation,
\begin{equation}
    \bz_i^c = \mathrm{RowEmbed}\left(\hat{\bx}_i^c, \mathrm{LabelEmbed}(y_i^c)\right),
\end{equation}
where \(\mathrm{RowEmbed}(\mathbf{x},\mathbf{e})\) concatenates its two arguments as \([\mathbf{x},\mathbf{e}]\) and maps the result into the shared token space using an MLP.
For target rows, the same construction is used except that the observed label embedding is replaced by the null-label embedding \(\mathrm{LabelEmbed}(y_{\mathrm{null}})\) to indicate an unknown label and distinguish target from context tokens, giving
\begin{equation}
    \bz_j^t = \mathrm{RowEmbed}\left(\hat{\bx}_j^t, \mathrm{LabelEmbed}(y_{\mathrm{null}})\right).
\end{equation}
All previous operations are applied row-wise, which is important both for the privacy analysis and for row-order symmetry.
In contrast, the row encoder is not exactly feature-order invariant: a permutation of feature columns permutes coordinates of $\tilde{\bx}_i$, and the subsequent MLP $\mathrm{RowEmbed}$ has coordinate-specific weights.
As with TabPFN and TabICL \citep{Hollmann2023TabPFNICLR,Qu2025TabICL}, we rely instead on the randomised synthetic task prior during pretraining to encourage robustness to the arbitrary ordering and roles of features.
This provides approximate feature-order invariance in practice, but not an exact architectural guarantee.

\subsection{DP-MHCA}\label{app:dp-mhca}\label{sec:dp-mhca}
The DP-MHCA is the central component of PrivTab.
It is a multi-head cross-attention mechanism in which a query sequence attends to the private context tokens, followed by Gaussian perturbation of the resulting representation.
In the first private layer, these queries are the learned pseudo-tokens \(\bu_{1:m}\); in later private layers, they are the current summary tokens.
Let \(d_{\mathrm{tok}}\) denote the token dimension, \(H\) the number of heads, and \(d_h\) the per-head dimension, so that the internal attention dimension is \(H d_h\).
Let \(\mathbf{r}_{1:m} \in (\mathbb{R}^{d_{\mathrm{tok}}})^m\) denote the query tokens supplied to a given DP-MHCA layer, and let \(\bz_{1:n_c}^c \in (\mathbb{R}^{d_{\mathrm{tok}}})^{n_c}\) denote the private context tokens.
These are first mapped by learned linear projections
\begin{equation}
    \mathbf{W}_Q \in \mathbb{R}^{H d_h \times d_{\mathrm{tok}}}, \qquad
    \mathbf{W}_K \in \mathbb{R}^{H d_h \times d_{\mathrm{tok}}}, \qquad
    \mathbf{W}_V \in \mathbb{R}^{H d_h \times d_{\mathrm{tok}}},
\end{equation}
so that for each query index \(i\) and context index \(j\),
\begin{equation}
    \bq_i = \mathbf{W}_Q \mathbf{r}_i \in \mathbb{R}^{H d_h}, \qquad
    \bk_j = \mathbf{W}_K \bz_j^c \in \mathbb{R}^{H d_h}, \qquad
    \bar{\bv}_j = \mathbf{W}_V \bz_j^c \in \mathbb{R}^{H d_h}.
\end{equation}
Each projected vector is then reshaped into \(H\) heads of dimension \(d_h\).
We denote the resulting per-head vectors by
\begin{equation}
    \bq_i^{(h)} \in \mathbb{R}^{d_h}, \qquad
    \bk_j^{(h)} \in \mathbb{R}^{d_h}, \qquad
    \bar{\bv}_j^{(h)} \in \mathbb{R}^{d_h},
\end{equation}
for heads \(h=1,\ldots,H\), query indices \(i=1,\ldots,m\), and context indices \(j=1,\ldots,n_c\).

In the implementation, the value vectors are normalised after projection, head by head:
\begin{equation}
    \bv_j^{(h)}
    =
    \frac{\bar{\bv}_j^{(h)}}{\max\!\left\{\lVert \bar{\bv}_j^{(h)} \rVert_2,\varepsilon_v\right\}},
\end{equation}
where \(\varepsilon_v = 10^{-12}\) is the numerical stabiliser used by the implementation.
Hence \(\lVert \bv_j^{(h)} \rVert_2 \leq 1\) for all \(h\) and \(j\).

For each head, the cross-attention output is computed using a bounded activation applied to the scaled dot products:
\begin{equation}\label{eq:dp-mhca-head}
    \operatorname{Att}^{(h)}\!\left(\bq_i^{(h)}, \bk_{1:n_c}^{(h)}, \bv_{1:n_c}^{(h)}\right)
    =
    \sum_{j=1}^{n_c}
    g\!\left(\bq_i^{(h)}, \bk_j^{(h)}\right)\bv_j^{(h)},
\end{equation}
where
\begin{equation}
    g(\bq,\bk)
    =
    \tanh\!\left(\frac{\langle \bq,\bk\rangle}{\sqrt{d_h}}\right).
\end{equation}
Unlike standard transformer attention, these coefficients are not normalised across \(j\); they are simply bounded by the \(\tanh\) nonlinearity.
Consequently, each summand in \cref{eq:dp-mhca-head} depends on exactly one context token and hence on exactly one individual's row in the private context set.
Moreover,
\begin{equation}
    \left\lVert
    g\!\left(\bq_i^{(h)}, \bk_j^{(h)}\right)\bv_j^{(h)}
    \right\rVert_2
    =
    \left|g\!\left(\bq_i^{(h)}, \bk_j^{(h)}\right)\right|
    \left\lVert \bv_j^{(h)} \right\rVert_2
    \leq 1.
\end{equation}
This bounded-contribution property is what makes the representation amenable to DP.

The head outputs are then concatenated across \(h=1,\ldots,H\) into the pre-noise attention output \(\mathbf{a}_i\), distinct from the complete block output \(\bs_i^{(\ell)}\) used in the architecture below:
\begin{equation}
    \mathbf{a}_i
    =
    \left[
        \mathbf{a}_i^{(1)},\ldots,\mathbf{a}_i^{(H)}
    \right]
    \in \mathbb{R}^{H d_h},
\end{equation}
where
\begin{equation}
    \mathbf{a}_i^{(h)}
    :=
    \operatorname{Att}^{(h)}\!\left(\bq_i^{(h)}, \bk_{1:n_c}^{(h)}, \bv_{1:n_c}^{(h)}\right)
\end{equation}
denotes the \(h\)-th pre-noise head output for query \(i\).
Gaussian noise is added to this concatenated representation; the subscript \(\mathrm{attn}\) distinguishes this intermediate attention output from the saved layer states \(\tilde{\bs}_i^{(\ell)}\) defined below:
\begin{equation}
    \tilde{\bs}_{i,\mathrm{attn}} = \mathbf{a}_i + \boldsymbol{\xi}_i,
    \qquad
    \boldsymbol{\xi}_i \sim \mathcal{N}(\mathbf{0}, \sigma^2 I_{H d_h}),
\end{equation}
where, jointly across the \(m\) summary tokens, the concatenated noise vector satisfies
\begin{equation}
    \boldsymbol{\xi}_{1:m}
    \sim
    \mathcal{N}\!\left(\mathbf{0},\sigma^2 I_{m H d_h}\right).
\end{equation}
Thus, the noise coordinates, including the blocks associated with distinct summary tokens, are mutually independent.
Although each per-row contribution in DP-MHCA is bounded, the aggregated pre-noise summary can still have a norm that varies substantially across inputs.
This variability is beneficial for improving the signal-to-noise ratio (SNR) when the context tokens are more informative, but after Gaussian perturbation a normalisation step is useful to stabilise the noisy attention-output norms before the residual and feed-forward computations.
For this reason, we first compute the maximum Euclidean norm across the perturbed tokens and then rescale each perturbed token by
\begin{equation}\label{eq:dp-mhca-normalization}
    \tilde{\bs}_{i,\mathrm{attn}} \leftarrow
    \frac{G\tilde{\bs}_{i,\mathrm{attn}}}
    {\max_{k=1,\ldots,m} \lVert \tilde{\bs}_{k,\mathrm{attn}} \rVert_2 + \varepsilon_s},
    \qquad \varepsilon_s = 10^{-6}.
\end{equation}
Thus, after this optional step, the largest token norm is strictly below and, when the maximum pre-normalisation norm is large relative to \(\varepsilon_s\), approximately equal to $G$.
We set $G=512$ in all reported experiments.
Finally, the resulting vector is mapped back to the original token dimension by an output projection
\begin{equation}
    \mathbf{W}_O \in \mathbb{R}^{d_{\mathrm{tok}} \times H d_h},
    \qquad
    \bo_i = \mathbf{W}_O \tilde{\bs}_{i,\mathrm{attn}} \in \mathbb{R}^{d_{\mathrm{tok}}}.
\end{equation}
In the benchmark configuration, \(H d_h=d_{\mathrm{tok}}=256\) and this output map is the identity.
Thus, the DP mechanism in DP-MHCA is applied after the bounded cross-attention computation in \(\mathbb{R}^{H d_h}\) and before any output mapping back to the token space \(\mathbb{R}^{d_{\mathrm{tok}}}\).

For the following results, we consider the deterministic DP-MHCA map before Gaussian perturbation, optional normalisation, and output projection, conditional on a fixed supplied query sequence \(\mathbf{r}_{1:m}\).
Writing
\begin{equation}
    \mathbf{a}_{1:m}(\calD_c;\mathbf{r}_{1:m}) = \left(\mathbf{a}_1(\calD_c;\mathbf{r}_{1:m}),\ldots,\mathbf{a}_m(\calD_c;\mathbf{r}_{1:m})\right) \in \mathbb{R}^{m H d_h},
\end{equation}
we treat \(\mathbf{a}_{1:m}(\calD_c;\mathbf{r}_{1:m})\) as the concatenation of the \(m\) pre-noise attention outputs defined above.

\paragraph{Privacy and utility guarantees.}

\begin{theorem}[Sensitivity of pre-noise DP-MHCA]
\label{thm:dp-mhca-sensitivity}
Fix any query sequence \(\br_{1:m}\).
Under substitute adjacency, the \(\ell_2\)-sensitivity of the deterministic pre-noise DP-MHCA map with respect to the private context dataset satisfies
\begin{equation}
    \Delta_2\!\left(\mathbf{a}_{1:m}(\cdot;\br_{1:m})\right)
    \leq
    2\sqrt{H m}.
\end{equation}
\end{theorem}
\begin{proof}
Let \(\calD_c\) and \(\calD_c'\) be neighbouring context datasets under substitute adjacency, so they differ in exactly one row.
Since all operations before the attention sum are applied row-wise, this changes exactly one context token in the sequence \(\bz_{1:n_c}^c\).
Let this differing token occur at index \(j^\star\).
For any query index \(i\) and head \(h\), the difference between the two head outputs is
\begin{equation}
\begin{aligned}
    &\operatorname{Att}^{(h)}\!\left(\bq_i^{(h)}, \bk_{1:n_c}^{(h)}, \bv_{1:n_c}^{(h)}\right)
    -
    \operatorname{Att}^{(h)}\!\left(\bq_i^{(h)}, \bk_{1:n_c}^{\prime (h)}, \bv_{1:n_c}^{\prime (h)}\right) \\
    &=
    g\!\left(\bq_i^{(h)}, \bk_{j^\star}^{(h)}\right)\bv_{j^\star}^{(h)}
    -
    g\!\left(\bq_i^{(h)}, \bk_{j^\star}^{\prime (h)}\right)\bv_{j^\star}^{\prime (h)},
\end{aligned}
\end{equation}
because all remaining summands are identical and cancel.
By the triangle inequality and the bound established above,
\begin{equation}
\begin{aligned}
    &\left\lVert
    \operatorname{Att}^{(h)}\!\left(\bq_i^{(h)}, \bk_{1:n_c}^{(h)}, \bv_{1:n_c}^{(h)}\right)
    -
    \operatorname{Att}^{(h)}\!\left(\bq_i^{(h)}, \bk_{1:n_c}^{\prime (h)}, \bv_{1:n_c}^{\prime (h)}\right)
    \right\rVert_2 \\
    &\leq
    \left\lVert
    g\!\left(\bq_i^{(h)}, \bk_{j^\star}^{(h)}\right)\bv_{j^\star}^{(h)}
    \right\rVert_2
    +
    \left\lVert
    g\!\left(\bq_i^{(h)}, \bk_{j^\star}^{\prime (h)}\right)\bv_{j^\star}^{\prime (h)}
    \right\rVert_2
    \leq 2.
\end{aligned}
\end{equation}
Thus each head of each summary token changes by at most \(2\) in \(\ell_2\)-norm.
Since \(\mathbf{a}_{1:m}(\calD_c;\mathbf{r}_{1:m})\) is obtained by concatenating the \(mH\) head outputs, the squared \(\ell_2\)-difference of the full pre-noise representation is at most
\begin{equation}
    \sum_{i=1}^{m}\sum_{h=1}^{H} 2^2 = 4 H m.
\end{equation}
Taking square roots yields
\begin{equation}
    \left\lVert
    \mathbf{a}_{1:m}(\calD_c;\mathbf{r}_{1:m}) - \mathbf{a}_{1:m}(\calD_c';\mathbf{r}_{1:m})
    \right\rVert_2
    \leq
    2\sqrt{H m},
\end{equation}
which proves the claim.
\end{proof}

An important consequence of \cref{thm:dp-mhca-sensitivity} is that the sensitivity bound does not depend on the context size \(n_c\).
Thus, increasing the number of context rows need not increase the DP noise scale, even though additional informative rows can strengthen the signal accumulated by the attention mechanism.

\begin{theorem}[\(\mu\)-GDP of the PrivTab Encoder]
\label{thm:dp-tnp-gdp}
Let \(L_{\mathrm{main}}\geq1\) and \(\mu>0\).
Under substitute adjacency, if the encoder contains \(L_{\mathrm{main}}\) DP-MHCA layers, and each layer \(\ell\in\{1,\ldots,L_{\mathrm{main}}\}\) adds isotropic Gaussian noise with standard deviation \(\sigma_\ell=2\sqrt{HmL_{\mathrm{main}}}/\mu\) to the full concatenated summary representation, with fresh noise independent across layers and of all preceding mechanism randomness, then the overall PrivTab encoder satisfies \(\mu\)-GDP with respect to \(\calD_c\).
\end{theorem}
\begin{proof}
We first note that all row embedding and token encoding operations preceding the DP-MHCA layer are applied point-wise (row-wise) to the private context dataset $\calD_c$.
Under substitute adjacency, two adjacent context datasets $\calD_c \simeq \calD_c'$ differ in exactly one row.
Because the preprocessing is point-wise, this single-row difference propagates to a difference in at most one context token in the resulting sequence of context tokens $\bz_{1:n_c}^c$.

Fix any private layer \(\ell\).
Conditional on all previous DP releases, the query sequence supplied to layer \(\ell\) is fixed, although it may depend adaptively on earlier private outputs.
Therefore, by \cref{thm:dp-mhca-sensitivity}, the conditional \(\ell_2\)-sensitivity of the pre-noise summary representation at layer \(\ell\) satisfies \(\Delta_2 \leq 2\sqrt{H m}\) with respect to the private context dataset.

Under the Gaussian mechanism, adding independent Gaussian noise with standard deviation $\sigma_{\ell}$ to a query with $\ell_2$-sensitivity $\Delta_2$ satisfies $\mu_{\ell}$-GDP with $\mu_{\ell} = \frac{\Delta_2}{\sigma_{\ell}}$ \citep{Dong2022GDP}.
Substituting the noise standard deviation, we have:
\begin{equation}
    \mu_{\ell} 
    \leq 
    \frac{2\sqrt{H m}}{\sigma_{\ell}} 
    = 
    \frac{2\sqrt{H m}}{\frac{2\sqrt{H m L_{\mathrm{main}}}}{\mu}}
    = 
    \frac{\mu}{\sqrt{L_{\mathrm{main}}}}.
\end{equation}

Apart from the DP-MHCA accesses to the private context, all intermediate computations in the encoder and decoder are post-processing of previous private releases.
In particular, standard self-attention, target-side cross-attention, feed-forward layers, and the MLP decoder do not consume additional privacy budget beyond the DP-MHCA layers themselves.
Since the query sequence of a later DP-MHCA layer may depend on earlier private releases, the resulting composition is adaptive.

By the adaptive composition theorem of GDP (\cref{thm:mu-gdp-composition}), the composition of the $L_{\mathrm{main}}$ private DP-MHCA mechanisms yields an overall privacy parameter of:
\begin{equation}
    \mu_{\mathrm{total}} 
    = 
    \sqrt{\sum_{\ell=1}^{L_{\mathrm{main}}} \mu_{\ell}^2} 
    \leq 
    \sqrt{L_{\mathrm{main}} \left( \frac{\mu}{\sqrt{L_{\mathrm{main}}}} \right)^2} 
    = 
    \mu.
\end{equation}
Therefore, the PrivTab encoder satisfies $\mu$-GDP.
\end{proof}

This result lifts the per-layer sensitivity bound to an end-to-end privacy guarantee: with the stated noise allocation, the entire encoder satisfies the selected \(\mu\)-GDP budget under substitute adjacency.
All subsequent target-side computation is post-processing and therefore incurs no further privacy loss.

\begin{theorem}[Directional SNR of DP-MHCA]
\label{thm:dp-mhca-snr}
Suppose \(\sigma>0\), and suppose there exist a unit vector \(\bu\in\mathbb{R}^{d_h}\), constants \(c,\kappa>0\), an index set \(\mathcal{I}\subseteq\{1,\ldots,n_c\}\), and a constant \(R\geq0\) such that \(g(\bq_i^{(h)},\bk_j^{(h)})\geq c\) and \(\langle\bu,\bv_j^{(h)}\rangle\geq\kappa\) for all \(j\in\mathcal I\), and
\begin{equation}
    \sum_{j\notin\mathcal I}
    g\!\left(\bq_i^{(h)},\bk_j^{(h)}\right)
    \left\langle\bu,\bv_j^{(h)}\right\rangle
    \geq-R.
\end{equation}
Then, if Gaussian noise with covariance \(\sigma^2I_{d_h}\) is added to this head output, the directional signal-to-noise ratio
\begin{equation}
    \mathrm{SNR}_{\bu}
    :=
    \frac{|\langle\bu,\mathbf{a}_i^{(h)}\rangle|}{\sigma}
    \geq
    \frac{\max\{c\kappa|\mathcal I|-R,0\}}{\sigma}.
\end{equation}
\end{theorem}
\begin{proof}
Taking the inner product of \cref{eq:dp-mhca-head} with \(\bu\) gives
\begin{equation}
\begin{aligned}
    \left\langle
        \bu, \mathbf{a}_i^{(h)}
    \right\rangle
    &=
    \sum_{j=1}^{n_c}
    g\!\left(\bq_i^{(h)}, \bk_j^{(h)}\right)
    \left\langle \bu, \bv_j^{(h)} \right\rangle \\
    &=
    \sum_{j\in \mathcal{I}}
    g\!\left(\bq_i^{(h)}, \bk_j^{(h)}\right)
    \left\langle \bu, \bv_j^{(h)} \right\rangle
    +
    \sum_{j\notin \mathcal{I}}
    g\!\left(\bq_i^{(h)}, \bk_j^{(h)}\right)
    \left\langle \bu, \bv_j^{(h)} \right\rangle.
\end{aligned}
\end{equation}
For every \(j \in \mathcal{I}\), the corresponding summand is at least \(c\kappa\).
Hence
\begin{equation}
    \left\langle
        \bu, \mathbf{a}_i^{(h)}
    \right\rangle
    \geq
    \sum_{j\in \mathcal{I}} c\kappa - R
    =
    c\kappa |\mathcal{I}| - R.
\end{equation}
If isotropic Gaussian noise \(\boldsymbol{\xi} \sim \mathcal{N}(\mathbf{0},\sigma^2 I_{d_h})\) is then added, the projected noise satisfies \(\langle \bu,\boldsymbol{\xi}\rangle \sim \mathcal{N}(0,\sigma^2)\) because \(\|\bu\|_2=1\).
Therefore,
\begin{equation}
    |\langle \bu, \mathbf{a}_i^{(h)} \rangle|
    \geq
    \max\{c\kappa |\mathcal{I}| - R,0\},
\end{equation}
which yields the stated lower bound on \(\mathrm{SNR}_{\bu}\).
\end{proof}

The theorem formalises when additional context is useful: rows whose values align with the query direction reinforce one another, whereas the remaining rows can contribute only the bounded offset \(R\).
Thus, the directional signal grows with the number of aligned rows while the noise level is fixed by the privacy mechanism.

Together, the sensitivity and SNR results show that, when a fixed fraction of context rows is aligned with a query and \(R\) remains bounded independently of \(n_c\), the directional signal grows linearly with the number of informative rows without increasing the required DP noise scale.

\subsection{Reference Implementation}\label{app:dp-mhca-reference-implementation}

The following listing gives the complete private summary module of the current PrivTab implementation, starting from already embedded context rows.
The \texttt{PrivateSummaryModule} class initialises 128 learned pseudo-tokens, applies three private cross-attention and latent self-attention blocks, then applies two further latent self-attention blocks.
It returns all five summary states used by the target-side predictor.
Its \texttt{DPMHCA} class uses one 256-dimensional head with bounded \(\tanh\) attention, Gaussian perturbation, and optional post-noise normalisation; the noisy output is returned directly without an output projection.
The module calibrates the per-layer noise standard deviation from the supplied positive \(\mu\) using \cref{thm:dp-mhca-sensitivity,thm:dp-tnp-gdp}.
The privacy guarantee assumes fixed model weights, independently embedded context rows, and fresh independent standard Gaussian noise at every private layer.
Dataset-level preprocessing is outside this listing and must use public or separately privatised statistics.
Training uses the default PyTorch noise sampler; private releases supply a secure Gaussian sampler.

\begingroup
\fontencoding{T1}\selectfont
\begin{Verbatim}[breaklines,breakanywhere,fontsize=\scriptsize,frame=single]
"""PrivTab private summary module for already embedded context rows.

The module uses three DP-MHCA layers and two further self-attention layers.
Context embeddings must be row-wise and model weights fixed for the stated
privacy guarantee. The noise sampler must return fresh independent standard
Gaussians on each call. For a private release, pass a secure sampler; the
default PyTorch sampler matches training and benchmark evaluation.
"""

import math

import torch
from torch import nn
from torch.nn import functional as F


class DPMHCA(nn.Module):
    """Bounded cross-attention followed by Gaussian noise."""

    def __init__(self):
        super().__init__()
        self.to_q = nn.Linear(256, 256, bias=False)
        self.to_k = nn.Linear(256, 256, bias=False)
        self.to_v = nn.Linear(256, 256, bias=False)

    def forward(self, xq, xkv, sigma, normalize=False, noise_sampler=torch.randn_like):
        """Return private summaries; ``sigma`` has one entry per batch item."""
        q = self.to_q(xq)
        k = self.to_k(xkv)
        v = F.normalize(self.to_v(xkv), p=2, dim=-1)

        out = torch.tanh((q @ k.transpose(-1, -2)) * 256**-0.5) @ v
        noise = noise_sampler(out)
        out = out + noise * sigma[:, None, None]
        if normalize:
            norm = torch.linalg.vector_norm(out, dim=-1, keepdim=True).amax(
                dim=1, keepdim=True
            )
            out = out / (norm + 1e-6) * 512
        return out


class SelfAttention(nn.Module):
    """Standard four-head attention on the summary tokens."""

    def __init__(self):
        super().__init__()
        self.to_q = nn.Linear(256, 256, bias=False)
        self.to_k = nn.Linear(256, 256, bias=False)
        self.to_v = nn.Linear(256, 256, bias=False)
        self.to_out = nn.Sequential(nn.Linear(256, 256), nn.Dropout(0))

    def forward(self, xq, xkv, sigma=None, normalize=False, noise_sampler=None):
        def heads(x):
            return x.reshape(x.shape[0], x.shape[1], 4, 64).transpose(1, 2)

        q = heads(self.to_q(xq))
        k = heads(self.to_k(xkv))
        v = heads(self.to_v(xkv))
        out = F.scaled_dot_product_attention(q, k, v, scale=64**-0.5)
        out = out.transpose(1, 2).reshape(xq.shape[0], xq.shape[1], 256)
        return self.to_out(out)


class SwiGLU(nn.Module):
    def __init__(self):
        super().__init__()
        self.proj_a = nn.Linear(256, 512)
        self.proj_b = nn.Linear(256, 512)

    def forward(self, x):
        return F.silu(self.proj_a(x)) * self.proj_b(x)


class AttentionLayer(nn.Module):
    """Pre-normalized attention and feed-forward residual block."""

    def __init__(self, attention):
        super().__init__()
        self.attn = attention
        self.ff_block = nn.Sequential(SwiGLU(), nn.Dropout(0), nn.Linear(512, 256), nn.Dropout(0))
        self.norm1 = nn.LayerNorm(256)
        self.norm2 = nn.LayerNorm(256)

    def forward(self, xq, xkv=None, sigma=None, normalize=False, noise_sampler=None):
        xkv = xq if xkv is None else xkv
        xq = xq + self.attn(self.norm1(xq), self.norm1(xkv), sigma, normalize, noise_sampler)
        return xq + self.ff_block(self.norm2(xq))


class PrivateSummaryModule(nn.Module):
    """Return the five summary states consumed by the target-side predictor."""

    def __init__(self):
        super().__init__()
        self.latents = nn.Parameter(torch.randn(128, 256))
        self.mhca_ctoq_layers = nn.ModuleList([AttentionLayer(DPMHCA()) for _ in range(3)])
        self.mhsa_layers = nn.ModuleList([AttentionLayer(SelfAttention()) for _ in range(3)])
        self.post_processing_mhsa_layers = nn.ModuleList(
            [AttentionLayer(SelfAttention()) for _ in range(2)]
        )

    def forward(self, context, mu, normalize=False, noise_sampler=torch.randn_like):
        """Summarize embedded context; ``mu`` has one GDP budget per batch item."""
        sigma = 2 * math.sqrt(128 * 3) / mu
        latents = self.latents.unsqueeze(0).expand(context.shape[0], -1, -1)
        states = []
        for private, self_attention in zip(self.mhca_ctoq_layers, self.mhsa_layers):
            latents = self_attention(
                private(latents, context, sigma, normalize, noise_sampler)
            )
            states.append(latents)
        for self_attention in self.post_processing_mhsa_layers:
            latents = self_attention(latents)
            states.append(latents)
        return torch.stack(states, dim=1)
\end{Verbatim}
\endgroup

\subsection{PrivTab Architecture}

PrivTab compresses the private context dataset $\calD_c$ into five fixed-size privatised summary states using privacy-preserving attention layers.
We denote their collection by $\tilde S$ and the state after layer $\ell$ by $\tilde{\bs}_{1:m}^{(\ell)}$.
The model then classifies public target queries using only non-private post-processing.
The detailed architecture is shown in \cref{fig:architecture}.
It has three components: transformer building blocks, a Perceiver-style private encoder, and an output decoder.

\subsubsection{Transformer Building Blocks}

We use a pre-normalisation configuration \citep{Ba2016LayerNorm}, in which Layer Normalisation ($\operatorname{LN}$) is applied before each attention or feed-forward block and the residual connection is added afterwards.

For standard multi-head self-attention ($\operatorname{MHSA}$) and multi-head cross-attention ($\operatorname{MHCA}$), we use the standard transformer formulation~\citep{Vaswani2017Transformers}.
Let $\mathbf{r}_{1:m}$ denote a sequence of query tokens and let $\mathbf{z}_{1:n}$ denote a sequence of key-value tokens, all in $\mathbb{R}^{d_{\mathrm{tok}}}$.
For multi-head cross-attention, the updated representation for query token $i \in \{1,\ldots,m\}$ is computed head-by-head and then concatenated:
\begin{align}
    \operatorname{MHCA}\left(\mathbf{r}_i; \mathbf{z}_{1:n}\right) &= \mathbf{W}_O \left[\mathbf{a}_i^{(1)}, \ldots, \mathbf{a}_i^{(H)}\right] + \mathbf{b}_O \in \mathbb{R}^{d_{\mathrm{tok}}},
\end{align}
where $\mathbf{W}_O \in \mathbb{R}^{d_{\mathrm{tok}} \times H d_h}$ and $\mathbf{b}_O \in \mathbb{R}^{d_{\mathrm{tok}}}$ are the learnable output projection parameters.
Each head output $\mathbf{a}_i^{(h)} \in \mathbb{R}^{d_h}$ for $h=1,\ldots,H$ is given by
\begin{align}
    \mathbf{a}_i^{(h)} &= \sum_{j=1}^n \alpha_{i,j}^{(h)} \mathbf{v}_j^{(h)},
\end{align}
with key, query, and value projections:
\begin{align}
    \mathbf{q}_i^{(h)} = \mathbf{W}_Q^{(h)} \mathbf{r}_i, \qquad
    \mathbf{k}_j^{(h)} = \mathbf{W}_K^{(h)} \mathbf{z}_j, \qquad
    \mathbf{v}_j^{(h)} = \mathbf{W}_V^{(h)} \mathbf{z}_j,
\end{align}
using learnable projections $\mathbf{W}_Q^{(h)}, \mathbf{W}_K^{(h)}, \mathbf{W}_V^{(h)} \in \mathbb{R}^{d_h \times d_{\mathrm{tok}}}$.
The attention weights $\alpha_{i,j}^{(h)}$ are
\begin{align}
    \alpha_{i,j}^{(h)} &= \frac{\exp\left(\frac{\langle \mathbf{q}_i^{(h)}, \mathbf{k}_j^{(h)} \rangle}{\sqrt{d_h}}\right)}{\sum_{j'=1}^n \exp\left(\frac{\langle \mathbf{q}_i^{(h)}, \mathbf{k}_{j'}^{(h)} \rangle}{\sqrt{d_h}}\right)}.
\end{align}
Multi-head self-attention is the special case in which the queries and key-values coincide:
\begin{align}
    \operatorname{MHSA}\left(\mathbf{r}_i; \mathbf{r}_{1:m}\right) &= \operatorname{MHCA}\left(\mathbf{r}_i; \mathbf{r}_{1:m}\right).
\end{align}

For the token-wise feed-forward network ($\operatorname{FFN}$), we use the SwiGLU variant of the gated linear unit \citep{Shazeer2020GLU}.
For a token representation $\mathbf{x} \in \mathbb{R}^{d_{\mathrm{tok}}}$, the SwiGLU activation and resulting FFN output are
\begin{align}
    \operatorname{SwiGLU}(\mathbf{x}) &= \operatorname{SiLU}\left(\mathbf{W}_a^\top \mathbf{x} + \mathbf{b}_a\right) \odot \left(\mathbf{W}_b^\top \mathbf{x} + \mathbf{b}_b\right), \\
    \operatorname{FFN}(\mathbf{x}) &= \mathbf{W}_o^\top \operatorname{SwiGLU}(\mathbf{x}) + \mathbf{b}_o,
\end{align}
where $\operatorname{SiLU}(\mathbf{t}) = \mathbf{t} \odot \operatorname{sigmoid}(\mathbf{t})$ is the sigmoid linear unit, with $\operatorname{sigmoid}(t)=(1+e^{-t})^{-1}$ applied coordinate-wise; $\mathbf{W}_a, \mathbf{W}_b \in \mathbb{R}^{d_{\mathrm{tok}} \times d_{\mathrm{ffn}}}$ and $\mathbf{W}_o \in \mathbb{R}^{d_{\mathrm{ffn}} \times d_{\mathrm{tok}}}$ are learnable projection matrices, $\mathbf{b}_a, \mathbf{b}_b \in \mathbb{R}^{d_{\mathrm{ffn}}}$ and $\mathbf{b}_o \in \mathbb{R}^{d_{\mathrm{tok}}}$ are bias vectors, and $\odot$ denotes the Hadamard product.
In all transformer feed-forward networks, we set $d_{\mathrm{ffn}} = 2d_{\mathrm{tok}}$, so the hidden dimension uses a $2\times$ expansion.

\subsubsection{Private Perceiver Encoder}

The private Perceiver encoder maps the context tokens $\bz_{1:n_c}^c$ and initial target tokens $\bz_{1:n_t}^{t,(0)} = \bz_{1:n_t}^t$ to updated target representations using $L_{\mathrm{main}} = 3$ main layers and $L_{\mathrm{post}} = 2$ post-processing layers.
In the benchmark configuration, the latent bottleneck contains $m=128$ learned pseudo-tokens.
In standard MHSA and MHCA, we use \(H_{\mathrm{std}}=4\) heads with per-head dimension \(d_{h,\mathrm{std}}=64\), so the internal attention width is \(H_{\mathrm{std}}d_{h,\mathrm{std}}=256\).
In DP-MHCA, we use \(H_{\mathrm{DP}}=1\) head with per-head dimension \(d_{h,\mathrm{DP}}=256\), so \(H_{\mathrm{DP}}d_{h,\mathrm{DP}}=256\).
The generic \(H\) and \(d_h\) in the DP-MHCA definitions and theorems above refer to \(H_{\mathrm{DP}}\) and \(d_{h,\mathrm{DP}}\), respectively.
Because the row embedding model is applied row-wise to the private context dataset $\calD_c$, a single-row change between adjacent context datasets affects at most one context token in $\bz_{1:n_c}^c$.
The DP-MHCA block is therefore the component responsible for controlling the sensitivity of the privatised summary.
We initialise the summary representations using the learned pseudo-tokens $\bu_{1:m}$:
\begin{equation}
    \tilde{\bs}_i^{(0)} = \bu_i, \qquad i = 1, \ldots, m.
\end{equation}

The complete blocks used in Methods are denoted by \(\mathrm{DP\text{-}MHCA}_{\mathrm{block}}^{(\ell)}\), \(\mathrm{MHSA}_{\mathrm{block}}^{(\ell)}\), and \(\mathrm{MHCA}_{\mathrm{block}}^{(\ell)}\).
The equations below expand these blocks into attention operations, pre-normalisation, residual updates and feed-forward networks; the attention operators \(\operatorname{DP-MHCA}\), \(\operatorname{MHSA}\), and \(\operatorname{MHCA}\) denote only the attention operations.

For each main layer $\ell = 1, \ldots, L_{\mathrm{main}}$, the summary and target tokens are updated in three stages:
\begin{enumerate}
    \item \textbf{DP Cross-Attention Update}: The summaries attend to the private context tokens:
    \begin{align}
        \bo_i^{(\ell)} &= \operatorname{DP-MHCA}\left(\operatorname{LN}\left(\tilde{\bs}_i^{(\ell-1)}\right); \operatorname{LN}\left(\bz_{1:n_c}^c\right)\right), \qquad i = 1,\ldots,m, \\
        \hat{\bs}_i^{(\ell)} &= \tilde{\bs}_i^{(\ell-1)} + \bo_i^{(\ell)}, \\
        \bs_i^{(\ell)} &= \hat{\bs}_i^{(\ell)} + \operatorname{FFN}_{\text{dp}}^{(\ell)}\left(\operatorname{LN}\left(\hat{\bs}_i^{(\ell)}\right)\right).
    \end{align}
    This DP-MHCA block is the only operation that transfers context information into the released summaries.

    \item \textbf{Latent Self-Attention Update}: The summary representations interact in the latent space:
    \begin{align}
        \mathbf{s}_{i,\text{self}}^{(\ell)} &= \operatorname{MHSA}\left(\operatorname{LN}\left(\bs_i^{(\ell)}\right); \operatorname{LN}\left(\bs_{1:m}^{(\ell)}\right)\right), \qquad i = 1,\ldots,m, \\
        \hat{\bs}_{i,\text{self}}^{(\ell)} &= \bs_i^{(\ell)} + \mathbf{s}_{i,\text{self}}^{(\ell)}, \\
        \tilde{\bs}_i^{(\ell)} &= \hat{\bs}_{i,\text{self}}^{(\ell)} + \operatorname{FFN}_{\text{self}}^{(\ell)}\left(\operatorname{LN}\left(\hat{\bs}_{i,\text{self}}^{(\ell)}\right)\right).
    \end{align}

    \item \textbf{Target Cross-Attention Update}: The target tokens attend to the updated privatised summaries:
    \begin{align}
        \mathbf{t}_j^{(\ell)} &= \operatorname{MHCA}\left(\operatorname{LN}\left(\bz_j^{t,(\ell-1)}\right); \operatorname{LN}\left(\tilde{\bs}_{1:m}^{(\ell)}\right)\right), \qquad j = 1,\ldots,n_t, \\
        \hat{\bz}_j^{t,(\ell)} &= \bz_j^{t,(\ell-1)} + \mathbf{t}_j^{(\ell)}, \\
        \bz_j^{t,(\ell)} &= \hat{\bz}_j^{t,(\ell)} + \operatorname{FFN}_{\text{tgt}}^{(\ell)}\left(\operatorname{LN}\left(\hat{\bz}_j^{t,(\ell)}\right)\right).
    \end{align}
\end{enumerate}

After the main layers, $L_{\mathrm{post}}$ additional post-processing layers refine the representations.
These layers do not access the private context set and therefore preserve the privacy guarantee established by the DP-MHCA blocks.
For each layer $\ell = L_{\mathrm{main}} + 1, \ldots, L_{\mathrm{main}} + L_{\mathrm{post}}$, we perform:
\begin{enumerate}
    \item \textbf{Latent Self-Attention Update}:
    \begin{align}
        \mathbf{s}_{i,\text{self}}^{(\ell)} &= \operatorname{MHSA}\left(\operatorname{LN}\left(\tilde{\bs}_i^{(\ell-1)}\right); \operatorname{LN}\left(\tilde{\bs}_{1:m}^{(\ell-1)}\right)\right), \qquad i = 1,\ldots,m, \\
        \hat{\bs}_i^{(\ell)} &= \tilde{\bs}_i^{(\ell-1)} + \mathbf{s}_{i,\text{self}}^{(\ell)}, \\
        \tilde{\bs}_i^{(\ell)} &= \hat{\bs}_i^{(\ell)} + \operatorname{FFN}_{\text{self}}^{(\ell)}\left(\operatorname{LN}\left(\hat{\bs}_i^{(\ell)}\right)\right).
    \end{align}

    \item \textbf{Target Cross-Attention Update}:
    \begin{align}
        \mathbf{t}_j^{(\ell)} &= \operatorname{MHCA}\left(\operatorname{LN}\left(\bz_j^{t,(\ell-1)}\right); \operatorname{LN}\left(\tilde{\bs}_{1:m}^{(\ell)}\right)\right), \qquad j = 1,\ldots,n_t, \\
        \hat{\bz}_j^{t,(\ell)} &= \bz_j^{t,(\ell-1)} + \mathbf{t}_j^{(\ell)}, \\
        \bz_j^{t,(\ell)} &= \hat{\bz}_j^{t,(\ell)} + \operatorname{FFN}_{\text{tgt}}^{(\ell)}\left(\operatorname{LN}\left(\hat{\bz}_j^{t,(\ell)}\right)\right).
    \end{align}
\end{enumerate}

\Cref{tab:layer-summary} summarises the token flow, query/key-value roles, and privacy status of the encoder sub-blocks.

\begin{table}[h]
\centering
\caption{Summary of operations and token flow in each layer $\ell$ of the PrivTab encoder.
Layer Normalisation ($\operatorname{LN}$) is applied before each operation in a pre-normalisation layout.}
\label{tab:layer-summary}
\small
\begin{tabular}{lllll}
\toprule
\textbf{Sub-block} & \textbf{Query Token(s)} & \textbf{Key-Value Token(s)} & \textbf{Updated Token(s)} & \textbf{Privacy Status} \\
\midrule
\multicolumn{5}{l}{\textit{Main Layers} ($\ell = 1, \ldots, L_{\mathrm{main}}$)} \\
DP Cross-Attention & Summaries $\tilde{\bs}_i^{(\ell-1)}$ & Context $\bz_j^c$ & $\bs_i^{(\ell)}$ & \textbf{Private} (DP-MHCA) \\
Latent Self-Attention & Summaries $\bs_i^{(\ell)}$ & Summaries $\bs_{1:m}^{(\ell)}$ & $\tilde{\bs}_i^{(\ell)}$ & Public (post-processing) \\
Target Cross-Attention & Targets $\bz_j^{t,(\ell-1)}$ & Summaries $\tilde{\bs}_{1:m}^{(\ell)}$ & $\bz_j^{t,(\ell)}$ & Public (post-processing) \\
\midrule
\multicolumn{5}{l}{\textit{Post-Processing Layers} ($\ell = L_{\mathrm{main}} + 1, \ldots, L_{\mathrm{main}} + L_{\mathrm{post}}$)} \\
Latent Self-Attention & Summaries $\tilde{\bs}_i^{(\ell-1)}$ & Summaries $\tilde{\bs}_{1:m}^{(\ell-1)}$ & $\tilde{\bs}_i^{(\ell)}$ & Public (post-processing) \\
Target Cross-Attention & Targets $\bz_j^{t,(\ell-1)}$ & Summaries $\tilde{\bs}_{1:m}^{(\ell)}$ & $\bz_j^{t,(\ell)}$ & Public (post-processing) \\
\bottomrule
\end{tabular}
\end{table}

\subsubsection{Output Decoder}

Let $L = L_{\mathrm{main}} + L_{\mathrm{post}}$ denote the total number of encoder layers.
The final target representations $\bz_{1:n_t}^{t,(L)}$ are decoded row-wise to obtain predictive logits.
The decoder is a multi-layer perceptron with GELU activations \citep{Hendrycks2016GELU}:
\begin{align}
    \mathbf{h}_j^{(1)} &= \operatorname{GELU}\left(\mathbf{W}_{d,1}^{\top}\bz_j^{t,(L)} + \mathbf{b}_{d,1}\right), \\
    \mathbf{h}_j^{(2)} &= \operatorname{GELU}\left(\mathbf{W}_{d,2}^{\top}\mathbf{h}_j^{(1)} + \mathbf{b}_{d,2}\right), \\
    \hat{\mathbf{y}}_j &= \mathbf{W}_{d,3}^{\top}\mathbf{h}_j^{(2)} + \mathbf{b}_{d,3},
\end{align}
where $\hat{\mathbf{y}}_j \in \mathbb{R}^{C_{\mathrm{max}}}$ are the raw output logits for target query $j$, $\mathbf{W}_{d,1} \in \mathbb{R}^{d_{\mathrm{tok}} \times d_{\mathrm{dec}}}$, $\mathbf{W}_{d,2} \in \mathbb{R}^{d_{\mathrm{dec}} \times d_{\mathrm{dec}}}$, and $\mathbf{W}_{d,3} \in \mathbb{R}^{d_{\mathrm{dec}} \times C_{\mathrm{max}}}$ are learnable weight matrices, and $\mathbf{b}_{d,1}, \mathbf{b}_{d,2}, \mathbf{b}_{d,3}$ are bias vectors.
We set $d_{\mathrm{dec}}=512$.
If a dataset has $C \le C_{\mathrm{max}}$ classes, only the first $C$ coordinates of $\hat{\mathbf{y}}_j$ are used.

\FloatBarrier
\section{PrivTab Pretraining}\label{app:dp-tnp-pretraining}\label{sec:dp-tnp-pretraining}\label{app:dp-tnp-pretraining-details}

\subsection{Noise-Aware Pretraining Objective}
We pretrain PrivTab by meta-learning across synthetic tabular classification tasks.
Each task \(\btau\) comprises a labelled context set \(\calD_c\) and a target set \(\calD_t\).
The DP-MHCA-based private encoder \(\calA_{\bphi}\) maps \(\calD_c\) to a released privatised summary \(\tilde{S}\), and the predictive network \(q_{\bpsi}\) maps \((\bx_{1:n_t}^{t},\tilde S)\) to a distribution over the target labels.
We use \(\bphi\) for parameters that determine the private summary and \(\bpsi\) for parameters that determine predictions.
The row embedding network is shared between context and query paths, so these parameter collections overlap in those weights and are trained jointly by minimising target cross-entropy, equivalently the population private negative log-likelihood
\begin{equation}
    \mathcal{L}^{\mathrm{priv}}_{\mathrm{NLL}}(\bpsi;\bphi)
    :=
    \mathbb{E}_{p_{\bphi}(\btau,\tilde{S})}
    \left[
        -\log q_{\bpsi}
        \left(
            y_{1:n_t}^{t}
            \mid
            \bx_{1:n_t}^{t},
            \tilde{S}
        \right)
    \right].
\end{equation}
As shown in \cref{app:noise-aware-amortized-inference}, this is a noise-aware objective: for a fixed private encoder, minimising it fits the posterior predictive distribution conditioned on the privatised summary rather than a predictor that ignores the privacy noise.

Let \(N_t(\btau) := |\mathcal{D}_t|\) denote the target-set size, and let \(p_{N_t}\) be its induced distribution under the task simulator.
For each fixed \(n_t\), let \(p_{\bphi}^{(n_t)}(\btau,\tilde{S})\) denote the joint law of \((\btau,\tilde{S})\) conditional on \(N_t=n_t\).
Applying \cref{thm:nll-noise-aware-predictive-kl} conditionally on \(N_t=n_t\) and then averaging over \(N_t\) gives
\begin{equation}
    \mathcal{L}^{\mathrm{priv}}_{\mathrm{NLL}}(\bpsi;\bphi)
    =
    \sum_{n_t \geq 1}
    p_{N_t}(n_t)
    \left(
        \mathrm{const}(\bphi,n_t)
        +
        \mathbb{E}_{p_{\bphi}^{(n_t)}(\btau,\tilde{S})}
        \left[
            \mathrm{KL}
            \left(
                p_{\bphi}(\cdot \mid \bx_{1:n_t}^{t},\tilde{S})
                \,\middle\|\,
                q_{\bpsi}(\cdot \mid \bx_{1:n_t}^{t},\tilde{S})
            \right)
        \right]
    \right).
\end{equation}
Thus, when \(\bphi\) is fixed, optimizing over \(\bpsi\) amounts to fitting \(q_{\bpsi}\) to the noise-aware posterior predictive induced by \(\calA_{\bphi}\).
Joint training is therefore the optimisation problem
\begin{equation}
    \min_{\bphi,\bpsi}
    \mathcal{L}^{\mathrm{priv}}_{\mathrm{NLL}}(\bpsi;\bphi),
\end{equation}
in which \(\bphi\) changes both the distribution of the released summary, through \(p_{\bphi}(\btau,\tilde{S})\), and the corresponding family of conditional noise-aware posterior predictives \(p_{\bphi}(\cdot \mid \bx_{1:n_t}^{t},\tilde{S})\).
A useful equivalent view is the profile objective
\begin{equation}
    \mathcal{J}(\bphi)
    :=
    \inf_{\bpsi}
    \mathcal{L}^{\mathrm{priv}}_{\mathrm{NLL}}(\bpsi;\bphi).
\end{equation}
This means that \(\bphi\) is chosen to define a private mechanism whose induced noise-aware posterior predictive is both informative for prediction and well approximated by the predictive family \(q_{\bpsi}\).

In practice, we first update the model jointly and then freeze the entire encoder late in training, including the shared row embedding, private summary blocks and query-side attention.
Only the output decoder remains trainable in this final stage, refining predictions for the fixed representations learned earlier.

\subsection{Simulator, Curriculum, and Optimisation}

To instantiate this objective, we draw the synthetic tasks from the mixed structural causal model (SCM) simulator introduced in TabICL \citep{Qu2025TabICL}.
Each generated task supplies the labelled context and target sets used by the objective above.

\paragraph{Synthetic Tabular Data Simulator.}
For each simulated task, we generate a synthetic tabular classification dataset:
\begin{align}
    \mathcal{D}_c &= \left\{ \left(\mathbf{x}_i^c, y_i^c\right) \right\}_{i=1}^{n_c}, \\
    \mathcal{D}_t &= \left\{ \left(\mathbf{x}_j^t, y_j^t\right) \right\}_{j=1}^{n_t},
\end{align}
where the number of active features $d$ is sampled uniformly from $[2, 120]$ and the input feature dimension is padded to $d_x = 120$.
The number of target classes $C$ is sampled uniformly up to $C_{\mathrm{max}} = 10$.
The context size $n_c$ and target size $n_t$ are sampled independently and uniformly from $[N_{\mathrm{min}}, N_{\mathrm{max}}]$ for each task.
To improve optimisation stability and sample efficiency, we employ a two-stage training curriculum that adjusts the range $[N_{\mathrm{min}}, N_{\mathrm{max}}]$:
\begin{enumerate}
    \item \textbf{Stage 1}: For the first $128{,}000$ steps, we set $N_{\mathrm{min}} = 100$ and $N_{\mathrm{max}} = 2048$.
In this stage, the scale normalisation in \cref{eq:dp-mhca-normalization} is disabled.
    \item \textbf{Stage 2}: For training steps $128{,}000$ to $256{,}000$, we set $N_{\mathrm{min}} = 1024$ and $N_{\mathrm{max}} = 8192$, and enable the scale normalisation in \cref{eq:dp-mhca-normalization}.
\end{enumerate}
This curriculum enables the model to first stabilise its representation learning on smaller datasets before adapting to larger context scales.

To optimise performance across different dataset scales, we train two separate model variants: one trained solely on Stage 1 tasks for an additional $256{,}000$ steps (yielding $384{,}000$ total steps on smaller contexts), and another trained sequentially on Stage 1 and Stage 2 tasks.
The small-context variant is trained without summary normalisation, which is enabled at inference; the large-context variant enables it during Stage 2 and at inference.
At inference time, we route queries dynamically based on the context size: the smaller-context model is used for datasets with $n_c < 4096$, while the large-context model is used for datasets with $n_c \geq 4096$.

\paragraph{Privacy Parameter Curriculum.}
To train the model across a wide range of privacy levels, each simulated task is assigned a privacy parameter $\mu$ according to a log-uniform curriculum during meta-training.
The overall range is $[\mu_{\mathrm{min}}, \mu_{\mathrm{max}}] = [0.16,64.0]$.
To ease optimisation at the start of training, we initialise the curriculum with minimal privacy noise by setting the lower bound to $\mu_{\mathrm{min}}^{(\mathrm{curr})} = 64.0$.
From steps $256$ to $38{,}400$, the lower bound $\mu_{\mathrm{min}}^{(\mathrm{curr})}$ decreases log-linearly from $64.0$ to $0.16$, while the upper bound remains fixed at $64.0$.
Afterwards, the curriculum reaches its final state and $\mu$ is sampled according to
\begin{equation}
    \log \mu = u \left(\log \mu_{\mathrm{max}} - \log \mu_{\mathrm{min}}\right) + \log \mu_{\mathrm{min}},
\end{equation}
where $u \sim \operatorname{Uniform}(0, 1)$.
Equivalently, in the final regime \(\mu\) is sampled from the density
\begin{equation}
    p(\mu)
    =
    \frac{1}{\mu\left(\log \mu_{\mathrm{max}} - \log \mu_{\mathrm{min}}\right)}
    \mathbf{1}_{[\mu_{\mathrm{min}},\,\mu_{\mathrm{max}}]}(\mu).
\end{equation}
The resulting pretraining objective averages the private negative log-likelihood over both target sizes and privacy levels.
Applying \cref{thm:nll-noise-aware-predictive-kl} pointwise in \((\mu,n_t)\) shows that, in the final regime, training minimises the KL divergence to the corresponding noise-aware posterior predictive on average over both tasks and privacy levels.
The derivation below makes this averaging explicit.
The curriculum therefore exposes the model first to lower-noise tasks and gradually introduces more challenging high-noise settings.

\paragraph{Objective across privacy levels and target sizes.}
The resulting pretraining objective is
\begin{equation}
    \mathcal{L}^{\mathrm{priv}}_{\mathrm{NLL}}(\bpsi;\bphi)
    =
    \int_{\mu_{\mathrm{min}}}^{\mu_{\mathrm{max}}}
    \sum_{n_t \geq 1}
    \mathcal{L}^{\mathrm{priv}}_{\mathrm{NLL}}(\bpsi;\bphi,\mu,n_t)\,
    p_{N_t}(n_t)
    p(\mu)\,\mathrm{d}\mu,
\end{equation}
where
\begin{equation}
    \mathcal{L}^{\mathrm{priv}}_{\mathrm{NLL}}(\bpsi;\bphi,\mu,n_t)
    :=
    \mathbb{E}_{p_{\bphi,\mu}^{(n_t)}(\btau,\tilde{S})}
    \left[
        -\log q_{\bpsi}
        \left(
            y_{1:n_t}^{t}
            \mid
            \bx_{1:n_t}^{t},
            \tilde{S}
        \right)
    \right].
\end{equation}
Here \(p_{\bphi,\mu}^{(n_t)}(\btau,\tilde{S})\) denotes the joint law conditional on both the privacy level \(\mu\) and the target size \(N_t=n_t\).
Applying \cref{thm:nll-noise-aware-predictive-kl} for each fixed pair \((\mu,n_t)\) gives
\begin{equation}
    \begin{split}
    \mathcal{L}^{\mathrm{priv}}_{\mathrm{NLL}}(\bpsi;\bphi)
    =
    &\int_{\mu_{\mathrm{min}}}^{\mu_{\mathrm{max}}}
    \sum_{n_t \geq 1}
    \left(
        \mathrm{const}(\bphi,\mu,n_t)
        \right.
        \\
    &\qquad\left.
        +
        \mathbb{E}_{p_{\bphi,\mu}^{(n_t)}(\btau,\tilde{S})}
        \left[
            \mathrm{KL}
            \left(
                p_{\bphi,\mu}(\cdot \mid \bx_{1:n_t}^{t},\tilde{S})
                \,\middle\|\,
                q_{\bpsi}(\cdot \mid \bx_{1:n_t}^{t},\tilde{S})
            \right)
        \right]
    \right)
    p_{N_t}(n_t)\, p(\mu)\,\mathrm{d}\mu,
    \end{split}
\end{equation}
where \(\mathrm{const}(\bphi,\mu,n_t)\) does not depend on \(\bpsi\).
Equivalently, writing
\begin{equation}
    \overline{\mathrm{const}}(\bphi)
    :=
    \int_{\mu_{\mathrm{min}}}^{\mu_{\mathrm{max}}}
    \sum_{n_t \geq 1}
    \mathrm{const}(\bphi,\mu,n_t)\,
    p_{N_t}(n_t)\,
    p(\mu)\,\mathrm{d}\mu,
\end{equation}
we obtain
\begin{equation}
    \begin{split}
    \mathcal{L}^{\mathrm{priv}}_{\mathrm{NLL}}(\bpsi;\bphi)
    =
    &\overline{\mathrm{const}}(\bphi)
    +
    \int_{\mu_{\mathrm{min}}}^{\mu_{\mathrm{max}}}
    \sum_{n_t \geq 1}
    \\
    &\qquad
    \mathbb{E}_{p_{\bphi,\mu}^{(n_t)}(\btau,\tilde{S})}
    \left[
        \mathrm{KL}
        \left(
            p_{\bphi,\mu}(\cdot \mid \bx_{1:n_t}^{t},\tilde{S})
            \,\middle\|\,
            q_{\bpsi}(\cdot \mid \bx_{1:n_t}^{t},\tilde{S})
        \right)
    \right]
    p_{N_t}(n_t)\, p(\mu)\,\mathrm{d}\mu.
    \end{split}
\end{equation}

\paragraph{Optimisation and Loss Function.}
The model receives the context set $\calD_c$ and target features $\mathbf{x}_{1:n_t}^t$ as inputs.
The training loss is the cross-entropy, equivalently the negative log-likelihood, computed on the target labels $y_{1:n_t}^t$ for the active classes.
We apply label smoothing \citep{Szegedy2016Rethinking} with smoothing factor $0.1$.
The theoretical result in \Cref{thm:nll-noise-aware-predictive-kl} characterises the corresponding unsmoothed NLL objective; label smoothing is used only as a training regulariser.

We optimise with AdamW \citep{Loshchilov2017AdamW} using peak learning rate $10^{-4}$, weight decay $10^{-3}$, $\beta_1 = 0.9$, $\beta_2 = 0.98$, and $\epsilon = 10^{-8}$.
We apply gradient clipping with maximum norm $1.0$.
The learning rate follows a warmup-cosine-constant schedule: it increases linearly from $0$ to $10^{-4}$ over the first $6400$ steps, decays with a cosine schedule to $\eta_{\mathrm{min}} = 10^{-5}$ over $51{,}200$ optimiser steps, and remains constant at $10^{-5}$ thereafter.

The large-context variant is trained for $256{,}000$ total steps with a batch size of 128 tasks (datasets).
At step $243{,}200$, we freeze the entire encoder and continue training only the output decoder for the remaining $12{,}800$ steps.
The small-context variant is trained for $384{,}000$ total steps; its final encoder-freezing stage starts at step $371{,}200$, also leaving $12{,}800$ steps to refine the decoder.
These final stages stabilise the predictive approximation relative to the learned privatised representation.

\FloatBarrier
\section{Baseline Details}\label{app:baseline-details}

\subsection{Accounting for DP Stochastic Optimisation under Substitute Adjacency}
\label{app:dp-stochastic-optimisation-accounting}

For DP stochastic optimisation \citep{Song2013DPSGD,Bassily2014DPEmpiricalRisk,Abadi2016DPSGD}, privacy is defined at the row level, so we use the same substitute adjacency relation as in the rest of the paper.
Let \(\calD=\{\bz_i\}_{i=1}^{N}\) denote the training set, let \(q\) be the Poisson subsampling rate, let \(\bar B=\max\{1,\lfloor qN\rfloor\}\) be the normalisation based on the expected minibatch size, let \(T\) be the number of optimisation steps, let \(C>0\) be the clipping norm, and let \(\sigma>0\) be the standard deviation of the noise added to the clipped gradient sum.
At each step \(t\), DP Adam samples a minibatch \(B_t \subseteq \calD\), clips each per-example gradient to have \(\ell_2\)-norm at most \(C\), and applies Adam to a noisy gradient of the form
\begin{equation}
    \tilde{\mathbf{g}}_t
    =
    \frac{1}{\bar B}
    \sum_{\bz_i \in B_t}
    \operatorname{clip}_C\!\left(\nabla_{\bw}\ell(\bw_t; \bz_i)\right)
    +
    \frac{\sigma}{\bar B}\,\boldsymbol{\xi}_t,
    \qquad
    \boldsymbol{\xi}_t \sim \mathcal{N}(\mathbf{0}, \mathbf{I}).
\end{equation}

To report privacy in the same form as for PrivTab, we do not use the asymptotic \(\mu\)-GDP approximation for subsampled Gaussian-gradient mechanisms \citep{Bu2020GDP}.
Instead, we compute the privacy profile of the Poisson-subsampled Gaussian mechanism under the replace-one neighbouring relation using a numerical privacy-loss-distribution accountant~\citep{Gomez2025GaussianDP}, compose it over \(T\) steps, and then convert the resulting trade-off function \(T\) to the smallest valid non-asymptotic \(\mu\)-GDP guarantee,
\begin{equation}
    \mu^\star
    =
    \inf\left\{
    \mu \geq 0 :
    T_\mu(\alpha) \leq T(\alpha)
    \text{ for all } \alpha \in [0,1]
    \right\}.
\end{equation}
where \(T_\mu\) is the GDP trade-off function from \cref{eq:mu-gdp-def-eq-1}.
For each target privacy budget \(\mu\), the DP Adam noise standard deviation is then chosen by numerically inverting this map so that the full training procedure satisfies \(\mu\)-GDP under substitute adjacency.

\subsection{Hyperparameter Calibration and Private Selection}

Search spaces, private learning-rate grids and fixed baseline hyperparameters are reported in \cref{tab:baseline_hyperparameters,tab:dp-baseline-lr-grids,tab:dp-baseline-fixed-hypers}.

\begin{table}[h!]
\centering
\caption{Baseline hyperparameter search spaces.
Bayesian optimisation is run separately for each private baseline.}
\label{tab:baseline_hyperparameters}
\begin{tabular}{lll}
\toprule
\textbf{Baseline} & \textbf{Hyperparameter} & \textbf{Search space} \\
\midrule
DP-LR
    & Learning rate & $[10^{-4}, 10^{-2}]$ log-uniform \\
    & Batch size & $\{128,256,512,1024\}$ \\
    & Clipping norm & $\{0.1,0.5,1,2,5\}$ \\
    & Epochs & $\{100,300,500,1000\}$ \\
\midrule
DP-MLP
    & Learning rate & $[10^{-4}, 10^{-2}]$ log-uniform \\
    & Batch size & $\{128,256,512,1024\}$ \\
    & Clipping norm & $\{0.1,0.5,1,2,5\}$ \\
    & Hidden width & $\{64,128,256,512\}$ \\
    & Layers & $\{1,2,3\}$ \\
    & Epochs & $\{100,300,500,1000\}$ \\
\bottomrule
\end{tabular}
\end{table}

\begin{figure*}[!ht]
\centering
\includegraphics[width=\textwidth]{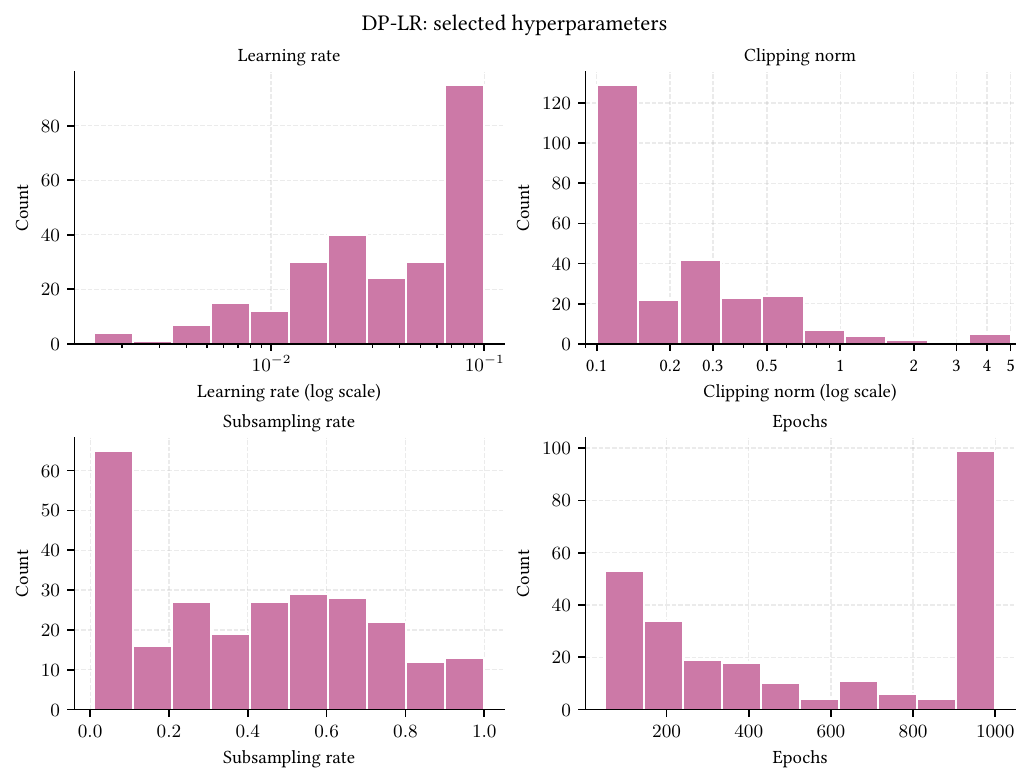}
\caption{Histogram of the hyperparameters selected for DP-LR across all datasets and context--target splits.
The selected values are concentrated enough to justify fixing the non-learning-rate hyperparameters in the final experiments.}
\label{fig:dp-logreg-hypers-tabllm}
\end{figure*}

\begin{figure*}[!ht]
\centering
\includegraphics[width=\textwidth]{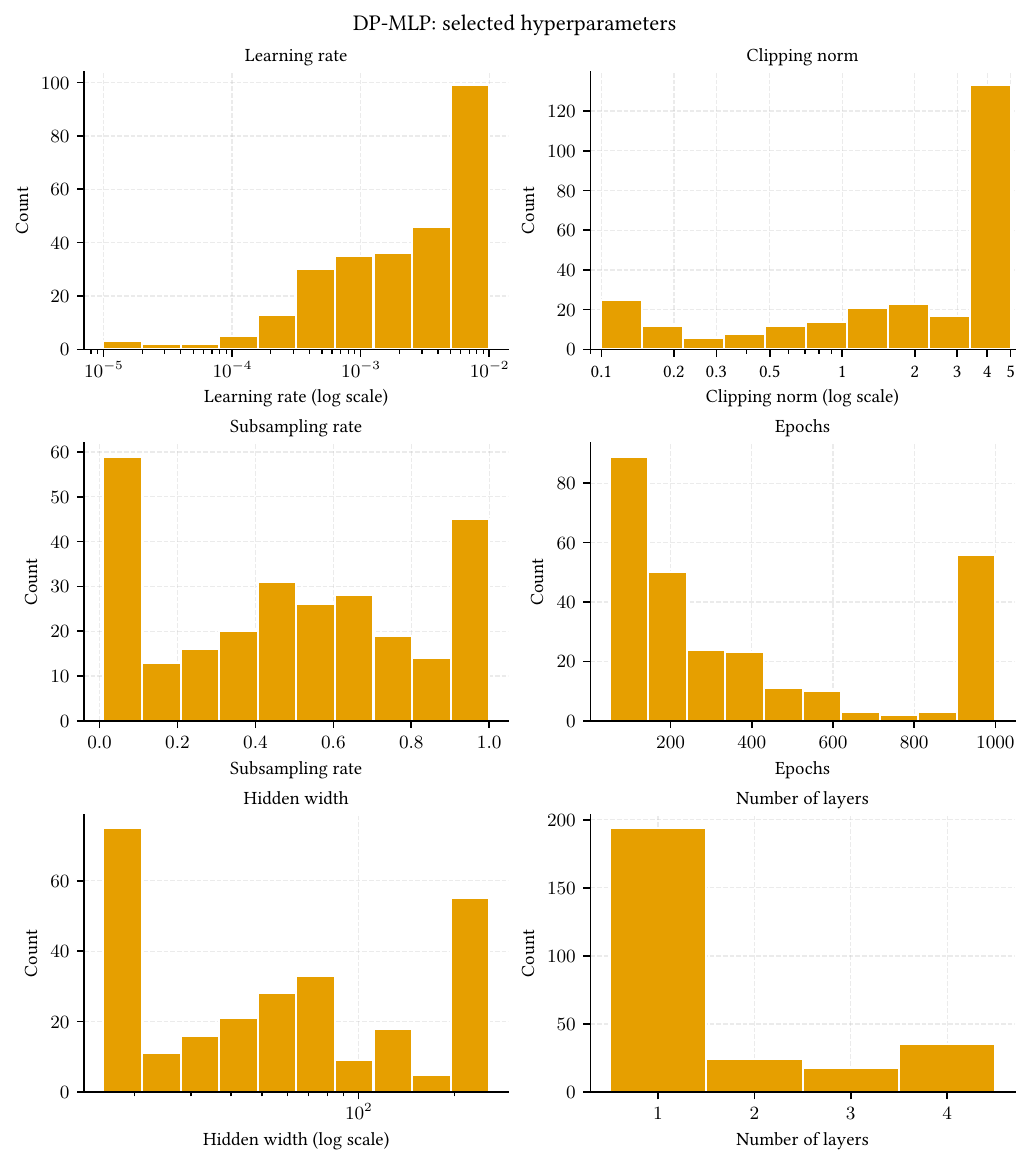}
\caption{Histogram of the hyperparameters selected for DP-MLP across all datasets and context--target splits.
As for DP-LR, the selected values are concentrated enough to justify fixing the non-learning-rate hyperparameters in the final experiments.}
\label{fig:dp-mlp-hypers-tabllm}
\end{figure*}

\paragraph{DP baseline hyperparameter selection.}
Based on \cref{fig:dp-logreg-hypers-tabllm,fig:dp-mlp-hypers-tabllm}, we fix the main hyperparameters of both baselines and tune only the learning rate under DP.
For each method, the learning rate is selected from a small grid using a private validation loss.
Concretely, after the outer 80/20 context--target split, we divide the context portion into subtraining and validation sets, using a validation fraction of $0.16$ within the context set.
Relative to the full dataset, this corresponds approximately to a $67/13/20$ split into subtraining, validation and target data.
For any single learning-rate candidate, the DP Adam training run acts only on the subtraining subset while the noisy validation-loss evaluation acts only on the validation subset, so these two components compose by parallel composition \citep{McSherry2018ParallelComp}.
Across the learning-rate grid, however, the same subtraining subset is reused for the different candidate training runs and the same validation subset is reused for the different candidate validation evaluations, so we also account for composition over all learning-rate candidates, that is, over the grid size.
After private learning-rate selection, we perform one additional final DP Adam training run on the union of the subtraining and validation sets using the selected learning rate.
In the final protocol, we reserve a fraction $0.5$ of the total \(\mu\) budget for private learning-rate selection and use the remaining budget for the final training run through \(\ell_2\)-composition in \(\mu\)-GDP.
The private validation criterion is the mean cross-entropy loss after clipping each per-example validation loss at $5.0$.

\begin{table}[!ht]
\centering
\caption{Learning-rate grids used for private model selection.
Only the learning rate is selected under DP; all other hyperparameters are fixed from the aggregate hyperparameter-optimisation results in \cref{fig:dp-logreg-hypers-tabllm,fig:dp-mlp-hypers-tabllm}.}
\label{tab:dp-baseline-lr-grids}
\begin{tabular}{ll}
\toprule
\textbf{Method} & \textbf{Learning-rate grid} \\
\midrule
DP-LR
& $\{0.005,\;0.02,\;0.05,\;0.1\}$ \\
DP-MLP
& $\{3{\times}10^{-4},\;10^{-3},\;3{\times}10^{-3},\;10^{-2}\}$ \\
\bottomrule
\end{tabular}
\end{table}

\begin{table}[!ht]
\centering
\caption{Fixed hyperparameters used for the DP baselines after inspecting the aggregate hyperparameter-selection histograms.}
\label{tab:dp-baseline-fixed-hypers}
\begin{tabular}{lll}
\toprule
\textbf{Method} & \textbf{Hyperparameter} & \textbf{Value} \\
\midrule
DP-LR & Clipping norm & $0.1$ \\
DP-LR & Subsampling rate $q$ & $0.05$ \\
DP-LR & Epochs & $1000$ \\
\midrule
DP-MLP & Clipping norm & $5.0$ \\
DP-MLP & Subsampling rate $q$ & $0.05$ \\
DP-MLP & Epochs & $250$ \\
DP-MLP & Hidden width & $64$ \\
DP-MLP & Hidden layers & $1$ \\
\bottomrule
\end{tabular}
\end{table}

\FloatBarrier
\section{Additional Experiment Details}\label{app:additional-experiment-details}

\subsection{Membership-Inference Audit Results}\label{app:membership-inference-results}

\Cref{tab:membership-inference} reports exact FPC-corrected means and complementary 95th-percentile point estimates, while \cref{fig:membership-inference-summary}a emphasises the maximum lower record-wise 95\% confidence endpoint.
The record-level curves in \cref{fig:membership-inference} provide a threshold-specific view of the exposure summarised by the record AUC values.

\begin{table}[!ht]
\centering
\scriptsize
\caption{FPC-corrected membership-inference audit on Maternal Health Risk.
Aggregate AUC measures attack performance over 100 held-out context draws; record AUC measures each row's fitted exposure across 4,000 draws.
Means are reported with bootstrap 95\% CIs over held-out draws or rows, respectively, and the record-level 95th percentile summarises the most exposed 5\% of rows.
Values near \(0.5\) indicate perfect privacy.}
\label{tab:membership-inference}
\begin{tabular}{lccc}
\toprule
Method & Aggregate AUC (mean [95\% CI]) & Record AUC (mean [95\% CI]) & Record AUC (95th percentile) \\
\midrule
PrivTab (\(\mu=0.05\)) & \(0.501\ [0.498,0.505]\) & \(0.504\ [0.504,0.504]\) & \(0.513\) \\
PrivTab (\(\mu=0.1\)) & \(0.503\ [0.500,0.507]\) & \(0.505\ [0.505,0.506]\) & \(0.515\) \\
PrivTab (\(\mu=0.2\)) & \(0.508\ [0.504,0.511]\) & \(0.508\ [0.508,0.509]\) & \(0.519\) \\
PrivTab (\(\mu=0.4\)) & \(0.515\ [0.511,0.518]\) & \(0.514\ [0.514,0.515]\) & \(0.526\) \\
PrivTab (\(\mu=0.8\)) & \(0.527\ [0.524,0.530]\) & \(0.523\ [0.522,0.523]\) & \(0.537\) \\
PrivTab (\(\mu=1.6\)) & \(0.543\ [0.540,0.546]\) & \(0.533\ [0.532,0.534]\) & \(0.554\) \\
TabPFN v3 (non-private) & \(0.752\ [0.749,0.755]\) & \(0.686\ [0.680,0.692]\) & \(0.864\) \\
TabICL v2 (non-private) & \(0.765\ [0.762,0.767]\) & \(0.708\ [0.701,0.715]\) & \(0.977\) \\
\bottomrule
\end{tabular}
\end{table}

\begin{figure}[!ht]
    \centering
    \includegraphics[width=\textwidth]{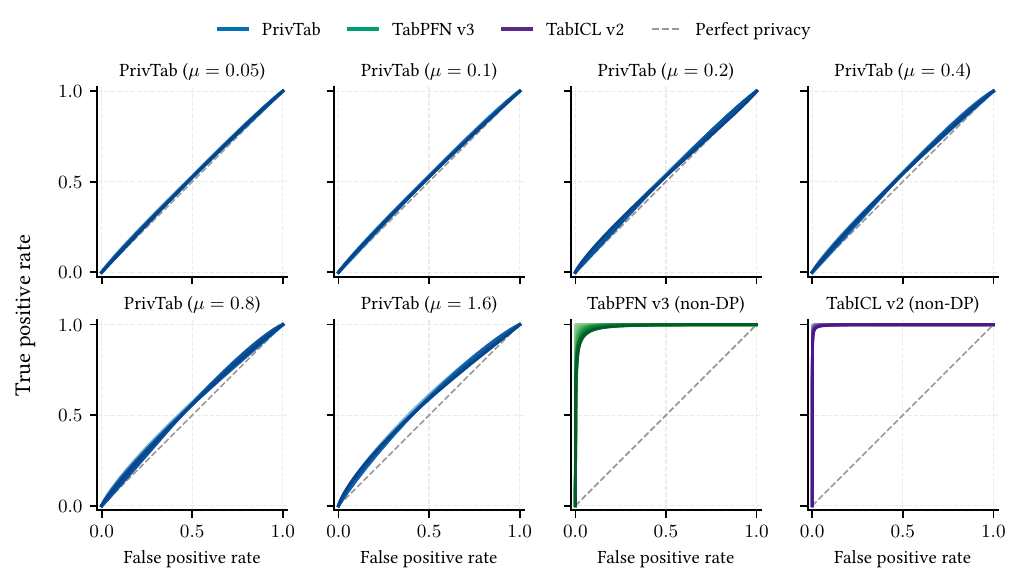}
    \caption{FPC-corrected record-level membership-inference curves.
ROC curves are shown for the ten most vulnerable Maternal Health Risk records at each setting.
Shades distinguish records, and the diagonal denotes perfect privacy.
The fitted IN/OUT standard deviations use the same FPC as \cref{fig:membership-inference-summary}a.
PrivTab remains substantially closer to the diagonal than non-private TabPFN v3 and TabICL v2, whose most vulnerable records are nearly perfectly distinguishable.}
    \label{fig:membership-inference}
\end{figure}

\paragraph{Fitted balanced attack accuracy and its uncertainty.}
For a fixed record \(i\), let \(S_{i,\mathrm{IN}}\) and \(S_{i,\mathrm{OUT}}\) denote its attack score conditional on the record being a member or a non-member, respectively.
We fit separate Gaussian distributions,
\[
 S_{i,\mathrm{IN}}\sim\mathcal N(\mu_{i,\mathrm{IN}},\sigma_{i,\mathrm{IN}}^2),
 \qquad
 S_{i,\mathrm{OUT}}\sim\mathcal N(\mu_{i,\mathrm{OUT}},\sigma_{i,\mathrm{OUT}}^2),
\]
where \(\sigma_{i,\mathrm{IN}}\) and \(\sigma_{i,\mathrm{OUT}}\) are the finite-population-corrected standard deviations defined in the Methods.
A score above threshold \(t\) predicts membership, so, writing \(\Phi\) for the standard-normal CDF,
\[
 \operatorname{TPR}_i(t)
 =\Pr(S_{i,\mathrm{IN}}>t)
 =1-\Phi\!\left(\frac{t-\mu_{i,\mathrm{IN}}}{\sigma_{i,\mathrm{IN}}}\right),
 \qquad
 \operatorname{TNR}_i(t)
 =\Pr(S_{i,\mathrm{OUT}}\leq t)
 =\Phi\!\left(\frac{t-\mu_{i,\mathrm{OUT}}}{\sigma_{i,\mathrm{OUT}}}\right).
\]
With \(P=N=2{,}000\) IN and OUT draws for each record, ordinary accuracy and balanced accuracy have the same form:
\[
 \frac{\mathrm{TP}+\mathrm{TN}}{P+N}
 =\frac12\left(\frac{\mathrm{TP}}{P}+\frac{\mathrm{TN}}{N}\right).
\]
Replacing the two observed fractions by the fitted Gaussian probabilities gives the fitted balanced accuracy
\[
 A_i(t;\boldsymbol\omega_i)
 =\frac12\left[\operatorname{TPR}_i(t)+\operatorname{TNR}_i(t)\right]
 =\frac12\left[
 \Phi\!\left(\frac{\mu_{i,\mathrm{IN}}-t}{\sigma_{i,\mathrm{IN}}}\right)
 +\Phi\!\left(\frac{t-\mu_{i,\mathrm{OUT}}}{\sigma_{i,\mathrm{OUT}}}\right)
 \right],
\]
Here \(\boldsymbol\omega_i=(\mu_{i,\mathrm{IN}},\mu_{i,\mathrm{OUT}},\sigma_{i,\mathrm{IN}},\sigma_{i,\mathrm{OUT}})^{\mathsf T}\) contains the score-model parameters; \(\btheta\) denotes the foundation-model parameters.
We choose \(\hat t_i\in\arg\max_t A_i(t;\hat{\boldsymbol\omega}_i)\) and write \(\hat A_i^*=A_i(\hat t_i;\hat{\boldsymbol\omega}_i)\), including the limiting accuracy of \(0.5\) from predicting a single class.

To quantify uncertainty in this fitted accuracy, let \(n_{i,\mathrm{IN}}\) and \(n_{i,\mathrm{OUT}}\) be the numbers of fitted scores, and let \(s_{i,\mathrm{IN}}\) and \(s_{i,\mathrm{OUT}}\) be their uncorrected sample standard deviations.
The fitted corrected standard deviations are \(\hat\sigma_{i,\mathrm{IN}}=s_{i,\mathrm{IN}}/\sqrt f\) and \(\hat\sigma_{i,\mathrm{OUT}}=s_{i,\mathrm{OUT}}/\sqrt f\), with \(f=0.5\).
Under independent Gaussian IN and OUT score samples, a large-sample diagonal covariance approximation for \(\hat{\boldsymbol\omega}_i\) is
\[
 \widehat V_i=\operatorname{diag}\!\left(
 \frac{s_{i,\mathrm{IN}}^2}{n_{i,\mathrm{IN}}},
 \frac{s_{i,\mathrm{OUT}}^2}{n_{i,\mathrm{OUT}}},
 \frac{\hat\sigma_{i,\mathrm{IN}}^2}{2n_{i,\mathrm{IN}}},
 \frac{\hat\sigma_{i,\mathrm{OUT}}^2}{2n_{i,\mathrm{OUT}}}
 \right).
\]
The first two entries use the variance of the observed score means; the last two propagate the fixed FPC scaling of the fitted standard deviations.
For a fixed threshold, the delta method gives
\[
 \widehat{\operatorname{Var}}\!\left[A_i(t;\hat{\boldsymbol\omega}_i)\right]
 =\boldsymbol g_i(t)^{\mathsf T}\widehat V_i\boldsymbol g_i(t),
 \qquad
 \boldsymbol g_i(t)=\nabla_{\boldsymbol\omega}A_i(t;\hat{\boldsymbol\omega}_i),
\]
and its standard error is the square root of this variance; we obtain the gradient by automatic differentiation.
At a unique interior maximum, \(\partial A_i/\partial t=0\), so the envelope theorem gives the first-order gradient of \(\max_t A_i(t;\boldsymbol\omega_i)\) by evaluating \(\boldsymbol g_i(t)\) at \(\hat t_i\), without differentiating through the selected threshold.
Thus the nominal record-wise 95\% interval for the maximised fitted accuracy is
\[
 \hat A_i^*\ \pm\ 1.96\,
 \sqrt{\boldsymbol g_i(\hat t_i)^{\mathsf T}\widehat V_i\boldsymbol g_i(\hat t_i)}.
\]
For \cref{fig:dp-tnp-overview}e, we clip its lower endpoint at \(0.5\) and plot the largest lower endpoint across records for each method.
These intervals are pointwise for each record; choosing the largest endpoint across records does not provide a simultaneous 95\% coverage guarantee.

\FloatBarrier

\subsection{Verification and Empirical Audits}\label{app:empirical-audits}

We verify the private summary mechanism at three complementary levels: implementation, formal verification and empirical auditing.
Together, these checks connect the mathematical privacy argument to the code that implements the mechanism.
Empirical agreement cannot replace a proof, but it can reveal implementation errors or violated assumptions that the abstract proof alone cannot detect.

\paragraph{Implementation.}
DP-MHCA localises the sensitivity-critical query, key and value transformations, bounded attention computation and Gaussian perturbation in one module.
The noise scale is calibrated from the number of private layers, attention heads and summary tokens according to the sensitivity and composition results in \cref{thm:dp-mhca-sensitivity,thm:dp-tnp-gdp}.
We also check that every transformation before the first DP-MHCA layer is applied independently to each context row, with no positional encoding or cross-row normalisation that could create an unaccounted dependency.
A self-contained PyTorch implementation that closely follows the implementation used in our experiments is provided in \cref{app:dp-mhca-reference-implementation}.

\paragraph{Formal verification.}
We provide a Lean 4 formalisation in \texttt{DPMHCA.Lean} \citep{deMouraUllrich2021Lean4}.
It formally verifies the \(2\sqrt{Hm}\) sensitivity bound for bounded attention and an abstract adaptive \(\mu\)-GDP composition theorem that permits deterministic post-processing between private layers.
This complements the manuscript proofs of the single-layer sensitivity bound and the privacy guarantee of the complete private summary mechanism.
The correspondence between the formal objects and the running implementation remains subject to implementation review, motivating the additional empirical checks below.

\paragraph{Empirical audits.}
We perform two empirical checks.
First, a gradient audit tests that computation before the first DP-MHCA layer is row-wise and detects unintended dependencies between context rows.
Second, a sensitivity audit uses gradient-based optimisation to search for adjacent token sets whose pre-noise summaries are far apart.
For each optimised pair, we repeatedly sample privatised outputs and estimate the hypothesis-testing trade-off curve using a likelihood-ratio test with access to the two unnoised means and the calibrated Gaussian variance.
This gives the auditor the strongest test available for that pair.
Agreement with the intended trade-off curve provides evidence against errors in sensitivity control or noise calibration, but is not a substitute for the formal guarantee.
The full mechanism-audit plots across privacy levels are shown in \cref{fig:dp-mhca-audit}.

\begin{figure}[!ht]
    \centering
    \begin{subfigure}[t]{0.31\textwidth}
        \centering
        \includegraphics[width=\textwidth]{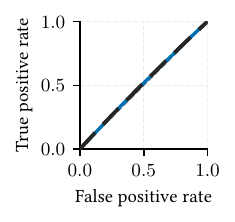}
        \caption{\(\mu=0.05\)}
    \end{subfigure}
    \hfill
    \begin{subfigure}[t]{0.31\textwidth}
        \centering
        \includegraphics[width=\textwidth]{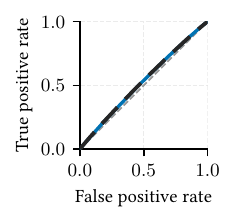}
        \caption{\(\mu=0.1\)}
    \end{subfigure}
    \hfill
    \begin{subfigure}[t]{0.31\textwidth}
        \centering
        \includegraphics[width=\textwidth]{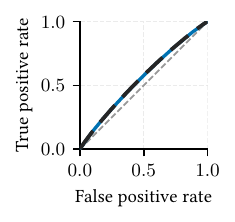}
        \caption{\(\mu=0.2\)}
    \end{subfigure}
    \par
    \vspace{0.3em}

    \begin{subfigure}[t]{0.31\textwidth}
        \centering
        \includegraphics[width=\textwidth]{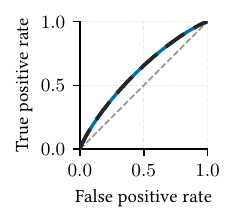}
        \caption{\(\mu=0.4\)}
    \end{subfigure}
    \hfill
    \begin{subfigure}[t]{0.31\textwidth}
        \centering
        \includegraphics[width=\textwidth]{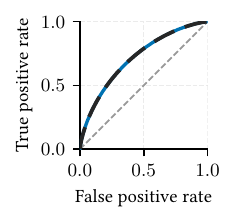}
        \caption{\(\mu=0.8\)}
    \end{subfigure}
    \hfill
    \begin{subfigure}[t]{0.31\textwidth}
        \centering
        \includegraphics[width=\textwidth]{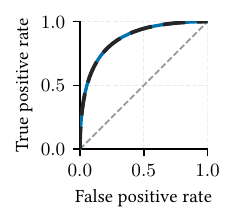}
        \caption{\(\mu=1.6\)}
    \end{subfigure}
    \par
    \vspace{0.1em}
    \includegraphics[width=0.90\textwidth]{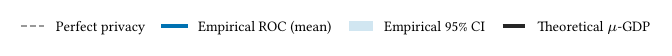}
    \caption{Empirical audit of the DP-MHCA mechanism across privacy levels.
For each \(\mu\), the audit optimises a pair of adjacent token sets to maximise the pre-noise output difference and then estimates the trade-off curve from repeated samples of the privatised outputs, using 10 repeated audit runs; shaded bands show pointwise bootstrap 95\% CIs over those runs.
The resulting curves are compared against the strongest likelihood-ratio attack with access to the unnoised means and the calibrated noise standard deviation.
The observed behaviour is consistent with the intended privacy calibration across the full range of \(\mu\).}
    \label{fig:dp-mhca-audit}
\end{figure}

\FloatBarrier
\subsection{TabArena Evaluation Details}\label{app:tabarena-evaluation-details}

\subsubsection{Preprocessing and Elo calibration}

The purpose of this benchmark is to compare the pure classification performance of the methods under a common preprocessing convention, so we treat $z$-score standardisation by population-level means and standard deviations as an idealised public preprocessing step and apply that convention to PrivTab and the baselines alike.
This is also a natural benchmark idealisation because the TabICL simulator standardises each generated dataset using that dataset's own mean and standard deviation before splitting it into context and target subsets.
In the actual evaluation on real datasets, the unknown population means and standard deviations are approximated by the empirical mean and standard deviation of the context set, and those context-set statistics are then used to standardise both context and target features.

The DP-LR and DP-MLP baselines are described in the Methods.
We rank PrivTab, DP-LR, DP-MLP and non-private logistic regression using the maximum-likelihood Elo procedure.
For interpretability, the fitted ratings are linearly recalibrated so that DP-LR scores 1,000 separately at each privacy level and in the aggregate comparison.
This transformation changes the origin of the Elo scale but not the ranking of the methods.
We use \(1 - \mathrm{AUC}\) as the error metric for binary datasets and log loss for multiclass datasets~\citep{EricksonTabArena2025}.
The dataset-level breakdowns evaluate discrimination on a common scale across all tasks.
For dataset \(d\), split \(r\), privacy level \(\mu\) and baseline \(b\), they compute \(\mathrm{AUC}_{\mathrm{PrivTab},d,r,\mu}-\mathrm{AUC}_{b,d,r,\mu}\), using macro-averaged one-versus-rest AUC when \(d\) is multiclass.
Marker colours average this difference over the ten splits within a dataset, whereas panel titles average it over all 330 paired splits.

\begin{figure}[!ht]
    \centering
    \includegraphics[width=\textwidth]{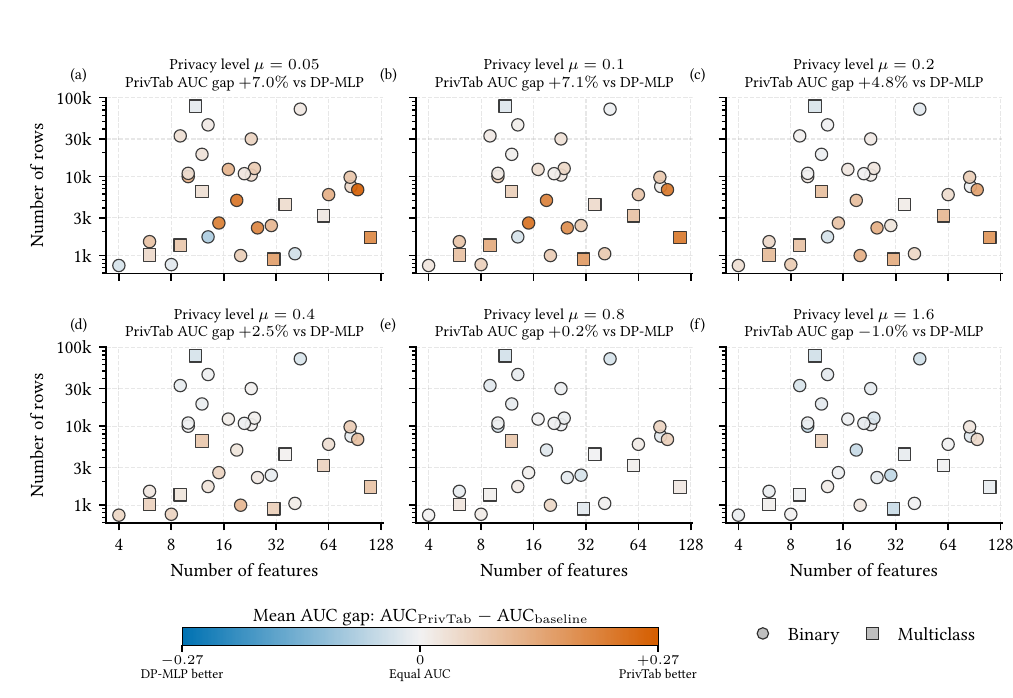}
    \caption{\textbf{Dataset-level AUC gap between PrivTab and DP-MLP across all privacy levels.} Each marker represents a dataset, positioned by its number of features and rows.
Marker colour gives the mean paired AUC gap over ten context--target splits; positive values favour PrivTab.
Binary AUC is used for binary datasets and macro-averaged one-versus-rest AUC for multiclass datasets.}
    \label{fig:supplement-auc-gap-dp-mlp}
\end{figure}

\begin{figure}[!ht]
    \centering
    \includegraphics[width=\textwidth]{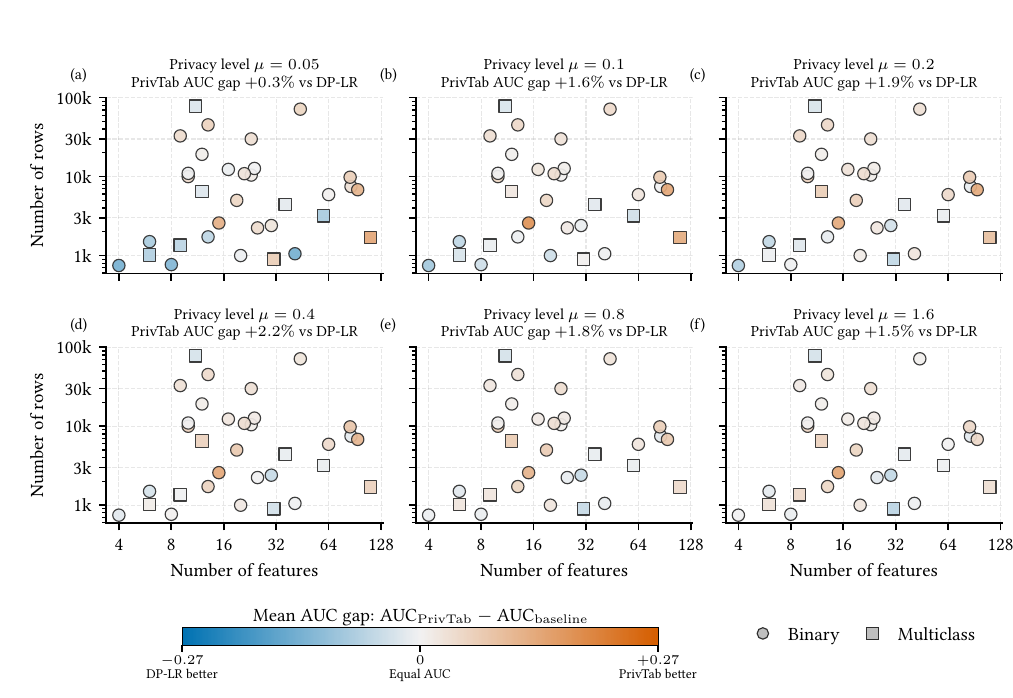}
    \caption{\textbf{Dataset-level AUC gap between PrivTab and DP-LR across all privacy levels.} Plotting conventions match \cref{fig:supplement-auc-gap-dp-mlp}; positive values favour PrivTab.}
    \label{fig:supplement-auc-gap-dp-lr}
\end{figure}

\begin{figure}[!ht]
    \centering
    \includegraphics[width=\textwidth]{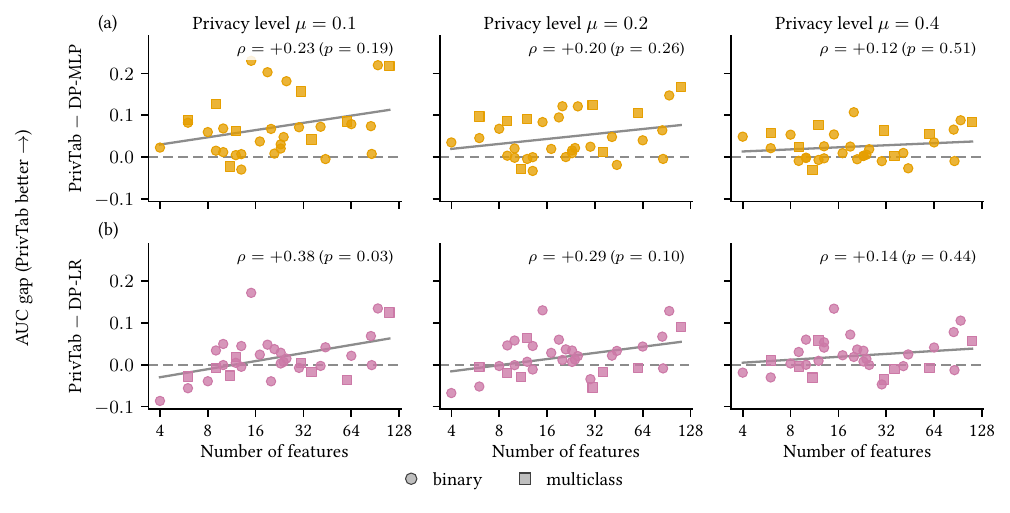}
    \caption{
    \textbf{AUC gaps versus feature count across 33 TabArena datasets.}
    \textbf{a}, PrivTab minus DP-MLP AUC; \textbf{b}, PrivTab minus DP-LR AUC.
    Columns show privacy levels \(\mu=0.1, 0.2\) and \(0.4\).
    Each point is a dataset's mean paired AUC gap across ten context--target splits, plotted against its number of features on a logarithmic scale; positive gaps favour PrivTab.
    Circles denote binary datasets and squares denote multiclass datasets; AUC is macro-averaged one-versus-rest for multiclass tasks.
    Dashed horizontal lines mark equal AUC, grey lines are least-squares fits against log feature count, and annotations report Spearman's \(\rho\) and its \(p\)-value.
    }
    \label{fig:supplement-auc-gap-features}
\end{figure}

\subsubsection{Derived comparison metrics}

To interpret the comparison plots, let \(m\) denote a method, let \(b\) denote the reference baseline, let \(d\) denote a dataset, let \(r\) denote a context--target split, and let \(\mu\) denote the privacy level.
Writing \(\mathrm{AUC}_{m,d,r,\mu}\) and \(\ell_{m,d,r,\mu}\) for the corresponding target-set AUC and log loss, the split-level AUC improvement over the baseline is
\begin{equation}
    \Delta_{\mathrm{AUC}}(m,b;d,r,\mu)
    :=
    \mathrm{AUC}_{m,d,r,\mu}
    -
    \mathrm{AUC}_{b,d,r,\mu}.
\end{equation}
For the non-private logistic-regression reference, we report the relative log-loss reduction
\begin{equation}
    \Delta_{\mathrm{LL}}^{\mathrm{nonDP}}(m;d,r,\mu)
    :=
    \frac{
        \ell_{\mathrm{non\text{-}private\,LR},d,r}
        -
        \ell_{m,d,r,\mu}
    }{
        \ell_{\mathrm{non\text{-}private\,LR},d,r}
    },
\end{equation}
so positive values indicate lower loss than the non-private reference.
Dividing by the reference loss makes the quantity comparable across datasets with different numbers of classes and hence different natural log-loss scales.
For comparisons against DP-LR, we first average over context--target splits within each dataset and then report the relative log-loss reduction
\begin{equation}
    \Delta_{\mathrm{LL}}^{\mathrm{rel}}(m,\mathrm{DP\text{-}LR};d,\mu)
    :=
    \frac{
        \bar{\ell}_{\mathrm{DP\text{-}LR},d,\mu}
        -
        \bar{\ell}_{m,d,\mu}
    }{
        \bar{\ell}_{\mathrm{DP\text{-}LR},d,\mu}
    },
\end{equation}
where \(\bar{\ell}_{m,d,\mu}\) denotes the average loss over context--target splits.
Finally, the win-rate plots summarise the indicator comparisons
\begin{equation}
    W_{\mathrm{AUC}}(m,b;d,r,\mu)
    :=
    \begin{cases}
        1, & \mathrm{AUC}_{m,d,r,\mu} > \mathrm{AUC}_{b,d,r,\mu}, \\
        0.5, & \mathrm{AUC}_{m,d,r,\mu} = \mathrm{AUC}_{b,d,r,\mu}, \\
        0, & \mathrm{AUC}_{m,d,r,\mu} < \mathrm{AUC}_{b,d,r,\mu},
    \end{cases}
\end{equation}
and analogously \(W_{\mathrm{LL}}(m,b;d,r,\mu)\), where smaller log loss is better.
Dataset-level win rates are obtained by averaging these quantities over context--target splits.

\FloatBarrier
\clearpage
\subsection{Aggregate Benchmark Comparisons}\label{app:aggregate-benchmark-tables}

Elo ratings, comparisons with non-private LR and dataset-level runtimes are reported in \cref{tab:supplement-elo-by-mu,tab:supplement-aggregate-elo,tab:supplement-auc-gap-non-dp,tab:supplement-logloss-gap-non-dp,tab:supplement-runtime-source}.

\scriptsize
\begin{longtable}{llrrr}
\caption{Numerical values underlying the ELO curves in Figure~\ref{fig:exp-elo}(c).}\label{tab:supplement-elo-by-mu}\\
\toprule
$\mu$ & Method & ELO & CI min & CI max \\
\midrule
\endfirsthead
\caption[]{Numerical values underlying the ELO curves in Figure~\ref{fig:exp-elo}(c). (continued)}\\
\toprule
$\mu$ & Method & ELO & CI min & CI max \\
\midrule
\endhead
\midrule
\multicolumn{5}{r}{Continued on next page} \\
\midrule
\endfoot
\bottomrule
\endlastfoot
\rowcolor{gray!8} 0.05 & PrivTab & 1142.8 & 1125.0 & 1161.3 \\
\rowcolor{gray!8} 0.05 & DP-LR & 1000.0 & 970.4 & 1027.8 \\
\rowcolor{gray!8} 0.05 & DP-MLP & 910.7 & 882.1 & 937.9 \\
0.1 & PrivTab & 1172.4 & 1155.4 & 1190.4 \\
0.1 & DP-LR & 1000.0 & 969.7 & 1028.7 \\
0.1 & DP-MLP & 966.7 & 939.5 & 993.0 \\
\rowcolor{gray!8} 0.2 & PrivTab & 1189.5 & 1171.7 & 1208.3 \\
\rowcolor{gray!8} 0.2 & DP-MLP & 1038.7 & 1011.6 & 1065.2 \\
\rowcolor{gray!8} 0.2 & DP-LR & 1000.0 & 968.9 & 1029.0 \\
0.4 & PrivTab & 1185.7 & 1168.3 & 1204.3 \\
0.4 & DP-MLP & 1131.1 & 1104.6 & 1158.7 \\
0.4 & DP-LR & 1000.0 & 968.7 & 1030.3 \\
\rowcolor{gray!8} 0.8 & DP-MLP & 1243.5 & 1216.3 & 1273.3 \\
\rowcolor{gray!8} 0.8 & PrivTab & 1174.6 & 1157.4 & 1193.0 \\
\rowcolor{gray!8} 0.8 & DP-LR & 1000.0 & 968.2 & 1029.8 \\
1.6 & DP-MLP & 1296.3 & 1265.8 & 1330.9 \\
1.6 & PrivTab & 1155.8 & 1136.8 & 1176.1 \\
1.6 & DP-LR & 1000.0 & 969.9 & 1029.4 \\
\end{longtable}
\normalsize

\begin{table}[t]
\centering
\small
\caption{Numerical values underlying the aggregate ELO coordinates in Figure~\ref{fig:exp-elo}(d), with ratings aggregated over all privacy budgets.}
\label{tab:supplement-aggregate-elo}
\begin{tabular}{lrrr}
\toprule
Method & ELO & CI min & CI max \\
\midrule
PrivTab & 1166.5 & 1159.3 & 1173.7 \\
DP-MLP & 1097.5 & 1086.1 & 1109.2 \\
DP-LR & 1000.0 & 988.1 & 1011.9 \\
\bottomrule
\end{tabular}
\normalsize
\end{table}

\scriptsize
\begin{longtable}{llrrrrr}
\caption{Numerical values underlying Figure~\ref{fig:exp-nondp-gap}(a).}\label{tab:supplement-auc-gap-non-dp}\\
\toprule
Method & $\mu$ & Mean & CI low & CI high & $n$ datasets & $n$ splits \\
\midrule
\endfirsthead
\caption[]{Numerical values underlying Figure~\ref{fig:exp-nondp-gap}(a). (continued)}\\
\toprule
Method & $\mu$ & Mean & CI low & CI high & $n$ datasets & $n$ splits \\
\midrule
\endhead
\midrule
\multicolumn{7}{r}{Continued on next page} \\
\midrule
\endfoot
\bottomrule
\endlastfoot
\rowcolor{gray!8} DP-LR & 0.05 & -0.1249 & -0.1579 & -0.0950 & 33 & 10 \\
DP-LR & 0.1 & -0.0933 & -0.1210 & -0.0669 & 33 & 10 \\
\rowcolor{gray!8} DP-LR & 0.2 & -0.0663 & -0.0880 & -0.0460 & 33 & 10 \\
DP-LR & 0.4 & -0.0471 & -0.0643 & -0.0314 & 33 & 10 \\
\rowcolor{gray!8} DP-LR & 0.8 & -0.0335 & -0.0471 & -0.0211 & 33 & 10 \\
DP-LR & 1.6 & -0.0238 & -0.0359 & -0.0131 & 33 & 10 \\
\rowcolor{gray!8} DP-MLP & 0.05 & -0.1924 & -0.2359 & -0.1519 & 33 & 10 \\
DP-MLP & 0.1 & -0.1482 & -0.1883 & -0.1092 & 33 & 10 \\
\rowcolor{gray!8} DP-MLP & 0.2 & -0.0948 & -0.1255 & -0.0656 & 33 & 10 \\
DP-MLP & 0.4 & -0.0505 & -0.0709 & -0.0320 & 33 & 10 \\
\rowcolor{gray!8} DP-MLP & 0.8 & -0.0169 & -0.0312 & -0.0030 & 33 & 10 \\
DP-MLP & 1.6 & 0.0007 & -0.0120 & 0.0137 & 33 & 10 \\
\rowcolor{gray!8} PrivTab & 0.05 & -0.1222 & -0.1586 & -0.0876 & 33 & 10 \\
PrivTab & 0.1 & -0.0773 & -0.1025 & -0.0532 & 33 & 10 \\
\rowcolor{gray!8} PrivTab & 0.2 & -0.0472 & -0.0665 & -0.0297 & 33 & 10 \\
PrivTab & 0.4 & -0.0255 & -0.0386 & -0.0132 & 33 & 10 \\
\rowcolor{gray!8} PrivTab & 0.8 & -0.0150 & -0.0273 & -0.0036 & 33 & 10 \\
PrivTab & 1.6 & -0.0088 & -0.0201 & 0.0014 & 33 & 10 \\
\end{longtable}
\normalsize

\scriptsize
\begin{longtable}{llrrrrr}
\caption{Numerical values underlying the relative log-loss improvements in Figure~\ref{fig:exp-nondp-gap}(b).}\label{tab:supplement-logloss-gap-non-dp}\\
\toprule
Method & $\mu$ & Mean & CI low & CI high & $n$ datasets & $n$ splits \\
\midrule
\endfirsthead
\caption[]{Numerical values underlying the relative log-loss improvements in Figure~\ref{fig:exp-nondp-gap}(b). (continued)}\\
\toprule
Method & $\mu$ & Mean & CI low & CI high & $n$ datasets & $n$ splits \\
\midrule
\endhead
\midrule
\multicolumn{7}{r}{Continued on next page} \\
\midrule
\endfoot
\bottomrule
\endlastfoot
\rowcolor{gray!8} DP-LR & 0.05 & -12.0599 & -20.0871 & -6.3234 & 33 & 10 \\
DP-LR & 0.1 & -6.4823 & -10.8479 & -3.1667 & 33 & 10 \\
\rowcolor{gray!8} DP-LR & 0.2 & -3.1523 & -4.8574 & -1.8595 & 33 & 10 \\
DP-LR & 0.4 & -1.5109 & -1.8378 & -1.2354 & 33 & 10 \\
\rowcolor{gray!8} DP-LR & 0.8 & -1.1520 & -1.2902 & -1.0277 & 33 & 10 \\
DP-LR & 1.6 & -1.0440 & -1.1426 & -0.9516 & 33 & 10 \\
\rowcolor{gray!8} DP-MLP & 0.05 & -1.5079 & -2.5324 & -0.8257 & 33 & 10 \\
DP-MLP & 0.1 & -1.0130 & -1.6921 & -0.5475 & 33 & 10 \\
\rowcolor{gray!8} DP-MLP & 0.2 & -0.7080 & -1.3381 & -0.3270 & 33 & 10 \\
DP-MLP & 0.4 & -0.3822 & -0.6844 & -0.1853 & 33 & 10 \\
\rowcolor{gray!8} DP-MLP & 0.8 & -0.1543 & -0.3202 & -0.0466 & 33 & 10 \\
DP-MLP & 1.6 & -0.0560 & -0.1257 & -0.0007 & 33 & 10 \\
\rowcolor{gray!8} PrivTab & 0.05 & -0.5217 & -0.8835 & -0.2618 & 33 & 10 \\
PrivTab & 0.1 & -0.4293 & -0.7594 & -0.1951 & 33 & 10 \\
\rowcolor{gray!8} PrivTab & 0.2 & -0.3435 & -0.6366 & -0.1307 & 33 & 10 \\
PrivTab & 0.4 & -0.2787 & -0.5365 & -0.0878 & 33 & 10 \\
\rowcolor{gray!8} PrivTab & 0.8 & -0.2218 & -0.4303 & -0.0657 & 33 & 10 \\
PrivTab & 1.6 & -0.1926 & -0.3746 & -0.0561 & 33 & 10 \\
\end{longtable}
\normalsize

\scriptsize
\begin{longtable}{lrrccc}
\caption{Dataset sizes and per-context--target-split runtimes in seconds underlying Figure~\ref{fig:exp-runtime}.}\label{tab:supplement-runtime-source}\\
\toprule
Dataset & Rows & Features & PrivTab & DP-MLP & DP-LR \\
\midrule
\endfirsthead
\caption[]{Dataset sizes and per-context--target-split runtimes in seconds underlying Figure~\ref{fig:exp-runtime}. (continued)}\\
\toprule
Dataset & Rows & Features & PrivTab & DP-MLP & DP-LR \\
\midrule
\endhead
\midrule
\multicolumn{6}{r}{Continued on next page} \\
\midrule
\endfoot
\bottomrule
\endlastfoot
Amazon employee access & 32769 & 9 & 0.037 & 524.350 & 1540.025 \\
Anneal & 898 & 31 & 0.012 & 92.250 & 284.975 \\
Bank marketing & 45211 & 13 & 0.021 & 724.675 & 2105.300 \\
Bank customer churn & 10000 & 10 & 0.013 & 202.975 & 616.075 \\
Blood transfusion & 748 & 4 & 0.013 & 91.700 & 299.800 \\
Churn & 5000 & 19 & 0.013 & 145.150 & 452.825 \\
COIL2000 & 9822 & 85 & 0.013 & 202.075 & 624.475 \\
Credit-g & 1000 & 20 & 0.013 & 90.425 & 300.750 \\
Credit card default & 30000 & 23 & 0.015 & 498.500 & 1401.675 \\
Diabetes & 768 & 8 & 0.013 & 92.350 & 296.650 \\
Diabetes130US & 71518 & 44 & 0.030 & 1067.925 & 3078.550 \\
E-commerce shipping & 10999 & 10 & 0.013 & 208.600 & 634.600 \\
Fitness club & 1500 & 6 & 0.013 & 92.125 & 311.000 \\
Hazelnut contaminant & 2400 & 30 & 0.013 & 93.250 & 299.800 \\
HELOC & 10459 & 23 & 0.013 & 203.550 & 629.150 \\
HR analytics job change & 19158 & 12 & 0.014 & 321.525 & 948.675 \\
In-vehicle coupon & 12684 & 24 & 0.014 & 241.775 & 753.800 \\
Good customer & 1723 & 13 & 0.013 & 93.250 & 297.550 \\
Marketing campaign & 2240 & 25 & 0.013 & 93.250 & 301.350 \\
Maternal health risk & 1014 & 6 & 0.012 & 93.250 & 300.450 \\
NATICUSdroid & 7491 & 86 & 0.013 & 172.000 & 533.475 \\
Online shoppers & 12330 & 17 & 0.014 & 238.725 & 713.125 \\
Polish bankruptcy & 5910 & 64 & 0.013 & 150.850 & 454.375 \\
QSAR biodeg & 1054 & 41 & 0.013 & 92.475 & 307.175 \\
SDSS17 & 78053 & 11 & 0.029 & 1151.625 & 3282.500 \\
Seismic bumps & 2584 & 15 & 0.013 & 91.850 & 296.300 \\
Splice & 3190 & 60 & 0.012 & 100.950 & 316.675 \\
Student dropout & 4424 & 36 & 0.012 & 138.450 & 447.400 \\
Taiwanese bankruptcy & 6819 & 94 & 0.013 & 156.200 & 485.800 \\
Website phishing & 1353 & 9 & 0.012 & 93.975 & 300.450 \\
Wine quality & 6497 & 12 & 0.012 & 156.200 & 480.025 \\
MIC & 1699 & 111 & 0.012 & 91.650 & 303.275 \\
JM1 & 10885 & 21 & 0.013 & 204.850 & 629.275 \\
\end{longtable}
\normalsize

\begin{figure}[!ht]
    \centering
    \includegraphics[width=0.84\textwidth]{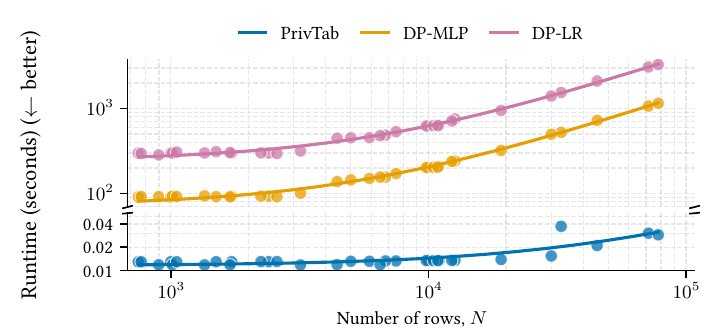}
    \caption{\textbf{Fitting time as a function of dataset size.} Points are dataset-level measurements and curves are fitted shifted power laws.
PrivTab is approximately four orders of magnitude faster than the DP Adam baselines over the studied range.}
    \label{fig:exp-runtime}
\end{figure}

\FloatBarrier
\clearpage
\paragraph{Prediction time on a laptop.}
We additionally measured post-fit prediction time on a laptop with a 13th-generation Intel Core i5 CPU across the same 33 datasets, six privacy levels and ten context--target splits.
Each configuration used one warm-up call followed by three timed runs, with a target batch size of 1,024; we report the mean of the per-configuration median times.
Timing covers prediction for all target rows, and excludes preprocessing, model loading, fitting and private-summary construction.
These measurements assess prediction cost rather than predictive performance.
The dataset-level results are reported in \cref{tab:supplement-inference-timing,fig:inference-timing}.

\scriptsize
\begin{longtable}{lrrccc}
\caption{Dataset sizes and post-fit prediction times in milliseconds on a laptop with a 13th-generation Intel Core i5 CPU, underlying Figure~\ref{fig:inference-timing}. Values average the per-configuration median of three timed runs over six privacy levels and ten context--target splits.}\label{tab:supplement-inference-timing}\\
\toprule
Dataset & Target rows & Features & PrivTab & DP-MLP & DP-LR \\
\midrule
\endfirsthead
\caption[]{Dataset sizes and post-fit prediction times in milliseconds on a laptop with a 13th-generation Intel Core i5 CPU, underlying Figure~\ref{fig:inference-timing}. Values average the per-configuration median of three timed runs over six privacy levels and ten context--target splits. (continued)}\\
\toprule
Dataset & Target rows & Features & PrivTab & DP-MLP & DP-LR \\
\midrule
\endhead
\midrule
\multicolumn{6}{r}{Continued on next page} \\
\midrule
\endfoot
\bottomrule
\endlastfoot
Amazon employee access & 6554 & 9 & 355 & 1.73 & 0.727 \\
Anneal & 180 & 31 & 15.2 & 0.159 & 0.114 \\
Bank marketing & 9042 & 13 & 582 & 2.83 & 0.928 \\
Bank customer churn & 2000 & 10 & 105 & 0.585 & 0.2 \\
Blood transfusion & 150 & 4 & 13 & 0.124 & 0.0849 \\
Churn & 1000 & 19 & 64.4 & 0.346 & 0.115 \\
COIL2000 & 1964 & 85 & 106 & 0.549 & 0.192 \\
Credit-g & 200 & 20 & 15.1 & 0.0715 & 0.0381 \\
Credit card default & 6000 & 23 & 394 & 1.7 & 0.664 \\
Diabetes & 154 & 8 & 11.9 & 0.107 & 0.0687 \\
Diabetes130US & 14304 & 44 & 754 & 3.62 & 1.12 \\
E-commerce shipping & 2200 & 10 & 112 & 0.588 & 0.259 \\
Fitness club & 300 & 6 & 19.1 & 0.164 & 0.116 \\
Hazelnut contaminant & 480 & 30 & 29.8 & 0.232 & 0.182 \\
HELOC & 2092 & 23 & 111 & 0.605 & 0.245 \\
HR analytics job change & 3832 & 12 & 197 & 1.06 & 0.361 \\
In-vehicle coupon & 2537 & 24 & 147 & 0.772 & 0.355 \\
Good customer & 345 & 13 & 24.1 & 0.174 & 0.123 \\
JM1 & 2177 & 21 & 127 & 1.67 & 1.06 \\
Marketing campaign & 448 & 25 & 34.5 & 0.212 & 0.152 \\
Maternal health risk & 203 & 6 & 18.9 & 0.13 & 0.062 \\
MIC & 340 & 111 & 35.4 & 0.243 & 0.0644 \\
NATICUSdroid & 1498 & 86 & 96.4 & 0.653 & 0.233 \\
Online shoppers & 2466 & 17 & 142 & 0.661 & 0.313 \\
Polish bankruptcy & 1182 & 64 & 66.6 & 0.354 & 0.17 \\
QSAR biodeg & 211 & 41 & 18.8 & 0.114 & 0.0754 \\
SDSS17 & 15611 & 11 & 805 & 3.46 & 1.46 \\
Seismic bumps & 517 & 15 & 28.3 & 0.173 & 0.116 \\
Splice & 638 & 60 & 32.7 & 0.202 & 0.0834 \\
Student dropout & 885 & 36 & 46.9 & 0.196 & 0.207 \\
Taiwanese bankruptcy & 1364 & 94 & 81.8 & 0.448 & 0.528 \\
Website phishing & 271 & 9 & 20.9 & 0.244 & 0.133 \\
Wine quality & 1299 & 12 & 76.5 & 1.09 & 0.18 \\
\end{longtable}
\normalsize

\begin{figure}[!ht]
    \centering
    \captionsetup{skip=3pt}
    \includegraphics[width=0.84\textwidth]{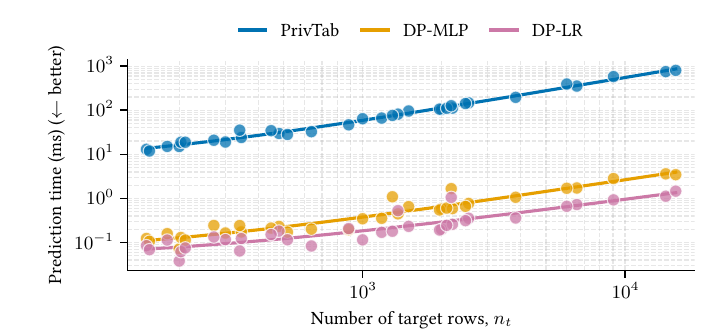}
    \caption{\textbf{Prediction time as a function of target-set size on a laptop CPU.}
    Points show dataset-level mean post-fit prediction times in milliseconds (ms) on a 13th-generation Intel Core i5, averaged over six privacy levels and ten context--target splits.
    Curves are shifted power laws fitted in log time as visual guides.
    Both axes are logarithmic. All methods predict in batches of up to 1,024 target rows.}
    \label{fig:inference-timing}
\end{figure}
\FloatBarrier

The following figure summarises the direct comparisons with DP-LR; the corresponding numerical values are reported in \cref{tab:supplement-auc-improvement-logreg,tab:supplement-auc-winrate-logreg,tab:supplement-logloss-improvement-logreg,tab:supplement-loss-winrate-logreg}.

\begin{figure}[!ht]
    \centering
    \begin{subfigure}[t]{0.48\textwidth}
        \centering
        \includegraphics[width=\textwidth]{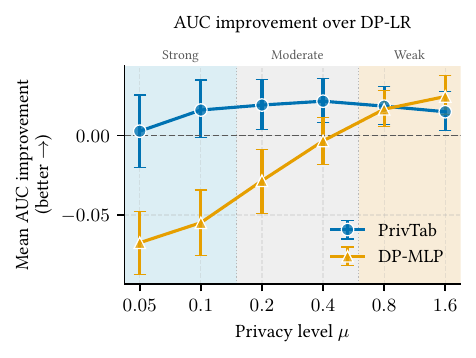}
        \caption{Mean AUC improvement.}
    \end{subfigure}
    \hfill
    \begin{subfigure}[t]{0.48\textwidth}
        \centering
        \includegraphics[width=\textwidth]{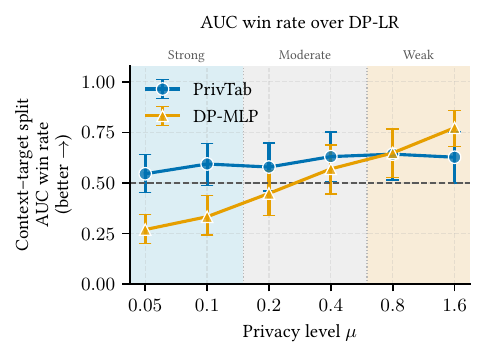}
        \caption{AUC win rate.}
    \end{subfigure}
    \vspace{0.6em}
    \begin{subfigure}[t]{0.48\textwidth}
        \centering
        \includegraphics[width=\textwidth]{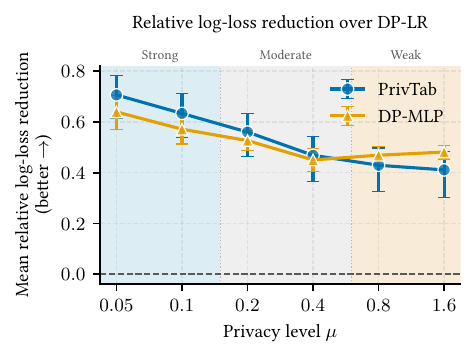}
        \caption{Relative log-loss reduction.}
    \end{subfigure}
    \hfill
    \begin{subfigure}[t]{0.48\textwidth}
        \centering
        \includegraphics[width=\textwidth]{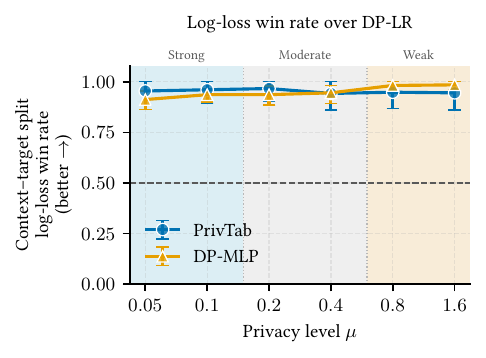}
        \caption{Log-loss win rate.}
    \end{subfigure}
    \caption{Performance relative to DP-LR.
Results aggregate ten context--target splits per dataset; error bars show bootstrap 95\% CIs over datasets.
Positive improvements and win rates above $0.5$ favour the compared method.
Bands are visual guides, not formal privacy thresholds: strong (\(\mu<0.15\)), moderate (\(0.15\leq\mu<0.6\)) and weak (\(\mu\geq0.6\)) formal privacy.}
    \label{fig:exp-dp-logreg}
\end{figure}

\scriptsize
\begin{longtable}{llrrrl}
\caption{Numerical values underlying Figure~\ref{fig:exp-dp-logreg}(a) (AUC improvement over DP-LR).}\label{tab:supplement-auc-improvement-logreg}\\
\toprule
Method & $\mu$ & Mean & CI low & CI high & $n$ datasets \\
\midrule
\endfirsthead
\caption[]{Numerical values underlying Figure~\ref{fig:exp-dp-logreg}(a) (AUC improvement over DP-LR). (continued)}\\
\toprule
Method & $\mu$ & Mean & CI low & CI high & $n$ datasets \\
\midrule
\endhead
\midrule
\multicolumn{6}{r}{Continued on next page} \\
\midrule
\endfoot
\bottomrule
\endlastfoot
\rowcolor{gray!8} PrivTab & 0.05 & 0.0027 & -0.0198 & 0.0257 & 33 \\
PrivTab & 0.1 & 0.0160 & -0.0011 & 0.0352 & 33 \\
\rowcolor{gray!8} PrivTab & 0.2 & 0.0192 & 0.0040 & 0.0355 & 33 \\
PrivTab & 0.4 & 0.0216 & 0.0086 & 0.0355 & 33 \\
\rowcolor{gray!8} PrivTab & 0.8 & 0.0185 & 0.0068 & 0.0307 & 33 \\
PrivTab & 1.6 & 0.0150 & 0.0035 & 0.0280 & 33 \\
\rowcolor{gray!8} DP-MLP & 0.05 & -0.0675 & -0.0871 & -0.0476 & 33 \\
DP-MLP & 0.1 & -0.0548 & -0.0758 & -0.0342 & 33 \\
\rowcolor{gray!8} DP-MLP & 0.2 & -0.0284 & -0.0486 & -0.0087 & 33 \\
DP-MLP & 0.4 & -0.0034 & -0.0184 & 0.0108 & 33 \\
\rowcolor{gray!8} DP-MLP & 0.8 & 0.0166 & 0.0060 & 0.0283 & 33 \\
DP-MLP & 1.6 & 0.0245 & 0.0133 & 0.0376 & 33 \\
\end{longtable}
\normalsize

\scriptsize
\begin{longtable}{llrrrl}
\caption{Numerical values underlying Figure~\ref{fig:exp-dp-logreg}(b) (AUC win rate over DP-LR).}\label{tab:supplement-auc-winrate-logreg}\\
\toprule
Method & $\mu$ & Mean & CI low & CI high & $n$ datasets \\
\midrule
\endfirsthead
\caption[]{Numerical values underlying Figure~\ref{fig:exp-dp-logreg}(b) (AUC win rate over DP-LR). (continued)}\\
\toprule
Method & $\mu$ & Mean & CI low & CI high & $n$ datasets \\
\midrule
\endhead
\midrule
\multicolumn{6}{r}{Continued on next page} \\
\midrule
\endfoot
\bottomrule
\endlastfoot
\rowcolor{gray!8} PrivTab & 0.05 & 0.5455 & 0.4485 & 0.6424 & 33 \\
PrivTab & 0.1 & 0.5939 & 0.4909 & 0.6970 & 33 \\
\rowcolor{gray!8} PrivTab & 0.2 & 0.5788 & 0.4576 & 0.6939 & 33 \\
PrivTab & 0.4 & 0.6303 & 0.5061 & 0.7455 & 33 \\
\rowcolor{gray!8} PrivTab & 0.8 & 0.6424 & 0.5152 & 0.7606 & 33 \\
PrivTab & 1.6 & 0.6273 & 0.5000 & 0.7545 & 33 \\
\rowcolor{gray!8} DP-MLP & 0.05 & 0.2697 & 0.2000 & 0.3424 & 33 \\
DP-MLP & 0.1 & 0.3333 & 0.2394 & 0.4303 & 33 \\
\rowcolor{gray!8} DP-MLP & 0.2 & 0.4485 & 0.3394 & 0.5636 & 33 \\
DP-MLP & 0.4 & 0.5697 & 0.4455 & 0.6879 & 33 \\
\rowcolor{gray!8} DP-MLP & 0.8 & 0.6485 & 0.5242 & 0.7636 & 33 \\
DP-MLP & 1.6 & 0.7727 & 0.6818 & 0.8606 & 33 \\
\end{longtable}
\normalsize

\scriptsize
\begin{longtable}{llrrrl}
\caption{Numerical values underlying Figure~\ref{fig:exp-dp-logreg}(c) (Relative log-loss reduction over DP-LR).}\label{tab:supplement-logloss-improvement-logreg}\\
\toprule
Method & $\mu$ & Mean & CI low & CI high & $n$ datasets \\
\midrule
\endfirsthead
\caption[]{Numerical values underlying Figure~\ref{fig:exp-dp-logreg}(c) (Relative log-loss reduction over DP-LR). (continued)}\\
\toprule
Method & $\mu$ & Mean & CI low & CI high & $n$ datasets \\
\midrule
\endhead
\midrule
\multicolumn{6}{r}{Continued on next page} \\
\midrule
\endfoot
\bottomrule
\endlastfoot
\rowcolor{gray!8} PrivTab & 0.05 & 0.7061 & 0.6100 & 0.7827 & 33 \\
PrivTab & 0.1 & 0.6333 & 0.5370 & 0.7092 & 33 \\
\rowcolor{gray!8} PrivTab & 0.2 & 0.5598 & 0.4586 & 0.6342 & 33 \\
PrivTab & 0.4 & 0.4684 & 0.3679 & 0.5398 & 33 \\
\rowcolor{gray!8} PrivTab & 0.8 & 0.4292 & 0.3305 & 0.4971 & 33 \\
PrivTab & 1.6 & 0.4109 & 0.3036 & 0.4811 & 33 \\
\rowcolor{gray!8} DP-MLP & 0.05 & 0.6392 & 0.5725 & 0.7031 & 33 \\
DP-MLP & 0.1 & 0.5712 & 0.5105 & 0.6295 & 33 \\
\rowcolor{gray!8} DP-MLP & 0.2 & 0.5263 & 0.4869 & 0.5669 & 33 \\
DP-MLP & 0.4 & 0.4502 & 0.4050 & 0.4943 & 33 \\
\rowcolor{gray!8} DP-MLP & 0.8 & 0.4688 & 0.4317 & 0.5027 & 33 \\
DP-MLP & 1.6 & 0.4812 & 0.4547 & 0.5067 & 33 \\
\end{longtable}
\normalsize

\scriptsize
\begin{longtable}{llrrrl}
\caption{Numerical values underlying Figure~\ref{fig:exp-dp-logreg}(d) (Log-loss win rate over DP-LR).}\label{tab:supplement-loss-winrate-logreg}\\
\toprule
Method & $\mu$ & Mean & CI low & CI high & $n$ datasets \\
\midrule
\endfirsthead
\caption[]{Numerical values underlying Figure~\ref{fig:exp-dp-logreg}(d) (Log-loss win rate over DP-LR). (continued)}\\
\toprule
Method & $\mu$ & Mean & CI low & CI high & $n$ datasets \\
\midrule
\endhead
\midrule
\multicolumn{6}{r}{Continued on next page} \\
\midrule
\endfoot
\bottomrule
\endlastfoot
\rowcolor{gray!8} PrivTab & 0.05 & 0.9545 & 0.8788 & 1.0000 & 33 \\
PrivTab & 0.1 & 0.9606 & 0.8909 & 1.0000 & 33 \\
\rowcolor{gray!8} PrivTab & 0.2 & 0.9667 & 0.9030 & 1.0000 & 33 \\
PrivTab & 0.4 & 0.9424 & 0.8606 & 1.0000 & 33 \\
\rowcolor{gray!8} PrivTab & 0.8 & 0.9485 & 0.8667 & 1.0000 & 33 \\
PrivTab & 1.6 & 0.9455 & 0.8606 & 1.0000 & 33 \\
\rowcolor{gray!8} DP-MLP & 0.05 & 0.9121 & 0.8636 & 0.9545 & 33 \\
DP-MLP & 0.1 & 0.9364 & 0.9000 & 0.9667 & 33 \\
\rowcolor{gray!8} DP-MLP & 0.2 & 0.9364 & 0.8848 & 0.9788 & 33 \\
DP-MLP & 0.4 & 0.9455 & 0.8939 & 0.9818 & 33 \\
\rowcolor{gray!8} DP-MLP & 0.8 & 0.9818 & 0.9515 & 1.0000 & 33 \\
DP-MLP & 1.6 & 0.9848 & 0.9636 & 1.0000 & 33 \\
\end{longtable}
\normalsize

\FloatBarrier
\subsection{Per-dataset Benchmark Results}\label{app:per-dataset-results}

Per-dataset AUC and log loss are reported for PrivTab in \cref{tab:supplement-privtab-auc,tab:supplement-privtab-loss}, DP-LR in \cref{tab:supplement-dp-logistic-regression-auc,tab:supplement-dp-logistic-regression-loss}, DP-MLP in \cref{tab:supplement-dp-mlp-auc,tab:supplement-dp-mlp-loss} and non-private LR in \cref{tab:supplement-non-dp-logistic-regression-auc,tab:supplement-non-dp-logistic-regression-loss}.

\scriptsize
\begin{longtable}{lcccccc}
\caption{Per-dataset AUC for PrivTab; entries report the mean and bootstrap 95\% CI over context--target splits.}\label{tab:supplement-privtab-auc}\\
\toprule
Dataset & $\mu=0.05$ & $\mu=0.1$ & $\mu=0.2$ & $\mu=0.4$ & $\mu=0.8$ & $\mu=1.6$ \\
\midrule
\endfirsthead
\caption[]{Per-dataset AUC for PrivTab; entries report the mean and bootstrap 95\% CI over context--target splits. (continued)}\\
\toprule
Dataset & $\mu=0.05$ & $\mu=0.1$ & $\mu=0.2$ & $\mu=0.4$ & $\mu=0.8$ & $\mu=1.6$ \\
\midrule
\endhead
\midrule
\multicolumn{7}{r}{Continued on next page} \\
\midrule
\endfoot
\bottomrule
\endlastfoot
Amazon employee access & \shortstack{0.562\\{}[0.553,0.571]} & \shortstack{0.560\\{}[0.546,0.573]} & \shortstack{0.570\\{}[0.561,0.579]} & \shortstack{0.569\\{}[0.561,0.578]} & \shortstack{0.569\\{}[0.561,0.577]} & \shortstack{0.570\\{}[0.563,0.576]} \\
Anneal & \shortstack{0.638\\{}[0.596,0.673]} & \shortstack{0.685\\{}[0.660,0.708]} & \shortstack{0.735\\{}[0.721,0.749]} & \shortstack{0.827\\{}[0.812,0.844]} & \shortstack{0.852\\{}[0.827,0.875]} & \shortstack{0.882\\{}[0.861,0.906]} \\
Bank customer churn & \shortstack{0.753\\{}[0.739,0.766]} & \shortstack{0.776\\{}[0.769,0.783]} & \shortstack{0.785\\{}[0.779,0.792]} & \shortstack{0.790\\{}[0.783,0.797]} & \shortstack{0.793\\{}[0.786,0.800]} & \shortstack{0.794\\{}[0.786,0.801]} \\
Bank marketing & \shortstack{0.715\\{}[0.710,0.720]} & \shortstack{0.726\\{}[0.723,0.728]} & \shortstack{0.729\\{}[0.725,0.732]} & \shortstack{0.730\\{}[0.726,0.734]} & \shortstack{0.730\\{}[0.726,0.734]} & \shortstack{0.730\\{}[0.726,0.733]} \\
Blood transfusion & \shortstack{0.532\\{}[0.479,0.588]} & \shortstack{0.616\\{}[0.577,0.652]} & \shortstack{0.656\\{}[0.625,0.685]} & \shortstack{0.709\\{}[0.686,0.733]} & \shortstack{0.725\\{}[0.704,0.750]} & \shortstack{0.734\\{}[0.712,0.759]} \\
Churn & \shortstack{0.748\\{}[0.722,0.770]} & \shortstack{0.797\\{}[0.786,0.808]} & \shortstack{0.820\\{}[0.815,0.825]} & \shortstack{0.839\\{}[0.833,0.845]} & \shortstack{0.842\\{}[0.835,0.850]} & \shortstack{0.839\\{}[0.831,0.847]} \\
COIL2000 & \shortstack{0.603\\{}[0.587,0.620]} & \shortstack{0.625\\{}[0.613,0.636]} & \shortstack{0.650\\{}[0.627,0.669]} & \shortstack{0.692\\{}[0.677,0.706]} & \shortstack{0.700\\{}[0.689,0.711]} & \shortstack{0.701\\{}[0.690,0.712]} \\
Credit card default & \shortstack{0.743\\{}[0.737,0.748]} & \shortstack{0.742\\{}[0.736,0.748]} & \shortstack{0.750\\{}[0.745,0.755]} & \shortstack{0.750\\{}[0.746,0.754]} & \shortstack{0.752\\{}[0.747,0.756]} & \shortstack{0.751\\{}[0.747,0.756]} \\
\end{longtable}
\normalsize

\scriptsize
\begin{longtable}{lcccccc}
\caption{Per-dataset log loss for PrivTab; entries report the mean and bootstrap 95\% CI over context--target splits.}\label{tab:supplement-privtab-loss}\\
\toprule
Dataset & $\mu=0.05$ & $\mu=0.1$ & $\mu=0.2$ & $\mu=0.4$ & $\mu=0.8$ & $\mu=1.6$ \\
\midrule
\endfirsthead
\caption[]{Per-dataset log loss for PrivTab; entries report the mean and bootstrap 95\% CI over context--target splits. (continued)}\\
\toprule
Dataset & $\mu=0.05$ & $\mu=0.1$ & $\mu=0.2$ & $\mu=0.4$ & $\mu=0.8$ & $\mu=1.6$ \\
\midrule
\endhead
\midrule
\multicolumn{7}{r}{Continued on next page} \\
\midrule
\endfoot
\bottomrule
\endlastfoot
Amazon employee access & \shortstack{0.230\\{}[0.228,0.232]} & \shortstack{0.232\\{}[0.231,0.234]} & \shortstack{0.232\\{}[0.231,0.232]} & \shortstack{0.232\\{}[0.232,0.232]} & \shortstack{0.232\\{}[0.231,0.232]} & \shortstack{0.232\\{}[0.231,0.232]} \\
Anneal & \shortstack{0.968\\{}[0.926,1.015]} & \shortstack{0.890\\{}[0.863,0.919]} & \shortstack{0.791\\{}[0.765,0.819]} & \shortstack{0.702\\{}[0.689,0.715]} & \shortstack{0.538\\{}[0.518,0.557]} & \shortstack{0.432\\{}[0.418,0.446]} \\
Bank customer churn & \shortstack{0.450\\{}[0.444,0.457]} & \shortstack{0.433\\{}[0.428,0.438]} & \shortstack{0.425\\{}[0.422,0.429]} & \shortstack{0.420\\{}[0.416,0.423]} & \shortstack{0.420\\{}[0.416,0.423]} & \shortstack{0.420\\{}[0.417,0.424]} \\
Bank marketing & \shortstack{0.345\\{}[0.344,0.346]} & \shortstack{0.341\\{}[0.340,0.342]} & \shortstack{0.339\\{}[0.338,0.340]} & \shortstack{0.338\\{}[0.337,0.339]} & \shortstack{0.337\\{}[0.337,0.338]} & \shortstack{0.337\\{}[0.337,0.338]} \\
Blood transfusion & \shortstack{0.562\\{}[0.553,0.570]} & \shortstack{0.541\\{}[0.533,0.550]} & \shortstack{0.532\\{}[0.523,0.541]} & \shortstack{0.506\\{}[0.498,0.514]} & \shortstack{0.497\\{}[0.487,0.507]} & \shortstack{0.488\\{}[0.476,0.499]} \\
Churn & \shortstack{0.391\\{}[0.386,0.397]} & \shortstack{0.354\\{}[0.346,0.361]} & \shortstack{0.329\\{}[0.325,0.333]} & \shortstack{0.318\\{}[0.314,0.323]} & \shortstack{0.319\\{}[0.315,0.322]} & \shortstack{0.321\\{}[0.317,0.324]} \\
COIL2000 & \shortstack{0.238\\{}[0.232,0.245]} & \shortstack{0.228\\{}[0.225,0.230]} & \shortstack{0.232\\{}[0.229,0.236]} & \shortstack{0.230\\{}[0.228,0.231]} & \shortstack{0.231\\{}[0.230,0.232]} & \shortstack{0.231\\{}[0.230,0.233]} \\
Credit card default & \shortstack{0.457\\{}[0.454,0.461]} & \shortstack{0.453\\{}[0.449,0.457]} & \shortstack{0.451\\{}[0.448,0.455]} & \shortstack{0.450\\{}[0.447,0.454]} & \shortstack{0.450\\{}[0.447,0.454]} & \shortstack{0.450\\{}[0.447,0.453]} \\
\end{longtable}
\normalsize

\scriptsize
\begin{longtable}{lcccccc}
\caption{Per-dataset AUC for DP-LR; entries report the mean and bootstrap 95\% CI over context--target splits.}\label{tab:supplement-dp-logistic-regression-auc}\\
\toprule
Dataset & $\mu=0.05$ & $\mu=0.1$ & $\mu=0.2$ & $\mu=0.4$ & $\mu=0.8$ & $\mu=1.6$ \\
\midrule
\endfirsthead
\caption[]{Per-dataset AUC for DP-LR; entries report the mean and bootstrap 95\% CI over context--target splits. (continued)}\\
\toprule
Dataset & $\mu=0.05$ & $\mu=0.1$ & $\mu=0.2$ & $\mu=0.4$ & $\mu=0.8$ & $\mu=1.6$ \\
\midrule
\endhead
\midrule
\multicolumn{7}{r}{Continued on next page} \\
\midrule
\endfoot
\bottomrule
\endlastfoot
Amazon employee access & \shortstack{0.523\\{}[0.508,0.538]} & \shortstack{0.526\\{}[0.509,0.542]} & \shortstack{0.524\\{}[0.502,0.545]} & \shortstack{0.539\\{}[0.524,0.552]} & \shortstack{0.548\\{}[0.531,0.562]} & \shortstack{0.553\\{}[0.532,0.567]} \\
Anneal & \shortstack{0.573\\{}[0.526,0.618]} & \shortstack{0.681\\{}[0.637,0.729]} & \shortstack{0.790\\{}[0.751,0.831]} & \shortstack{0.863\\{}[0.819,0.904]} & \shortstack{0.899\\{}[0.862,0.936]} & \shortstack{0.943\\{}[0.919,0.963]} \\
Bank customer churn & \shortstack{0.709\\{}[0.681,0.735]} & \shortstack{0.726\\{}[0.702,0.746]} & \shortstack{0.727\\{}[0.701,0.746]} & \shortstack{0.730\\{}[0.711,0.747]} & \shortstack{0.728\\{}[0.710,0.745]} & \shortstack{0.739\\{}[0.730,0.748]} \\
Bank marketing & \shortstack{0.659\\{}[0.625,0.687]} & \shortstack{0.681\\{}[0.664,0.696]} & \shortstack{0.684\\{}[0.663,0.701]} & \shortstack{0.689\\{}[0.674,0.703]} & \shortstack{0.701\\{}[0.693,0.708]} & \shortstack{0.706\\{}[0.697,0.714]} \\
Blood transfusion & \shortstack{0.667\\{}[0.631,0.698]} & \shortstack{0.703\\{}[0.673,0.727]} & \shortstack{0.724\\{}[0.700,0.749]} & \shortstack{0.728\\{}[0.708,0.750]} & \shortstack{0.731\\{}[0.709,0.754]} & \shortstack{0.735\\{}[0.712,0.758]} \\
Churn & \shortstack{0.698\\{}[0.661,0.734]} & \shortstack{0.749\\{}[0.724,0.772]} & \shortstack{0.760\\{}[0.730,0.787]} & \shortstack{0.767\\{}[0.744,0.788]} & \shortstack{0.776\\{}[0.756,0.796]} & \shortstack{0.787\\{}[0.772,0.803]} \\
COIL2000 & \shortstack{0.540\\{}[0.516,0.566]} & \shortstack{0.556\\{}[0.533,0.580]} & \shortstack{0.583\\{}[0.559,0.606]} & \shortstack{0.614\\{}[0.592,0.634]} & \shortstack{0.637\\{}[0.618,0.656]} & \shortstack{0.656\\{}[0.641,0.670]} \\
Credit card default & \shortstack{0.709\\{}[0.706,0.712]} & \shortstack{0.713\\{}[0.710,0.718]} & \shortstack{0.716\\{}[0.712,0.720]} & \shortstack{0.716\\{}[0.712,0.721]} & \shortstack{0.716\\{}[0.711,0.721]} & \shortstack{0.716\\{}[0.711,0.721]} \\
\end{longtable}
\normalsize

\scriptsize
\begin{longtable}{lcccccc}
\caption{Per-dataset log loss for DP-LR; entries report the mean and bootstrap 95\% CI over context--target splits.}\label{tab:supplement-dp-logistic-regression-loss}\\
\toprule
Dataset & $\mu=0.05$ & $\mu=0.1$ & $\mu=0.2$ & $\mu=0.4$ & $\mu=0.8$ & $\mu=1.6$ \\
\midrule
\endfirsthead
\caption[]{Per-dataset log loss for DP-LR; entries report the mean and bootstrap 95\% CI over context--target splits. (continued)}\\
\toprule
Dataset & $\mu=0.05$ & $\mu=0.1$ & $\mu=0.2$ & $\mu=0.4$ & $\mu=0.8$ & $\mu=1.6$ \\
\midrule
\endhead
\midrule
\multicolumn{7}{r}{Continued on next page} \\
\midrule
\endfoot
\bottomrule
\endlastfoot
Amazon employee access & \shortstack{0.510\\{}[0.407,0.649]} & \shortstack{0.462\\{}[0.404,0.534]} & \shortstack{0.419\\{}[0.388,0.459]} & \shortstack{0.396\\{}[0.382,0.413]} & \shortstack{0.386\\{}[0.377,0.396]} & \shortstack{0.378\\{}[0.374,0.383]} \\
Anneal & \shortstack{17.709\\{}[7.089,30.296]} & \shortstack{7.697\\{}[3.262,13.514]} & \shortstack{4.202\\{}[1.611,7.595]} & \shortstack{0.867\\{}[0.564,1.284]} & \shortstack{0.576\\{}[0.434,0.803]} & \shortstack{0.425\\{}[0.325,0.568]} \\
Bank customer churn & \shortstack{1.089\\{}[0.900,1.395]} & \shortstack{0.914\\{}[0.873,0.961]} & \shortstack{0.921\\{}[0.871,0.988]} & \shortstack{0.921\\{}[0.878,0.972]} & \shortstack{0.950\\{}[0.892,1.011]} & \shortstack{0.931\\{}[0.901,0.967]} \\
Bank marketing & \shortstack{0.721\\{}[0.662,0.796]} & \shortstack{0.686\\{}[0.654,0.723]} & \shortstack{0.673\\{}[0.653,0.695]} & \shortstack{0.673\\{}[0.662,0.684]} & \shortstack{0.658\\{}[0.650,0.667]} & \shortstack{0.645\\{}[0.639,0.651]} \\
Blood transfusion & \shortstack{1.859\\{}[0.925,2.898]} & \shortstack{1.438\\{}[0.898,2.109]} & \shortstack{1.188\\{}[0.910,1.530]} & \shortstack{0.989\\{}[0.889,1.110]} & \shortstack{0.894\\{}[0.822,0.980]} & \shortstack{0.884\\{}[0.820,0.959]} \\
Churn & \shortstack{1.496\\{}[0.715,2.575]} & \shortstack{1.053\\{}[0.683,1.709]} & \shortstack{0.908\\{}[0.695,1.266]} & \shortstack{0.741\\{}[0.703,0.779]} & \shortstack{0.725\\{}[0.672,0.793]} & \shortstack{0.694\\{}[0.668,0.722]} \\
COIL2000 & \shortstack{1.565\\{}[0.611,3.008]} & \shortstack{0.700\\{}[0.540,0.972]} & \shortstack{0.589\\{}[0.511,0.713]} & \shortstack{0.502\\{}[0.483,0.523]} & \shortstack{0.475\\{}[0.461,0.489]} & \shortstack{0.464\\{}[0.449,0.483]} \\
Credit card default & \shortstack{0.986\\{}[0.958,1.019]} & \shortstack{0.955\\{}[0.938,0.975]} & \shortstack{0.959\\{}[0.929,0.996]} & \shortstack{0.951\\{}[0.925,0.989]} & \shortstack{0.961\\{}[0.936,0.986]} & \shortstack{0.945\\{}[0.933,0.958]} \\
\end{longtable}
\normalsize

\scriptsize
\begin{longtable}{lcccccc}
\caption{Per-dataset AUC for DP-MLP; entries report the mean and bootstrap 95\% CI over context--target splits.}\label{tab:supplement-dp-mlp-auc}\\
\toprule
Dataset & $\mu=0.05$ & $\mu=0.1$ & $\mu=0.2$ & $\mu=0.4$ & $\mu=0.8$ & $\mu=1.6$ \\
\midrule
\endfirsthead
\caption[]{Per-dataset AUC for DP-MLP; entries report the mean and bootstrap 95\% CI over context--target splits. (continued)}\\
\toprule
Dataset & $\mu=0.05$ & $\mu=0.1$ & $\mu=0.2$ & $\mu=0.4$ & $\mu=0.8$ & $\mu=1.6$ \\
\midrule
\endhead
\midrule
\multicolumn{7}{r}{Continued on next page} \\
\midrule
\endfoot
\bottomrule
\endlastfoot
Amazon employee access & \shortstack{0.528\\{}[0.511,0.543]} & \shortstack{0.545\\{}[0.532,0.557]} & \shortstack{0.567\\{}[0.557,0.577]} & \shortstack{0.579\\{}[0.571,0.587]} & \shortstack{0.587\\{}[0.576,0.597]} & \shortstack{0.599\\{}[0.587,0.610]} \\
Anneal & \shortstack{0.489\\{}[0.431,0.542]} & \shortstack{0.528\\{}[0.467,0.580]} & \shortstack{0.610\\{}[0.539,0.675]} & \shortstack{0.764\\{}[0.714,0.805]} & \shortstack{0.870\\{}[0.836,0.895]} & \shortstack{0.928\\{}[0.909,0.944]} \\
Bank customer churn & \shortstack{0.633\\{}[0.595,0.671]} & \shortstack{0.707\\{}[0.685,0.728]} & \shortstack{0.765\\{}[0.752,0.779]} & \shortstack{0.792\\{}[0.777,0.809]} & \shortstack{0.826\\{}[0.813,0.838]} & \shortstack{0.847\\{}[0.839,0.854]} \\
Bank marketing & \shortstack{0.698\\{}[0.688,0.707]} & \shortstack{0.719\\{}[0.713,0.724]} & \shortstack{0.729\\{}[0.724,0.734]} & \shortstack{0.733\\{}[0.728,0.738]} & \shortstack{0.740\\{}[0.735,0.744]} & \shortstack{0.745\\{}[0.741,0.748]} \\
Blood transfusion & \shortstack{0.563\\{}[0.487,0.630]} & \shortstack{0.594\\{}[0.518,0.659]} & \shortstack{0.621\\{}[0.550,0.688]} & \shortstack{0.661\\{}[0.606,0.712]} & \shortstack{0.723\\{}[0.696,0.752]} & \shortstack{0.741\\{}[0.722,0.761]} \\
Churn & \shortstack{0.523\\{}[0.467,0.566]} & \shortstack{0.594\\{}[0.537,0.647]} & \shortstack{0.726\\{}[0.680,0.767]} & \shortstack{0.814\\{}[0.793,0.836]} & \shortstack{0.862\\{}[0.846,0.878]} & \shortstack{0.887\\{}[0.877,0.896]} \\
COIL2000 & \shortstack{0.524\\{}[0.501,0.548]} & \shortstack{0.551\\{}[0.527,0.574]} & \shortstack{0.587\\{}[0.563,0.610]} & \shortstack{0.626\\{}[0.604,0.647]} & \shortstack{0.646\\{}[0.628,0.664]} & \shortstack{0.679\\{}[0.667,0.690]} \\
Credit card default & \shortstack{0.689\\{}[0.680,0.698]} & \shortstack{0.712\\{}[0.702,0.722]} & \shortstack{0.734\\{}[0.724,0.743]} & \shortstack{0.748\\{}[0.743,0.753]} & \shortstack{0.758\\{}[0.753,0.762]} & \shortstack{0.766\\{}[0.762,0.770]} \\
\end{longtable}
\normalsize

\scriptsize
\begin{longtable}{lcccccc}
\caption{Per-dataset log loss for DP-MLP; entries report the mean and bootstrap 95\% CI over context--target splits.}\label{tab:supplement-dp-mlp-loss}\\
\toprule
Dataset & $\mu=0.05$ & $\mu=0.1$ & $\mu=0.2$ & $\mu=0.4$ & $\mu=0.8$ & $\mu=1.6$ \\
\midrule
\endfirsthead
\caption[]{Per-dataset log loss for DP-MLP; entries report the mean and bootstrap 95\% CI over context--target splits. (continued)}\\
\toprule
Dataset & $\mu=0.05$ & $\mu=0.1$ & $\mu=0.2$ & $\mu=0.4$ & $\mu=0.8$ & $\mu=1.6$ \\
\midrule
\endhead
\midrule
\multicolumn{7}{r}{Continued on next page} \\
\midrule
\endfoot
\bottomrule
\endlastfoot
Amazon employee access & \shortstack{0.283\\{}[0.229,0.358]} & \shortstack{0.250\\{}[0.223,0.298]} & \shortstack{0.224\\{}[0.221,0.227]} & \shortstack{0.221\\{}[0.220,0.223]} & \shortstack{0.220\\{}[0.219,0.221]} & \shortstack{0.219\\{}[0.218,0.220]} \\
Anneal & \shortstack{2.720\\{}[1.402,4.559]} & \shortstack{1.983\\{}[1.203,3.363]} & \shortstack{1.709\\{}[1.020,2.640]} & \shortstack{0.826\\{}[0.584,1.141]} & \shortstack{0.513\\{}[0.402,0.668]} & \shortstack{0.309\\{}[0.273,0.345]} \\
Bank customer churn & \shortstack{0.510\\{}[0.487,0.536]} & \shortstack{0.479\\{}[0.464,0.496]} & \shortstack{0.435\\{}[0.424,0.445]} & \shortstack{0.410\\{}[0.396,0.422]} & \shortstack{0.384\\{}[0.372,0.394]} & \shortstack{0.377\\{}[0.367,0.388]} \\
Bank marketing & \shortstack{0.339\\{}[0.336,0.343]} & \shortstack{0.330\\{}[0.328,0.334]} & \shortstack{0.325\\{}[0.322,0.329]} & \shortstack{0.322\\{}[0.319,0.326]} & \shortstack{0.317\\{}[0.315,0.319]} & \shortstack{0.316\\{}[0.315,0.318]} \\
Blood transfusion & \shortstack{0.764\\{}[0.653,0.909]} & \shortstack{0.708\\{}[0.607,0.831]} & \shortstack{0.610\\{}[0.565,0.651]} & \shortstack{0.582\\{}[0.540,0.619]} & \shortstack{0.532\\{}[0.504,0.562]} & \shortstack{0.502\\{}[0.483,0.521]} \\
Churn & \shortstack{0.653\\{}[0.452,1.015]} & \shortstack{0.547\\{}[0.407,0.790]} & \shortstack{0.371\\{}[0.354,0.392]} & \shortstack{0.324\\{}[0.310,0.337]} & \shortstack{0.288\\{}[0.270,0.305]} & \shortstack{0.271\\{}[0.256,0.288]} \\
COIL2000 & \shortstack{0.472\\{}[0.279,0.813]} & \shortstack{0.398\\{}[0.246,0.677]} & \shortstack{0.250\\{}[0.239,0.266]} & \shortstack{0.242\\{}[0.234,0.254]} & \shortstack{0.232\\{}[0.229,0.236]} & \shortstack{0.230\\{}[0.227,0.233]} \\
Credit card default & \shortstack{0.488\\{}[0.480,0.498]} & \shortstack{0.473\\{}[0.466,0.480]} & \shortstack{0.466\\{}[0.458,0.473]} & \shortstack{0.457\\{}[0.453,0.461]} & \shortstack{0.455\\{}[0.451,0.459]} & \shortstack{0.457\\{}[0.453,0.461]} \\
\end{longtable}
\normalsize

\scriptsize
\begin{longtable}{lc}
\caption{Per-dataset AUC for Non-private LR; entries report the mean and bootstrap 95\% CI over context--target splits.}\label{tab:supplement-non-dp-logistic-regression-auc}\\
\toprule
Dataset & Non-DP \\
\midrule
\endfirsthead
\caption[]{Per-dataset AUC for Non-private LR; entries report the mean and bootstrap 95\% CI over context--target splits. (continued)}\\
\toprule
Dataset & Non-DP \\
\midrule
\endhead
\midrule
\multicolumn{2}{r}{Continued on next page} \\
\midrule
\endfoot
\bottomrule
\endlastfoot
Amazon employee access & \shortstack{0.582\\{}[0.574,0.590]} \\
Anneal & \shortstack{0.985\\{}[0.979,0.990]} \\
Bank customer churn & \shortstack{0.767\\{}[0.758,0.776]} \\
Bank marketing & \shortstack{0.732\\{}[0.729,0.736]} \\
Blood transfusion & \shortstack{0.728\\{}[0.710,0.745]} \\
Churn & \shortstack{0.780\\{}[0.771,0.788]} \\
COIL2000 & \shortstack{0.738\\{}[0.730,0.748]} \\
Credit card default & \shortstack{0.742\\{}[0.738,0.746]} \\
\end{longtable}
\normalsize

\scriptsize
\begin{longtable}{lc}
\caption{Per-dataset log loss for Non-private LR; entries report the mean and bootstrap 95\% CI over context--target splits.}\label{tab:supplement-non-dp-logistic-regression-loss}\\
\toprule
Dataset & Non-DP \\
\midrule
\endfirsthead
\caption[]{Per-dataset log loss for Non-private LR; entries report the mean and bootstrap 95\% CI over context--target splits. (continued)}\\
\toprule
Dataset & Non-DP \\
\midrule
\endhead
\midrule
\multicolumn{2}{r}{Continued on next page} \\
\midrule
\endfoot
\bottomrule
\endlastfoot
Amazon employee access & \shortstack{0.219\\{}[0.218,0.219]} \\
Anneal & \shortstack{0.170\\{}[0.150,0.194]} \\
Bank customer churn & \shortstack{0.426\\{}[0.421,0.432]} \\
Bank marketing & \shortstack{0.316\\{}[0.314,0.317]} \\
Blood transfusion & \shortstack{0.494\\{}[0.482,0.508]} \\
Churn & \shortstack{0.341\\{}[0.337,0.345]} \\
COIL2000 & \shortstack{0.206\\{}[0.204,0.208]} \\
Credit card default & \shortstack{0.455\\{}[0.452,0.457]} \\
\end{longtable}
\normalsize

\FloatBarrier
\subsection{Private Preprocessing Pipeline Details}
\label{app:end-to-end-details}

This subsection supports the private-preprocessing evaluations in \cref{fig:end-to-end-case-auc,fig:supplement-preprocessing-logloss,fig:end-to-end-workflows} with the private moment mechanism, details of LLM-bound inference, numerical results, Elo ratings and per-feature clipping bounds.

\begin{figure}[!ht]
    \centering
    \includegraphics[width=0.82\textwidth]{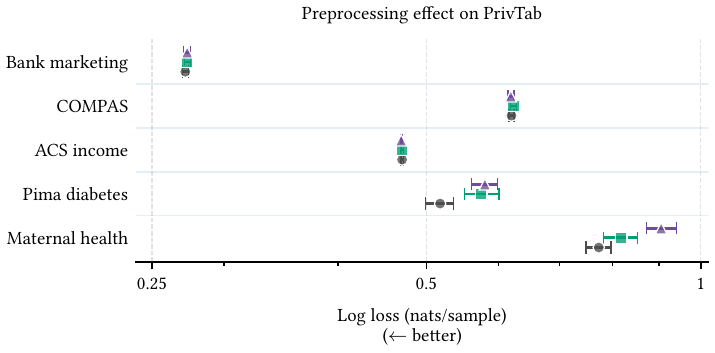}
    \par\vspace{0.1em}
    \includegraphics[width=\textwidth]{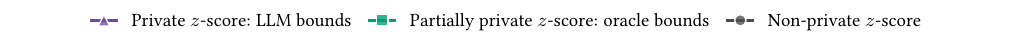}
    \caption{
    \textbf{Effect of preprocessing on PrivTab.}
    Mean log loss across ten context--target splits on five datasets, comparing non-private standardisation, partially private standardisation with oracle bounds, and private standardisation with LLM-derived bounds.
    Lower log loss is better. Error bars show bootstrap 95\% confidence intervals. The legend beneath the plot identifies the preprocessing variants.
    }
    \label{fig:supplement-preprocessing-logloss}
\end{figure}

\begin{figure}[!ht]
    \centering
    \begin{minipage}[t]{0.74\textwidth}
        \includegraphics[width=\linewidth]{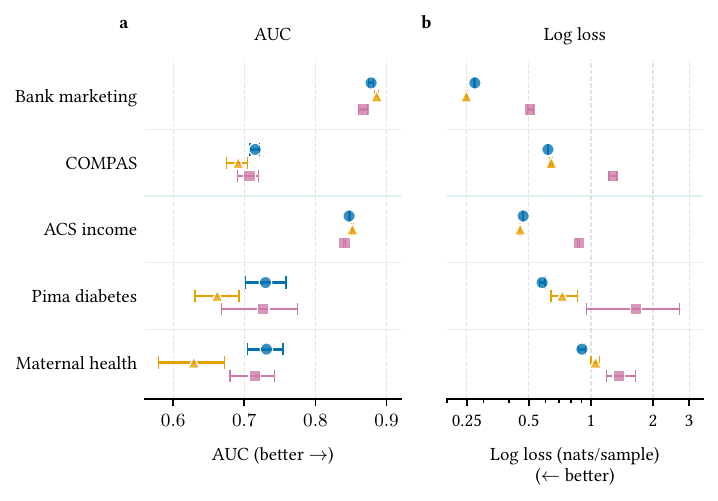}
    \end{minipage}\hfill
    \begin{minipage}[t]{0.23\textwidth}
        \includegraphics[width=\linewidth]{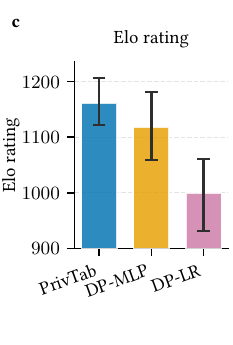}
    \end{minipage}
    \par\vspace{0.1em}
    \includegraphics[width=\textwidth]{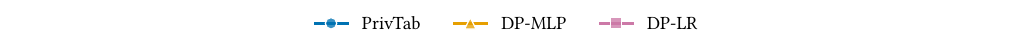}
    \caption{
    \textbf{Performance of privately preprocessed models at \(\mu=0.4\).}
    \textbf{a,b}, AUC and log loss for PrivTab, DP-MLP and DP-LR using LLM-derived bounds and private standardisation across five datasets.
    \textbf{c}, Elo ratings of these models; DP-LR is calibrated to 1,000.
    Higher AUC and Elo and lower log loss are better. Error bars show bootstrap 95\% confidence intervals. The method legend is beneath all three panels.
    }
    \label{fig:end-to-end-workflows}
\end{figure}

\FloatBarrier

\subsubsection{Private moment mechanism}

Let $n_c$ be the context size and $[a_k,b_k]$ the fixed clipping interval for input coordinate $k=1,\ldots,d$.
Coordinates with $b_k-a_k<10^{-6}$ are set to zero.
Otherwise, we clip and rescale each context input as
\begin{equation}
  c_{ik}=\min\{1,\max\{0,(x_{ik}-a_k)/(b_k-a_k)\}\},
\end{equation}
and apply the same transformation to target inputs.
We add independent Gaussian noise to the context sums and sums of squares:
\begin{equation}
    \tilde m_{1k}=\sum_{i=1}^{n_c}c_{ik}+\xi_{1,k},
    \qquad
    \tilde m_{2k}=\sum_{i=1}^{n_c}c_{ik}^{2}+\xi_{2,k},
    \qquad
    \xi_{1,k}\sim\mathcal N(0,\sigma_1^2),\quad
    \xi_{2,k}\sim\mathcal N(0,\sigma_2^2).
\end{equation}
Replacing one context row changes each sum vector by at most $\sqrt d$ in $\ell_2$ norm.
We allocate $\mu_1=\mu_{\mathrm{prep}}/2$ to the sums and $\mu_2=\sqrt{\mu_{\mathrm{prep}}^2-\mu_1^2}$ to the sums of squares, with $\sigma_1=\sqrt d/\mu_1$ and $\sigma_2=\sqrt d/\mu_2$.
By GDP composition, the two releases satisfy $\mu_{\mathrm{prep}}$-GDP because $\sqrt{\mu_1^2+\mu_2^2}=\mu_{\mathrm{prep}}$.

We compute the noisy mean $\tilde m_k=\tilde m_{1k}/n_c$ and variance $\tilde v_k=\tilde m_{2k}/n_c-\tilde m_k^2$.
If $\tilde v_k<10^{-6}$, including when it is negative, or the clipping interval is constant, we set that coordinate to zero for both context and target rows.
Otherwise, we obtain the standardised input $z_{ik}$ by clipping $(c_{ik}-\tilde m_k)/\sqrt{\tilde v_k}$ to $[-10,10]$.
The LLM-derived bounds use public descriptions; the oracle-bound workflow is partially private because its empirical dataset minima and maxima are selected without privacy accounting.

\subsubsection{LLM Bound-Inference Details}

For the end-to-end experiments described in the main-text section on private preprocessing, public clipping bounds for numeric features are inferred with Gemini 2.5 Flash.
The Gemini call uses temperature $0$, top-$k=1$, and a maximum output budget of 8192 tokens.
The response is constrained by an explicit JSON schema, and the code issues a plain \texttt{generate\_content} call without enabling any search or retrieval tool.
Accordingly, the model is prompted to rely only on the public dataset context, the public column descriptions, the provided units, and general background knowledge, rather than on any online lookup of the dataset.

The prompt is structured so that the model explicitly separates \emph{documented public bounds} from \emph{suggested clipping bounds}.
For each numeric feature, the response must contain documented lower and upper bounds, a documented-bound status, an evidence field, suggested lower and upper clipping bounds, a suggested-clip status, a guess-rationale field, and a single confidence label.
The documented bounds are allowed to be \texttt{null} when the column description does not justify an exact or semantically implied numeric limit.
For example, the prompt may allow the model to infer that a quantity is non-negative, or that an age variable refers to adults, while still leaving one side of the documented range unspecified.
The evidence field records why a documented bound is justified from the prompt alone.
The confidence field is not attached only to the documented part; it is a single feature-level label for the overall returned assessment.

The suggested clipping bounds are the values actually used by preprocessing, and they must always be finite and form a valid interval.
When exact public bounds are available, the suggested clipping bounds can simply inherit them.
Otherwise, the suggested clipping interval is chosen conservatively from public semantics, units, and general knowledge about the feature meaning.
This is why the response distinguishes ``what is justified directly by the public description'' from ``what is a conservative suggested clipping choice for DP preprocessing.'' The distinction matters because many features admit only partial public evidence, so one side of the documented range may remain \texttt{null} even though a finite suggested clipping interval is still required operationally.
In the implementation, the single confidence label is then conservatively post-processed: if the documented-bound status is \texttt{semantic\_public\_bound} or \texttt{insufficient\_information}, a returned \texttt{high} confidence is downgraded to \texttt{medium}; likewise, if the suggested-clip status is \texttt{chosen\_clipping\_bound}, a returned \texttt{high} confidence is also downgraded to \texttt{medium}.
In addition, when the documented-bound status is \texttt{insufficient\_information}, the validator forces both documented bounds to \texttt{null}.

Only columns that genuinely require semantic reasoning are sent to Gemini.
Categorical features and columns with direct public bounds, such as publicly documented ordinal codes, bypass the LLM entirely and are handled deterministically in code.
For those columns, the implementation sets both documented and suggested bounds directly from the public definition, marks both statuses as \texttt{exact\_public\_bound}, and assigns \texttt{high} confidence.
The LLM is therefore queried only for non-categorical columns that do not already have direct public bounds.
Prompts are cached before inference, and the resulting bounds are also cached and reused unless the cache is explicitly refreshed.
The implementation additionally validates that Gemini returns exactly the requested feature names and, if a batch request fails, falls back to retrying one column at a time.

\paragraph{Example Prompt.}
\begingroup
\fontencoding{T1}\selectfont
\begin{Verbatim}[fontsize=\scriptsize,breaklines=true,breakanywhere=true,frame=single]
SYSTEM:
You are extracting PUBLIC clipping bounds for differentially private preprocessing.

Return strict JSON only.
Do not include markdown.
Do not include explanations outside JSON.

USER:
Rules:
1. Analyze each target column independently.
2. Do NOT claim exact bounds unless they are explicit or follow from definitions.
3. Do NOT use empirical min/max values unless explicitly provided in the public description.
4. If true documented bounds are unavailable, you MUST still provide finite suggested clipping bounds.
5. The suggested clipping bounds must be conservative and based on public semantics, units, and dataset context.
6. Clearly distinguish documented bounds from guessed clipping bounds.
7. If a public semantic lower bound is clear, include it even when the upper bound is not documented.
8. All suggested clipping bounds must be finite numeric values.
9. Do not omit a target column. If uncertain, use documented_bound_status="insufficient_information"
   and suggested_clip_status="chosen_clipping_bound".
10. The top-level JSON value must be an array, not an object wrapper.
11. Use the dataset-specific guidance in the Dataset context when it defines units,
    special values, or semantic lower/upper bounds.

Return strict JSON only as an array:
[
  {
    "feature": string,
    "unit": string | null,
    "documented_lower_bound": number | null,
    "documented_upper_bound": number | null,
    "documented_bound_status": "exact_public_bound" | "documented_special_value"
                               | "semantic_public_bound" | "insufficient_information",
    "suggested_lower_clip": number,
    "suggested_upper_clip": number,
    "suggested_clip_status": "exact_public_bound" | "semantic_public_bound"
                             | "chosen_clipping_bound",
    "evidence": string,
    "guess_rationale": string,
    "confidence": "low" | "medium" | "high"
  }
]

Dataset context:
Pima Indians Diabetes dataset from OpenML/UCI. The task is binary prediction of diabetes onset
from diagnostic measurements. Numeric guidance: pregnancies, glucose, blood pressure, skin
thickness, insulin, BMI, diabetes pedigree function, and age are non-negative; age is adult
age in years. Use only public medical/column semantics for clipping bounds.
\end{Verbatim}
\endgroup

\paragraph{Example Response.}
\begingroup
\fontencoding{T1}\selectfont
\begin{Verbatim}[fontsize=\scriptsize,breaklines=true,breakanywhere=true,frame=single]
{
  "dataset": "Pima Indians Diabetes",
  "provider": "gemini",
  "model": "gemini-2.5-flash",
  "results": {
    "preg": {
      "documented_lower_bound": 0,
      "documented_upper_bound": null,
      "documented_bound_status": "semantic_public_bound",
      "suggested_lower_clip": 0,
      "suggested_upper_clip": 17,
      "suggested_clip_status": "chosen_clipping_bound",
      "evidence": "Dataset context states that pregnancies are non-negative.",
      "guess_rationale": "No explicit upper bound is given, so 17 is chosen as a conservative
                          medically plausible clipping value.",
      "confidence": "medium"
    },
    "plas": {
      "documented_lower_bound": 0,
      "documented_upper_bound": null,
      "documented_bound_status": "semantic_public_bound",
      "suggested_lower_clip": 40,
      "suggested_upper_clip": 400,
      "suggested_clip_status": "chosen_clipping_bound",
      "evidence": "The prompt states that glucose is non-negative.",
      "guess_rationale": "The final interval is chosen from general medical knowledge about
                          physiologically plausible glucose values.",
      "confidence": "medium"
    },
    "age": {
      "documented_lower_bound": 18,
      "documented_upper_bound": null,
      "documented_bound_status": "semantic_public_bound",
      "suggested_lower_clip": 18,
      "suggested_upper_clip": 100,
      "suggested_clip_status": "chosen_clipping_bound",
      "evidence": "The prompt states that age is adult age in years.",
      "guess_rationale": "The lower bound is inherited from the semantic evidence; the upper
                          bound is a conservative clipping choice based on general knowledge
                          about adult human ages.",
      "confidence": "medium"
    }
  }
}
\end{Verbatim}
\endgroup

\paragraph{Reading the Bound Tables.}
The supplement tables of public and LLM-derived bounds use the following column meanings.
\texttt{Feature} is the preprocessing feature name.
\texttt{Type} distinguishes columns that require LLM-based semantic inference from those that do not: \texttt{numeric} means that the feature was sent to Gemini for bound inference, whereas \texttt{categorical} means that the feature was handled deterministically from public category codes or another direct public definition.
\texttt{Exact min} and \texttt{Exact max} are the empirical minimum and maximum values observed in the dataset and are reported only for comparison; they are not treated as public bounds by the end-to-end pipeline.
\texttt{Source} records where the operational clipping interval comes from: \texttt{llm} means that the final suggested clipping bounds came from the Gemini response, while \texttt{categorical} means that the LLM was bypassed and the bounds were fixed directly from public information in code.
\texttt{Lower} and \texttt{Upper} are the final clipping bounds actually used by preprocessing.

The \texttt{Confidence} column is a feature-level label for the final bound assessment.
For rows with \texttt{Source} equal to \texttt{llm}, the values \texttt{low}, \texttt{medium}, and \texttt{high} are taken from the validated Gemini response after the conservative post-processing implemented in \texttt{end\_to\_end\_cases.py}.
Concretely, a response does not retain \texttt{high} confidence when the documented part is only semantic rather than exact, when the documented status is \texttt{insufficient\_\allowbreak information}, or when the suggested interval is marked as \texttt{chosen\_\allowbreak clipping\_\allowbreak bound} rather than directly justified by the public description.
The label \texttt{exact} appears only for rows with \texttt{Source} equal to \texttt{categorical}; it indicates that the bounds were not inferred by the LLM and instead came directly from an exact public feature definition handled deterministically in code.

Method-level results, Elo ratings and preprocessing comparisons are reported in \cref{tab:supplement-end-to-end-summary,tab:end-to-end-elo,tab:supplement-end-to-end-case-auc}.
Per-feature clipping bounds are listed in \cref{tab:supplement-acs-income-bounds,tab:supplement-bank-marketing-bounds,tab:supplement-compas-two-years-bounds,tab:supplement-maternal-health-risk-bounds,tab:supplement-pima-diabetes-bounds}.

\scriptsize
\begin{longtable}{llrrrrrr}
\caption{Numerical values underlying Figure~\ref{fig:end-to-end-case-auc}(c,d).}\label{tab:supplement-end-to-end-summary}\\
\toprule
Dataset & Method & AUC mean & CI low & CI high & Loss mean & CI low & CI high \\
\midrule
\endfirsthead
\caption[]{Numerical values underlying Figure~\ref{fig:end-to-end-case-auc}(c,d). (continued)}\\
\toprule
Dataset & Method & AUC mean & CI low & CI high & Loss mean & CI low & CI high \\
\midrule
\endhead
\midrule
\multicolumn{8}{r}{Continued on next page} \\
\midrule
\endfoot
\bottomrule
\endlastfoot
Bank marketing & PrivTab & 0.878 & 0.875 & 0.882 & 0.273 & 0.271 & 0.276 \\
Bank marketing & DP-LR & 0.868 & 0.861 & 0.873 & 0.507 & 0.490 & 0.528 \\
Bank marketing & DP-MLP & 0.886 & 0.883 & 0.889 & 0.249 & 0.245 & 0.252 \\
COMPAS & PrivTab & 0.715 & 0.708 & 0.722 & 0.619 & 0.615 & 0.624 \\
COMPAS & DP-LR & 0.707 & 0.690 & 0.720 & 1.282 & 1.224 & 1.340 \\
COMPAS & DP-MLP & 0.691 & 0.675 & 0.705 & 0.642 & 0.633 & 0.652 \\
ACS income & PrivTab & 0.847 & 0.846 & 0.848 & 0.469 & 0.468 & 0.470 \\
ACS income & DP-LR & 0.841 & 0.840 & 0.842 & 0.873 & 0.864 & 0.883 \\
ACS income & DP-MLP & 0.852 & 0.851 & 0.853 & 0.454 & 0.450 & 0.457 \\
Pima diabetes & PrivTab & 0.730 & 0.701 & 0.759 & 0.580 & 0.561 & 0.598 \\
Pima diabetes & DP-LR & 0.726 & 0.667 & 0.775 & 1.655 & 0.956 & 2.688 \\
Pima diabetes & DP-MLP & 0.662 & 0.631 & 0.692 & 0.726 & 0.641 & 0.861 \\
Maternal health & PrivTab & 0.731 & 0.704 & 0.754 & 0.905 & 0.872 & 0.940 \\
Maternal health & DP-LR & 0.715 & 0.680 & 0.743 & 1.364 & 1.187 & 1.641 \\
Maternal health & DP-MLP & 0.629 & 0.580 & 0.672 & 1.053 & 1.000 & 1.098 \\
\end{longtable}
\normalsize

\begin{table}[!ht]
\centering
\caption{ELO ratings for the private-preprocessing evaluation underlying \cref{fig:end-to-end-case-auc}e.
Ratings combine AUC for binary datasets and log loss for multiclass datasets across the five datasets and ten context--target splits at total privacy level \(\mu=0.4\).
The scale is calibrated so that DP-LR has an ELO rating of 1,000.
Intervals are 95\% confidence intervals.}
\label{tab:end-to-end-elo}
\begin{tabular}{lrrr}
\toprule
Method & ELO rating & 95\% CI lower & 95\% CI upper \\
\midrule
PrivTab & 1,161.5 & 1,122.0 & 1,206.2 \\
DP-MLP & 1,117.5 & 1,059.4 & 1,181.1 \\
DP-LR & 1,000.0 & 930.9 & 1,060.0 \\
\bottomrule
\end{tabular}
\end{table}

\footnotesize
\begin{longtable}{llp{0.19\textwidth}p{0.19\textwidth}p{0.19\textwidth}}
\caption{Numerical values underlying Figure~\ref{fig:end-to-end-case-auc}(a,b). Entries report means with bootstrap 95\% CIs over context--target splits; log loss is in nats/sample.}\label{tab:supplement-end-to-end-case-auc}\\
\toprule
Dataset & Metric & Non-DP $z$-score & \shortstack[l]{Partially private $z$-score\\(oracle bounds)} & \shortstack[l]{Private $z$-score\\(LLM bounds)} \\
\midrule
\endfirsthead
\caption[]{Numerical values underlying Figure~\ref{fig:end-to-end-case-auc}(a,b). Entries report means with bootstrap 95\% CIs over context--target splits; log loss is in nats/sample. (continued)}\\
\toprule
Dataset & Metric & Non-DP $z$-score & \shortstack[l]{Partially private $z$-score\\(oracle bounds)} & \shortstack[l]{Private $z$-score\\(LLM bounds)} \\
\midrule
\endhead
\midrule
\multicolumn{5}{r}{Continued on next page} \\
\midrule
\endfoot
\bottomrule
\endlastfoot
Bank marketing & AUC & 0.880 [0.878, 0.882] & 0.877 [0.874, 0.880] & 0.878 [0.875, 0.882] \\
Bank marketing & Log loss & 0.272 [0.270, 0.274] & 0.273 [0.271, 0.275] & 0.273 [0.271, 0.276] \\
COMPAS & AUC & 0.715 [0.708, 0.721] & 0.711 [0.704, 0.718] & 0.715 [0.708, 0.722] \\
COMPAS & Log loss & 0.620 [0.616, 0.624] & 0.623 [0.618, 0.628] & 0.619 [0.615, 0.624] \\
ACS income & AUC & 0.847 [0.846, 0.848] & 0.847 [0.846, 0.848] & 0.847 [0.846, 0.848] \\
ACS income & Log loss & 0.470 [0.469, 0.472] & 0.470 [0.469, 0.472] & 0.469 [0.468, 0.470] \\
Pima diabetes & AUC & 0.801 [0.782, 0.820] & 0.724 [0.649, 0.773] & 0.730 [0.701, 0.759] \\
Pima diabetes & Log loss & 0.518 [0.499, 0.536] & 0.574 [0.551, 0.601] & 0.580 [0.561, 0.598] \\
Maternal health & AUC & 0.807 [0.790, 0.822] & 0.789 [0.765, 0.811] & 0.731 [0.704, 0.754] \\
Maternal health & Log loss & 0.773 [0.749, 0.797] & 0.818 [0.783, 0.853] & 0.905 [0.872, 0.940] \\
\end{longtable}
\normalsize

\footnotesize
\begin{longtable}{llrrp{0.12\textwidth}rrl}
\caption{Public and LLM-derived bounds for ACS income.}\label{tab:supplement-acs-income-bounds}\\
\toprule
Feature & Type & Exact min & Exact max & Source & Lower & Upper & Confidence \\
\midrule
\endfirsthead
\caption[]{Public and LLM-derived bounds for ACS income. (continued)}\\
\toprule
Feature & Type & Exact min & Exact max & Source & Lower & Upper & Confidence \\
\midrule
\endhead
\midrule
\multicolumn{8}{r}{Continued on next page} \\
\midrule
\endfoot
\bottomrule
\endlastfoot
AGEP & numeric & 17 & 95 & llm & 16 & 90 & medium \\
SEX & categorical & 1 & 2 & categorical & 1 & 2 & exact \\
RAC1P & categorical & 1 & 9 & categorical & 1 & 9 & exact \\
SCHL & categorical & 1 & 24 & categorical & 1 & 24 & exact \\
MAR & categorical & 1 & 5 & categorical & 1 & 5 & exact \\
RELP & categorical & 0 & 38 & categorical & 0 & 17 & exact \\
COW & categorical & 1 & 8 & categorical & 1 & 9 & exact \\
OCCP & categorical & 0 & 626 & categorical & 0 & 626 & exact \\
WKHP & numeric & 1 & 99 & llm & 0 & 100 & medium \\
STATE & categorical & 0 & 50 & categorical & 0 & 50 & exact \\
\end{longtable}
\normalsize

\footnotesize
\begin{longtable}{llrrp{0.12\textwidth}rrl}
\caption{Public and LLM-derived bounds for Bank marketing.}\label{tab:supplement-bank-marketing-bounds}\\
\toprule
Feature & Type & Exact min & Exact max & Source & Lower & Upper & Confidence \\
\midrule
\endfirsthead
\caption[]{Public and LLM-derived bounds for Bank marketing. (continued)}\\
\toprule
Feature & Type & Exact min & Exact max & Source & Lower & Upper & Confidence \\
\midrule
\endhead
\midrule
\multicolumn{8}{r}{Continued on next page} \\
\midrule
\endfoot
\bottomrule
\endlastfoot
age & numeric & 18 & 95 & llm & 18 & 100 & medium \\
job & categorical & 0 & 11 & categorical & 0 & 11 & exact \\
marital & categorical & 0 & 2 & categorical & 0 & 2 & exact \\
education & categorical & 0 & 3 & categorical & 0 & 3 & exact \\
default & categorical & 0 & 1 & categorical & 0 & 1 & exact \\
balance & numeric & -8019 & 102127 & llm & -100000 & 1e+06 & low \\
housing & categorical & 0 & 1 & categorical & 0 & 1 & exact \\
loan & categorical & 0 & 1 & categorical & 0 & 1 & exact \\
contact & categorical & 0 & 2 & categorical & 0 & 2 & exact \\
day & numeric & 1 & 31 & llm & 1 & 31 & high \\
month & categorical & 0 & 11 & categorical & 0 & 11 & exact \\
duration & numeric & 0 & 4918 & llm & 0 & 3600 & medium \\
campaign & numeric & 1 & 63 & llm & 1 & 50 & medium \\
pdays & numeric & -1 & 871 & llm & -1 & 7300 & medium \\
previous & numeric & 0 & 275 & llm & 0 & 50 & medium \\
poutcome & categorical & 0 & 3 & categorical & 0 & 3 & exact \\
\end{longtable}
\normalsize

\footnotesize
\begin{longtable}{llrrp{0.12\textwidth}rrl}
\caption{Public and LLM-derived bounds for COMPAS.}\label{tab:supplement-compas-two-years-bounds}\\
\toprule
Feature & Type & Exact min & Exact max & Source & Lower & Upper & Confidence \\
\midrule
\endfirsthead
\caption[]{Public and LLM-derived bounds for COMPAS. (continued)}\\
\toprule
Feature & Type & Exact min & Exact max & Source & Lower & Upper & Confidence \\
\midrule
\endhead
\midrule
\multicolumn{8}{r}{Continued on next page} \\
\midrule
\endfoot
\bottomrule
\endlastfoot
sex & categorical & 0 & 1 & categorical & 0 & 1 & exact \\
age & numeric & 18 & 80 & llm & 18 & 100 & medium \\
juv\_fel\_count & numeric & 0 & 10 & llm & 0 & 10 & medium \\
juv\_misd\_count & numeric & 0 & 13 & llm & 0 & 10 & medium \\
juv\_other\_count & numeric & 0 & 7 & llm & 0 & 10 & medium \\
priors\_count & numeric & 0 & 38 & llm & 0 & 50 & medium \\
age\_cat\_25-45 & categorical & 0 & 1 & categorical & 0 & 1 & exact \\
age\_cat\_Greaterthan45 & categorical & 0 & 1 & categorical & 0 & 1 & exact \\
age\_cat\_Lessthan25 & categorical & 0 & 1 & categorical & 0 & 1 & exact \\
race\_African-American & categorical & 0 & 1 & categorical & 0 & 1 & exact \\
race\_Caucasian & categorical & 0 & 1 & categorical & 0 & 1 & exact \\
c\_charge\_degree\_F & categorical & 0 & 1 & categorical & 0 & 1 & exact \\
c\_charge\_degree\_M & categorical & 0 & 1 & categorical & 0 & 1 & exact \\
\end{longtable}
\normalsize

\footnotesize
\begin{longtable}{llrrp{0.12\textwidth}rrl}
\caption{Public and LLM-derived bounds for Maternal health.}\label{tab:supplement-maternal-health-risk-bounds}\\
\toprule
Feature & Type & Exact min & Exact max & Source & Lower & Upper & Confidence \\
\midrule
\endfirsthead
\caption[]{Public and LLM-derived bounds for Maternal health. (continued)}\\
\toprule
Feature & Type & Exact min & Exact max & Source & Lower & Upper & Confidence \\
\midrule
\endhead
\midrule
\multicolumn{8}{r}{Continued on next page} \\
\midrule
\endfoot
\bottomrule
\endlastfoot
Age & numeric & 10 & 70 & llm & 10 & 60 & medium \\
SystolicBP & numeric & 70 & 160 & llm & 70 & 250 & medium \\
DiastolicBP & numeric & 49 & 100 & llm & 40 & 150 & medium \\
BS & numeric & 6 & 19 & llm & 0 & 40 & medium \\
BodyTemp & numeric & 98 & 103 & llm & 80 & 110 & medium \\
HeartRate & numeric & 7 & 90 & llm & 0 & 250 & medium \\
\end{longtable}
\normalsize

\footnotesize
\begin{longtable}{llrrp{0.12\textwidth}rrl}
\caption{Public and LLM-derived bounds for Pima diabetes.}\label{tab:supplement-pima-diabetes-bounds}\\
\toprule
Feature & Type & Exact min & Exact max & Source & Lower & Upper & Confidence \\
\midrule
\endfirsthead
\caption[]{Public and LLM-derived bounds for Pima diabetes. (continued)}\\
\toprule
Feature & Type & Exact min & Exact max & Source & Lower & Upper & Confidence \\
\midrule
\endhead
\midrule
\multicolumn{8}{r}{Continued on next page} \\
\midrule
\endfoot
\bottomrule
\endlastfoot
preg & numeric & 0 & 17 & llm & 0 & 17 & medium \\
plas & numeric & 0 & 199 & llm & 40 & 400 & medium \\
pres & numeric & 0 & 122 & llm & 30 & 130 & medium \\
skin & numeric & 0 & 99 & llm & 5 & 80 & medium \\
insu & numeric & 0 & 846 & llm & 10 & 500 & medium \\
mass & numeric & 0 & 67.1 & llm & 10 & 70 & medium \\
pedi & numeric & 0.078 & 2.42 & llm & 0 & 3 & medium \\
age & numeric & 21 & 81 & llm & 18 & 100 & medium \\
\end{longtable}
\normalsize

\FloatBarrier
\subsection{Ablations}\label{app:ablations}

Results for noise-aware pretraining, the privacy curriculum, layer and head counts, the attention function, and normalisation and model combination are reported in \cref{tab:supplement-ablation-noise-training,tab:supplement-ablation-mu-curriculum,tab:supplement-ablation-layers,tab:supplement-ablation-heads,tab:supplement-ablation-attention-function,tab:supplement-ablation-normalization-combination}, respectively.

\subsubsection{Ablation: Noise-Aware Pretraining}

We compare PrivTab trained with the same privacy-noise mechanism used at inference against a variant trained without noise-aware pretraining. The noise-aware model ranks substantially higher at every privacy level, which indicates that the effect of privatized summaries should be incorporated during pretraining rather than treated as a perturbation added only at deployment time.

\footnotesize
\begin{longtable}{lp{0.53\textwidth}rrr}
\caption{ELO rankings for the noise-aware-pretraining ablation.}\label{tab:supplement-ablation-noise-training}\\
\toprule
$\mu$ & Method & ELO & CI min & CI max \\
\midrule
\endfirsthead
\caption[]{ELO rankings for the noise-aware-pretraining ablation. (continued)}\\
\toprule
$\mu$ & Method & ELO & CI min & CI max \\
\midrule
\endhead
\midrule
\multicolumn{5}{r}{Continued on next page} \\
\midrule
\endfoot
\bottomrule
\endlastfoot
\rowcolor{gray!8} overall & PrivTab with noise-aware pretraining & 1440.2 & 1430.9 & 1449.5 \\
\rowcolor{gray!8} overall & PrivTab without noise-aware pretraining & 1188.1 & 1176.2 & 1200.0 \\
0.05 & PrivTab with noise-aware pretraining & 1404.5 & 1381.8 & 1427.4 \\
0.05 & PrivTab without noise-aware pretraining & 1108.9 & 1079.7 & 1135.6 \\
\rowcolor{gray!8} 0.1 & PrivTab with noise-aware pretraining & 1459.1 & 1435.9 & 1482.6 \\
\rowcolor{gray!8} 0.1 & PrivTab without noise-aware pretraining & 1084.6 & 1055.2 & 1111.6 \\
0.2 & PrivTab with noise-aware pretraining & 1508.7 & 1487.3 & 1530.3 \\
0.2 & PrivTab without noise-aware pretraining & 1139.5 & 1111.3 & 1164.8 \\
\rowcolor{gray!8} 0.4 & PrivTab with noise-aware pretraining & 1492.1 & 1467.7 & 1516.1 \\
\rowcolor{gray!8} 0.4 & PrivTab without noise-aware pretraining & 1236.1 & 1205.8 & 1264.0 \\
0.8 & PrivTab with noise-aware pretraining & 1450.8 & 1426.0 & 1475.0 \\
0.8 & PrivTab without noise-aware pretraining & 1279.1 & 1247.1 & 1308.3 \\
\rowcolor{gray!8} 1.6 & PrivTab with noise-aware pretraining & 1432.1 & 1405.9 & 1457.5 \\
\rowcolor{gray!8} 1.6 & PrivTab without noise-aware pretraining & 1328.4 & 1298.9 & 1356.4 \\
\end{longtable}
\normalsize

\subsubsection{Ablation: \texorpdfstring{$\mu$}{mu}-Range Curriculum}

We compare different privacy-range curricula used during PrivTab pretraining. The benchmark configuration with $\mu \in [0.16,64]$ is strongest overall; nearby ranges remain competitive, but shifting the curriculum too far toward either weaker or stronger privacy leads to a measurable degradation in the aggregate ranking.

\footnotesize
\begin{longtable}{lp{0.53\textwidth}rrr}
\caption{ELO rankings for the $\mu$-range pretraining ablation.}\label{tab:supplement-ablation-mu-curriculum}\\
\toprule
$\mu$ & Method & ELO & CI min & CI max \\
\midrule
\endfirsthead
\caption[]{ELO rankings for the $\mu$-range pretraining ablation. (continued)}\\
\toprule
$\mu$ & Method & ELO & CI min & CI max \\
\midrule
\endhead
\midrule
\multicolumn{5}{r}{Continued on next page} \\
\midrule
\endfoot
\bottomrule
\endlastfoot
\rowcolor{gray!8} overall & PrivTab ($\mu \in [0.16,64]$) & 1533.1 & 1523.6 & 1542.6 \\
\rowcolor{gray!8} overall & PrivTab ($\mu \in [0.2,80]$) & 1501.6 & 1493.1 & 1510.2 \\
\rowcolor{gray!8} overall & PrivTab ($\mu \in [0.64,256]$) & 1489.9 & 1480.7 & 1498.9 \\
\rowcolor{gray!8} overall & PrivTab ($\mu \in [0.1,40]$) & 1486.0 & 1478.0 & 1494.0 \\
\rowcolor{gray!8} overall & PrivTab ($\mu \in [0.4,160]$) & 1480.9 & 1472.9 & 1489.0 \\
\rowcolor{gray!8} overall & PrivTab ($\mu \in [0.05,20]$) & 1470.4 & 1462.8 & 1478.2 \\
\rowcolor{gray!8} overall & PrivTab ($\mu \in [0.04,16]$) & 1450.9 & 1442.8 & 1458.9 \\
\rowcolor{gray!8} overall & PrivTab ($\mu \in [0.8,320]$) & 1434.4 & 1426.5 & 1442.5 \\
\rowcolor{gray!8} overall & PrivTab ($\mu \in [0.8,16]$) & 1399.6 & 1390.7 & 1408.3 \\
\rowcolor{gray!8} overall & PrivTab ($\mu \in [2.56,1024]$) & 1358.2 & 1348.5 & 1367.9 \\
0.05 & PrivTab ($\mu \in [0.16,64]$) & 1483.8 & 1459.1 & 1509.5 \\
0.05 & PrivTab ($\mu \in [0.05,20]$) & 1461.8 & 1442.0 & 1483.3 \\
0.05 & PrivTab ($\mu \in [0.1,40]$) & 1455.5 & 1433.9 & 1478.3 \\
0.05 & PrivTab ($\mu \in [0.04,16]$) & 1449.2 & 1429.0 & 1470.2 \\
0.05 & PrivTab ($\mu \in [0.2,80]$) & 1438.1 & 1418.0 & 1459.6 \\
0.05 & PrivTab ($\mu \in [0.4,160]$) & 1432.3 & 1412.7 & 1453.3 \\
0.05 & PrivTab ($\mu \in [0.64,256]$) & 1391.9 & 1371.1 & 1412.5 \\
0.05 & PrivTab ($\mu \in [0.8,320]$) & 1347.3 & 1327.4 & 1367.6 \\
0.05 & PrivTab ($\mu \in [0.8,16]$) & 1332.7 & 1311.7 & 1353.6 \\
0.05 & PrivTab ($\mu \in [2.56,1024]$) & 1225.2 & 1201.6 & 1248.0 \\
\rowcolor{gray!8} 0.1 & PrivTab ($\mu \in [0.1,40]$) & 1521.6 & 1500.6 & 1543.7 \\
\rowcolor{gray!8} 0.1 & PrivTab ($\mu \in [0.16,64]$) & 1519.8 & 1497.9 & 1543.5 \\
\rowcolor{gray!8} 0.1 & PrivTab ($\mu \in [0.2,80]$) & 1497.0 & 1476.5 & 1518.8 \\
\rowcolor{gray!8} 0.1 & PrivTab ($\mu \in [0.05,20]$) & 1495.8 & 1476.4 & 1515.2 \\
\rowcolor{gray!8} 0.1 & PrivTab ($\mu \in [0.04,16]$) & 1493.8 & 1473.4 & 1515.2 \\
\rowcolor{gray!8} 0.1 & PrivTab ($\mu \in [0.4,160]$) & 1482.6 & 1463.4 & 1502.6 \\
\rowcolor{gray!8} 0.1 & PrivTab ($\mu \in [0.64,256]$) & 1453.2 & 1433.1 & 1474.5 \\
\rowcolor{gray!8} 0.1 & PrivTab ($\mu \in [0.8,320]$) & 1414.2 & 1394.8 & 1433.8 \\
\rowcolor{gray!8} 0.1 & PrivTab ($\mu \in [0.8,16]$) & 1392.1 & 1370.5 & 1414.0 \\
\rowcolor{gray!8} 0.1 & PrivTab ($\mu \in [2.56,1024]$) & 1297.5 & 1272.2 & 1321.2 \\
0.2 & PrivTab ($\mu \in [0.16,64]$) & 1559.4 & 1537.8 & 1581.9 \\
0.2 & PrivTab ($\mu \in [0.2,80]$) & 1547.9 & 1526.6 & 1570.7 \\
0.2 & PrivTab ($\mu \in [0.1,40]$) & 1545.5 & 1526.7 & 1565.9 \\
0.2 & PrivTab ($\mu \in [0.05,20]$) & 1524.5 & 1505.1 & 1545.6 \\
0.2 & PrivTab ($\mu \in [0.4,160]$) & 1509.6 & 1489.8 & 1529.7 \\
0.2 & PrivTab ($\mu \in [0.04,16]$) & 1504.2 & 1485.0 & 1524.0 \\
0.2 & PrivTab ($\mu \in [0.64,256]$) & 1490.2 & 1468.2 & 1512.0 \\
0.2 & PrivTab ($\mu \in [0.8,320]$) & 1452.7 & 1432.1 & 1474.0 \\
0.2 & PrivTab ($\mu \in [0.8,16]$) & 1433.6 & 1412.3 & 1454.2 \\
0.2 & PrivTab ($\mu \in [2.56,1024]$) & 1336.7 & 1311.5 & 1360.4 \\
\rowcolor{gray!8} 0.4 & PrivTab ($\mu \in [0.16,64]$) & 1595.2 & 1572.5 & 1619.2 \\
\rowcolor{gray!8} 0.4 & PrivTab ($\mu \in [0.2,80]$) & 1564.8 & 1544.4 & 1586.1 \\
\rowcolor{gray!8} 0.4 & PrivTab ($\mu \in [0.64,256]$) & 1557.6 & 1535.7 & 1580.4 \\
\rowcolor{gray!8} 0.4 & PrivTab ($\mu \in [0.1,40]$) & 1536.4 & 1518.5 & 1554.8 \\
\rowcolor{gray!8} 0.4 & PrivTab ($\mu \in [0.4,160]$) & 1533.0 & 1512.7 & 1554.1 \\
\rowcolor{gray!8} 0.4 & PrivTab ($\mu \in [0.05,20]$) & 1516.0 & 1497.8 & 1534.9 \\
\rowcolor{gray!8} 0.4 & PrivTab ($\mu \in [0.04,16]$) & 1493.3 & 1473.7 & 1512.8 \\
\rowcolor{gray!8} 0.4 & PrivTab ($\mu \in [0.8,320]$) & 1489.5 & 1469.7 & 1509.8 \\
\rowcolor{gray!8} 0.4 & PrivTab ($\mu \in [0.8,16]$) & 1459.8 & 1438.0 & 1481.8 \\
\rowcolor{gray!8} 0.4 & PrivTab ($\mu \in [2.56,1024]$) & 1404.0 & 1379.8 & 1428.2 \\
0.8 & PrivTab ($\mu \in [0.16,64]$) & 1578.5 & 1554.7 & 1602.8 \\
0.8 & PrivTab ($\mu \in [0.64,256]$) & 1567.1 & 1543.8 & 1590.9 \\
0.8 & PrivTab ($\mu \in [0.2,80]$) & 1540.7 & 1520.1 & 1561.3 \\
0.8 & PrivTab ($\mu \in [0.4,160]$) & 1518.6 & 1498.1 & 1539.9 \\
0.8 & PrivTab ($\mu \in [0.8,320]$) & 1484.6 & 1465.1 & 1504.3 \\
0.8 & PrivTab ($\mu \in [0.1,40]$) & 1483.7 & 1465.5 & 1502.6 \\
0.8 & PrivTab ($\mu \in [0.05,20]$) & 1466.8 & 1449.3 & 1484.7 \\
0.8 & PrivTab ($\mu \in [2.56,1024]$) & 1451.3 & 1427.1 & 1475.9 \\
0.8 & PrivTab ($\mu \in [0.04,16]$) & 1444.0 & 1425.6 & 1462.1 \\
0.8 & PrivTab ($\mu \in [0.8,16]$) & 1425.0 & 1400.8 & 1448.4 \\
\rowcolor{gray!8} 1.6 & PrivTab ($\mu \in [0.64,256]$) & 1548.7 & 1524.9 & 1573.9 \\
\rowcolor{gray!8} 1.6 & PrivTab ($\mu \in [0.16,64]$) & 1533.9 & 1510.5 & 1558.5 \\
\rowcolor{gray!8} 1.6 & PrivTab ($\mu \in [0.2,80]$) & 1489.2 & 1469.3 & 1509.9 \\
\rowcolor{gray!8} 1.6 & PrivTab ($\mu \in [0.8,320]$) & 1478.9 & 1459.6 & 1498.5 \\
\rowcolor{gray!8} 1.6 & PrivTab ($\mu \in [2.56,1024]$) & 1478.0 & 1454.4 & 1502.6 \\
\rowcolor{gray!8} 1.6 & PrivTab ($\mu \in [0.4,160]$) & 1475.2 & 1454.5 & 1496.7 \\
\rowcolor{gray!8} 1.6 & PrivTab ($\mu \in [0.1,40]$) & 1441.0 & 1423.0 & 1459.7 \\
\rowcolor{gray!8} 1.6 & PrivTab ($\mu \in [0.05,20]$) & 1423.1 & 1404.5 & 1441.9 \\
\rowcolor{gray!8} 1.6 & PrivTab ($\mu \in [0.8,16]$) & 1410.6 & 1386.1 & 1434.8 \\
\rowcolor{gray!8} 1.6 & PrivTab ($\mu \in [0.04,16]$) & 1381.9 & 1363.4 & 1400.0 \\
\end{longtable}
\normalsize

\subsubsection{Ablation: Number of Private Layers}

We vary the number of main DP-MHCA layers while keeping the rest of the architecture fixed. Three private layers give the strongest overall ranking; reducing the depth to two layers weakens performance substantially, whereas increasing to four layers does not improve over the benchmark configuration.

\footnotesize
\begin{longtable}{lp{0.53\textwidth}rrr}
\caption{ELO rankings for the ablation over the number of main DP-MHCA layers.}\label{tab:supplement-ablation-layers}\\
\toprule
$\mu$ & Method & ELO & CI min & CI max \\
\midrule
\endfirsthead
\caption[]{ELO rankings for the ablation over the number of main DP-MHCA layers. (continued)}\\
\toprule
$\mu$ & Method & ELO & CI min & CI max \\
\midrule
\endhead
\midrule
\multicolumn{5}{r}{Continued on next page} \\
\midrule
\endfoot
\bottomrule
\endlastfoot
\rowcolor{gray!8} overall & PrivTab ($L_{\mathrm{main}}=3$) & 1507.2 & 1496.3 & 1518.6 \\
\rowcolor{gray!8} overall & PrivTab ($L_{\mathrm{main}}=4$) & 1404.0 & 1394.2 & 1414.0 \\
\rowcolor{gray!8} overall & PrivTab ($L_{\mathrm{main}}=2$) & 1283.2 & 1273.8 & 1292.9 \\
0.05 & PrivTab ($L_{\mathrm{main}}=3$) & 1493.3 & 1465.6 & 1524.2 \\
0.05 & PrivTab ($L_{\mathrm{main}}=4$) & 1453.3 & 1430.0 & 1478.8 \\
0.05 & PrivTab ($L_{\mathrm{main}}=2$) & 1294.6 & 1271.2 & 1317.0 \\
\rowcolor{gray!8} 0.1 & PrivTab ($L_{\mathrm{main}}=3$) & 1575.4 & 1546.5 & 1608.5 \\
\rowcolor{gray!8} 0.1 & PrivTab ($L_{\mathrm{main}}=4$) & 1472.5 & 1448.7 & 1498.9 \\
\rowcolor{gray!8} 0.1 & PrivTab ($L_{\mathrm{main}}=2$) & 1320.6 & 1296.5 & 1345.4 \\
0.2 & PrivTab ($L_{\mathrm{main}}=3$) & 1595.6 & 1568.0 & 1627.4 \\
0.2 & PrivTab ($L_{\mathrm{main}}=4$) & 1464.7 & 1440.8 & 1491.8 \\
0.2 & PrivTab ($L_{\mathrm{main}}=2$) & 1323.5 & 1298.7 & 1349.5 \\
\rowcolor{gray!8} 0.4 & PrivTab ($L_{\mathrm{main}}=3$) & 1514.4 & 1490.5 & 1542.4 \\
\rowcolor{gray!8} 0.4 & PrivTab ($L_{\mathrm{main}}=4$) & 1409.2 & 1384.9 & 1435.1 \\
\rowcolor{gray!8} 0.4 & PrivTab ($L_{\mathrm{main}}=2$) & 1302.7 & 1279.3 & 1327.8 \\
0.8 & PrivTab ($L_{\mathrm{main}}=3$) & 1472.8 & 1449.8 & 1498.4 \\
0.8 & PrivTab ($L_{\mathrm{main}}=4$) & 1357.1 & 1334.3 & 1381.4 \\
0.8 & PrivTab ($L_{\mathrm{main}}=2$) & 1267.9 & 1245.1 & 1291.9 \\
\rowcolor{gray!8} 1.6 & PrivTab ($L_{\mathrm{main}}=3$) & 1432.3 & 1408.8 & 1458.9 \\
\rowcolor{gray!8} 1.6 & PrivTab ($L_{\mathrm{main}}=4$) & 1305.2 & 1283.5 & 1327.8 \\
\rowcolor{gray!8} 1.6 & PrivTab ($L_{\mathrm{main}}=2$) & 1218.3 & 1196.0 & 1241.1 \\
\end{longtable}
\normalsize

\subsubsection{Ablation: Number of DP-MHCA Heads}

We compare different head counts inside DP-MHCA. In this ablation, a single DP-MHCA head performs better overall than using four DP-MHCA heads, especially in the more private regime, which supports the benchmark choice of a single privacy-critical head with the full internal width.

\footnotesize
\begin{longtable}{lp{0.53\textwidth}rrr}
\caption{ELO rankings for the ablation over the number of DP-MHCA heads.}\label{tab:supplement-ablation-heads}\\
\toprule
$\mu$ & Method & ELO & CI min & CI max \\
\midrule
\endfirsthead
\caption[]{ELO rankings for the ablation over the number of DP-MHCA heads. (continued)}\\
\toprule
$\mu$ & Method & ELO & CI min & CI max \\
\midrule
\endhead
\midrule
\multicolumn{5}{r}{Continued on next page} \\
\midrule
\endfoot
\bottomrule
\endlastfoot
\rowcolor{gray!8} overall & PrivTab ($H_{\mathrm{DP}}=1$) & 1508.9 & 1500.1 & 1517.9 \\
\rowcolor{gray!8} overall & PrivTab ($H_{\mathrm{DP}}=4$) & 1468.5 & 1459.3 & 1478.0 \\
0.05 & PrivTab ($H_{\mathrm{DP}}=1$) & 1465.8 & 1441.1 & 1491.8 \\
0.05 & PrivTab ($H_{\mathrm{DP}}=4$) & 1422.8 & 1399.6 & 1448.4 \\
\rowcolor{gray!8} 0.1 & PrivTab ($H_{\mathrm{DP}}=1$) & 1464.8 & 1442.4 & 1488.9 \\
\rowcolor{gray!8} 0.1 & PrivTab ($H_{\mathrm{DP}}=4$) & 1462.6 & 1440.0 & 1487.7 \\
0.2 & PrivTab ($H_{\mathrm{DP}}=1$) & 1527.9 & 1507.3 & 1550.8 \\
0.2 & PrivTab ($H_{\mathrm{DP}}=4$) & 1506.5 & 1483.1 & 1532.0 \\
\rowcolor{gray!8} 0.4 & PrivTab ($H_{\mathrm{DP}}=1$) & 1566.0 & 1544.8 & 1588.8 \\
\rowcolor{gray!8} 0.4 & PrivTab ($H_{\mathrm{DP}}=4$) & 1517.9 & 1495.8 & 1542.1 \\
0.8 & PrivTab ($H_{\mathrm{DP}}=1$) & 1587.0 & 1568.0 & 1609.2 \\
0.8 & PrivTab ($H_{\mathrm{DP}}=4$) & 1523.6 & 1499.9 & 1549.3 \\
\rowcolor{gray!8} 1.6 & PrivTab ($H_{\mathrm{DP}}=1$) & 1570.8 & 1549.0 & 1595.9 \\
\rowcolor{gray!8} 1.6 & PrivTab ($H_{\mathrm{DP}}=4$) & 1492.6 & 1468.3 & 1519.6 \\
\end{longtable}
\normalsize

\subsubsection{Ablation: DP-MHCA Interaction Function}

We compare alternative bounded interaction functions inside DP-MHCA:
\begin{align*}
g_{\text{tanh}}(\mathbf{q},\mathbf{k}) &= \tanh(\langle \mathbf{q},\mathbf{k}\rangle/\sqrt{d_h}), \\
g_{\text{linear clipped}}(\mathbf{q},\mathbf{k}) &= \langle \operatorname{clip}_1(\mathbf{q}), \operatorname{clip}_1(\mathbf{k})\rangle, \\
g_{\text{linear normalised}}(\mathbf{q},\mathbf{k}) &= \langle \mathbf{q}/\|\mathbf{q}\|_2, \mathbf{k}/\|\mathbf{k}\|_2\rangle, \\
g_{\text{shifted softsign}}(\mathbf{q},\mathbf{k}) &= \frac{1}{2}(1+\operatorname{softsign}(\langle \mathbf{q},\mathbf{k}\rangle)).
\end{align*}
The benchmark tanh choice is strongest overall and becomes increasingly preferable as privacy weakens, while the alternatives rank lower in aggregate.

\footnotesize
\begin{longtable}{lp{0.53\textwidth}rrr}
\caption{ELO rankings for the ablation over the DP-MHCA interaction function.}\label{tab:supplement-ablation-attention-function}\\
\toprule
$\mu$ & Method & ELO & CI min & CI max \\
\midrule
\endfirsthead
\caption[]{ELO rankings for the ablation over the DP-MHCA interaction function. (continued)}\\
\toprule
$\mu$ & Method & ELO & CI min & CI max \\
\midrule
\endhead
\midrule
\multicolumn{5}{r}{Continued on next page} \\
\midrule
\endfoot
\bottomrule
\endlastfoot
\rowcolor{gray!8} overall & PrivTab (tanh) & 1503.9 & 1492.4 & 1515.9 \\
\rowcolor{gray!8} overall & PrivTab (shifted softsign) & 1409.4 & 1399.9 & 1419.5 \\
\rowcolor{gray!8} overall & PrivTab (linear normalised) & 1392.4 & 1384.1 & 1400.8 \\
\rowcolor{gray!8} overall & PrivTab (linear clipped) & 1370.2 & 1361.3 & 1379.3 \\
0.05 & PrivTab (linear clipped) & 1397.6 & 1376.8 & 1420.3 \\
0.05 & PrivTab (linear normalised) & 1394.7 & 1374.9 & 1416.1 \\
0.05 & PrivTab (tanh) & 1374.7 & 1349.5 & 1401.3 \\
0.05 & PrivTab (shifted softsign) & 1313.8 & 1290.7 & 1338.2 \\
\rowcolor{gray!8} 0.1 & PrivTab (tanh) & 1475.1 & 1451.1 & 1501.5 \\
\rowcolor{gray!8} 0.1 & PrivTab (linear normalised) & 1438.1 & 1417.1 & 1460.8 \\
\rowcolor{gray!8} 0.1 & PrivTab (linear clipped) & 1432.6 & 1409.2 & 1457.1 \\
\rowcolor{gray!8} 0.1 & PrivTab (shifted softsign) & 1369.8 & 1345.6 & 1394.3 \\
0.2 & PrivTab (tanh) & 1535.3 & 1508.3 & 1565.5 \\
0.2 & PrivTab (linear normalised) & 1509.5 & 1488.0 & 1533.6 \\
0.2 & PrivTab (linear clipped) & 1473.7 & 1451.7 & 1498.4 \\
0.2 & PrivTab (shifted softsign) & 1461.9 & 1436.5 & 1489.0 \\
\rowcolor{gray!8} 0.4 & PrivTab (tanh) & 1561.3 & 1532.4 & 1593.8 \\
\rowcolor{gray!8} 0.4 & PrivTab (shifted softsign) & 1462.4 & 1439.1 & 1487.5 \\
\rowcolor{gray!8} 0.4 & PrivTab (linear normalised) & 1412.2 & 1391.3 & 1434.6 \\
\rowcolor{gray!8} 0.4 & PrivTab (linear clipped) & 1377.3 & 1354.7 & 1401.0 \\
0.8 & PrivTab (tanh) & 1581.6 & 1549.2 & 1619.9 \\
0.8 & PrivTab (shifted softsign) & 1455.6 & 1432.0 & 1481.8 \\
0.8 & PrivTab (linear normalised) & 1323.1 & 1302.7 & 1343.7 \\
0.8 & PrivTab (linear clipped) & 1299.3 & 1277.1 & 1321.5 \\
\rowcolor{gray!8} 1.6 & PrivTab (tanh) & 1612.2 & 1575.6 & 1656.1 \\
\rowcolor{gray!8} 1.6 & PrivTab (shifted softsign) & 1469.9 & 1446.3 & 1495.0 \\
\rowcolor{gray!8} 1.6 & PrivTab (linear normalised) & 1326.8 & 1306.7 & 1347.3 \\
\rowcolor{gray!8} 1.6 & PrivTab (linear clipped) & 1282.3 & 1262.2 & 1302.4 \\
\end{longtable}
\normalsize

\subsubsection{Ablation: Context-Scale Normalization and Model Combination}

The small-context model trained on context lengths sampled uniformly from $[100,2048]$ benefits substantially from enabling the normalisation in \cref{eq:dp-mhca-normalization} at inference time, and the large-context model trained on context lengths sampled uniformly from $[1024,8192]$ with normalisation is also strong. The final combination described in \cref{sec:dp-tnp-pretraining}, which routes to the small-context normalised model for context lengths below 4096 and to the large-context normalised model otherwise, achieves the best aggregate ranking. This supports the motivation for the final benchmark model as a context-scale-aware combination rather than a single fixed variant.

\footnotesize
\begin{longtable}{lp{0.53\textwidth}rrr}
\caption{ELO rankings for the normalisation and model-combination ablation.}\label{tab:supplement-ablation-normalization-combination}\\
\toprule
$\mu$ & Method & ELO & CI min & CI max \\
\midrule
\endfirsthead
\caption[]{ELO rankings for the normalisation and model-combination ablation. (continued)}\\
\toprule
$\mu$ & Method & ELO & CI min & CI max \\
\midrule
\endhead
\midrule
\multicolumn{5}{r}{Continued on next page} \\
\midrule
\endfoot
\bottomrule
\endlastfoot
\rowcolor{gray!8} overall & PrivTab combination (final) & 1578.1 & 1571.0 & 1585.2 \\
\rowcolor{gray!8} overall & PrivTab small-context model + inference normalisation & 1563.4 & 1556.1 & 1570.6 \\
\rowcolor{gray!8} overall & PrivTab large-context model + normalisation & 1552.2 & 1543.8 & 1560.7 \\
\rowcolor{gray!8} overall & DP-MLP & 1500.5 & 1489.1 & 1512.0 \\
\rowcolor{gray!8} overall & PrivTab small-context model & 1414.3 & 1404.3 & 1424.1 \\
\rowcolor{gray!8} overall & DP-LR & 1411.6 & 1400.3 & 1422.8 \\
0.05 & PrivTab small-context model + inference normalisation & 1502.9 & 1481.9 & 1526.1 \\
0.05 & PrivTab combination (final) & 1496.6 & 1477.5 & 1517.0 \\
0.05 & PrivTab large-context model + normalisation & 1451.2 & 1431.3 & 1472.3 \\
0.05 & PrivTab small-context model & 1420.1 & 1396.0 & 1445.2 \\
0.05 & DP-LR & 1352.5 & 1325.7 & 1378.9 \\
0.05 & DP-MLP & 1271.8 & 1244.7 & 1297.7 \\
\rowcolor{gray!8} 0.1 & PrivTab small-context model + inference normalisation & 1602.6 & 1583.9 & 1623.7 \\
\rowcolor{gray!8} 0.1 & PrivTab combination (final) & 1597.5 & 1580.2 & 1616.1 \\
\rowcolor{gray!8} 0.1 & PrivTab large-context model + normalisation & 1539.2 & 1519.3 & 1560.6 \\
\rowcolor{gray!8} 0.1 & PrivTab small-context model & 1456.7 & 1432.5 & 1481.9 \\
\rowcolor{gray!8} 0.1 & DP-LR & 1408.3 & 1379.7 & 1435.8 \\
\rowcolor{gray!8} 0.1 & DP-MLP & 1377.8 & 1350.2 & 1404.7 \\
0.2 & PrivTab combination (final) & 1647.6 & 1629.6 & 1667.3 \\
0.2 & PrivTab small-context model + inference normalisation & 1633.5 & 1614.9 & 1654.7 \\
0.2 & PrivTab large-context model + normalisation & 1588.6 & 1568.2 & 1611.5 \\
0.2 & DP-MLP & 1482.2 & 1456.2 & 1509.0 \\
0.2 & PrivTab small-context model & 1447.3 & 1424.7 & 1470.7 \\
0.2 & DP-LR & 1438.7 & 1408.3 & 1468.2 \\
\rowcolor{gray!8} 0.4 & PrivTab combination (final) & 1646.4 & 1631.1 & 1663.9 \\
\rowcolor{gray!8} 0.4 & PrivTab small-context model + inference normalisation & 1633.2 & 1617.8 & 1650.5 \\
\rowcolor{gray!8} 0.4 & PrivTab large-context model + normalisation & 1623.3 & 1602.3 & 1646.1 \\
\rowcolor{gray!8} 0.4 & DP-MLP & 1581.3 & 1553.6 & 1609.9 \\
\rowcolor{gray!8} 0.4 & DP-LR & 1460.9 & 1431.5 & 1490.3 \\
\rowcolor{gray!8} 0.4 & PrivTab small-context model & 1441.4 & 1416.1 & 1467.2 \\
0.8 & DP-MLP & 1683.5 & 1654.8 & 1716.5 \\
0.8 & PrivTab large-context model + normalisation & 1624.5 & 1603.6 & 1647.0 \\
0.8 & PrivTab combination (final) & 1615.0 & 1599.3 & 1632.5 \\
0.8 & PrivTab small-context model + inference normalisation & 1576.6 & 1562.6 & 1592.2 \\
0.8 & DP-LR & 1449.3 & 1419.8 & 1478.4 \\
0.8 & PrivTab small-context model & 1406.9 & 1378.3 & 1434.9 \\
\rowcolor{gray!8} 1.6 & DP-MLP & 1727.5 & 1694.1 & 1765.7 \\
\rowcolor{gray!8} 1.6 & PrivTab large-context model + normalisation & 1591.6 & 1569.6 & 1614.9 \\
\rowcolor{gray!8} 1.6 & PrivTab combination (final) & 1574.8 & 1557.7 & 1594.4 \\
\rowcolor{gray!8} 1.6 & PrivTab small-context model + inference normalisation & 1538.6 & 1523.3 & 1555.2 \\
\rowcolor{gray!8} 1.6 & DP-LR & 1435.9 & 1408.2 & 1464.2 \\
\rowcolor{gray!8} 1.6 & PrivTab small-context model & 1376.7 & 1348.2 & 1403.8 \\
\end{longtable}
\normalsize

\FloatBarrier
\end{document}